%% file: neurips2020.tex
\documentclass{article}

     \usepackage[final]{neurips2020}

\usepackage[utf8]{inputenc} 
\usepackage[T1]{fontenc}    
\usepackage{hyperref}       
\usepackage{url}            
\usepackage{booktabs}       
\usepackage{amsfonts}       
\usepackage{nicefrac}       
\usepackage{microtype}      

\usepackage{amssymb}

\usepackage{dblfloatfix}    
\usepackage{todonotes}
\usepackage{subcaption}
\usepackage{amsthm}

\usepackage[inline]{enumitem}
\usepackage{calc}

\usepackage{amsmath, nccmath}
\newtheorem{mydef}{Definition}
\newtheorem{theorem}{Theorem}
\newtheorem{assump}{Assumption}
\newtheorem{lemma}{Lemma}
\newtheorem{corollary}{Corollary}

\input{math_commands.tex}

\makeatletter
\providecommand*{\barvee}{%
  \mathbin{%
    \mathpalette\@barvee{}%
  }%
}
\newcommand*{\@barvee}[2]{%
  \sbox0{$#1\veebar\m@th$}%
  \sbox2{%
    \hbox to \wd0{%
      \hss
      \resizebox{1.05\wd0}{\height}{$#1-\m@th$}%
      \hss
    }%
  }%
  \sbox4{%
    \resizebox{\wd0}{.7\ht0}{$#1\vee\m@th$}%
  }%
  \sbox6{$#1\vcenter{}$}
  \ht2=\ht6 %
  \vbox to \ht0{%
    \copy2 %
    \vss
    \copy4 %
  }%
}
\makeatother

\newcommand{\state}{\mathcal{S}}
\newcommand{\action}{\mathcal{A}}
\newcommand{\dynamics}{\rho}
\newcommand{\reward}{r}
\newcommand{\rmax}{\reward_{\text{MAX}}}
\newcommand{\rmin}{\reward_{\text{MIN}}}
\newcommand{\goals}{\mathcal{G}}
\renewcommand{\v}{V}
\newcommand{\q}{Q}
\newcommand{\vpi}{V^\pi}
\newcommand{\qpi}{Q^\pi}
\newcommand{\vstar}{V^*}
\newcommand{\qstar}{Q^*}
\newcommand{\pistar}{\pi^*}

\newcommand{\virtual}{\omega}
\newcommand{\tasks}{\mathcal{M}}

\usepackage{mathtools}

\newcommand{\universal}{\mathcal{U}}
\newcommand{\nothing}{\varnothing}
\newcommand{\rbig}{\reward_{\universal}}
\newcommand{\rsmall}{\reward_{\nothing}}
\newcommand{\mbig}{\tasks_{\universal}}
\newcommand{\msmall}{\tasks_{\nothing}}

\newcommand{\rbar}{\bar{\reward}}
\newcommand{\rbarmin}{\rbar_{\text{MIN}}}
\newcommand{\qbar}{\bar{\q}}
\newcommand{\qstarbar}{\bar{\q}^*}

\newcommand{\constantr}{\reward_{s,a}}

\newcommand{\qstarbarbig}{\qstarbar_{\universal}}
\newcommand{\qstarbarsmall}{\qstarbar_{\nothing}}

\newcommand{\goalq}{\bar{\mathcal{\q}}^*}

\newcommand{\rneg}[1]{
\reward_{\neg #1}
}
\newcommand{\rand}[2]{
\reward_{#1 \wedge #2}
}
\newcommand{\ror}[2]{
\reward_{#1 \vee #2}
}

\newcommand{\gstar}{G^*_{s:g, a}}

\usepackage{hyperref}

\title{A Boolean Task Algebra For Reinforcement Learning}

\author{%
  Geraud Nangue Tasse, Steven James, Benjamin Rosman
  \\
  School of Computer Science and Applied Mathematics\\
  University of the Witwatersrand\\
  Johannesburg, South Africa \\
  \texttt{geraudnt@gmail.com, \{steven.james, benjamin.rosman1\}@wits.ac.za}
}

\begin{document}

\maketitle

\begin{abstract}

The ability to compose learned skills to solve new tasks is an important property of lifelong-learning agents. 
In this work, we formalise the logical composition of tasks as a Boolean algebra. 
This allows us to formulate new tasks in terms of the negation, disjunction and conjunction of a set of base tasks.
We then show that by learning goal-oriented value functions and restricting the transition dynamics of the tasks, an agent can solve these new tasks with no further learning.
We prove that by composing these value functions in specific ways, we immediately recover the optimal policies for all tasks expressible under the Boolean algebra.
We verify our approach in two domains---including a high-dimensional video game environment requiring function approximation---where an agent first learns a set of base skills, and then composes them to solve a super-exponential number of new tasks. 

\end{abstract}

\section{Introduction}

Reinforcement learning (RL) has achieved recent success in a number of difficult, high-dimensional environments \citep{mnih15,levine16,lillicrap16,silver17}.
However, these methods generally require millions of samples from the environment to learn optimal behaviours, limiting their real-world applicability. 
A major challenge is thus in designing sample-efficient agents that can transfer their existing knowledge to solve new tasks quickly.
This is particularly important for agents in a multitask or lifelong setting, since learning to solve complex tasks from scratch is typically impractical.

One approach to transfer is \textit{composition} \citep{todorov09}, which allows an agent to leverage existing skills to build complex, novel behaviours.
These newly-formed skills can then be used to solve or speed up learning in a new task.  
In this work, we focus on concurrent composition, where existing base skills are combined to produce new skills \citep{todorov09,saxe17,haarnoja18,vanniekerk19,hunt19,peng19}.
This differs from other forms of composition, such as options \citep{sutton99} and hierarchical RL \citep{barto03}, where actions and skills are chained in a temporal sequence.  

While previous work on logical composition considers only the union and intersection of tasks \citep{haarnoja18,vanniekerk19,hunt19}, they do not formally define them. 
However, union and intersection are operations on sets, rather than tasks.
We therefore formalise the notion of union and intersection of tasks using the Boolean algebra structure, since this is the algebraic structure that abstracts the notions of union, intersection, and complement of sets. 
We then define a Boolean algebra over the space of optimal value functions, and then prove that there exists a homomorphism between the task and value function algebras. 
Given a set of base tasks that have been previously solved by the agent, any new task written as a Boolean expression can immediately be solved without further learning,  resulting in a zero-shot super-exponential explosion in the agent's abilities. We summarise our main contributions as follows:

\begin{enumerate}
	\item \textit{Boolean task algebra:} We formalise the disjunction, conjunction, and negation of tasks in a Boolean algebra structure. This extends previous composition work to encompass \textit{all} Boolean operators, and enables us to apply logic to tasks, much as we would to propositions.  
	
	\item \textit{Extended value functions:} We introduce a new type of goal-oriented value function that encodes how to achieve all goals in an environment. 
	We then prove that this richer value function allows us to achieve zero-shot composition when an agent is given a new task.
	
	\item \textit{Zero-shot composition:} We improve on previous work \citep{vanniekerk19} by showing zero-shot logical composition of tasks without any additional assumptions. This is an important result as it enables lifelong-learning agents to solve a super-exponentially increasing number of tasks as the number of base tasks they learn increase.
\end{enumerate}


We illustrate our approach in the Four Rooms domain \citep{sutton99}, where an agent first learns to reach a number of rooms, after which it can then optimally solve any task expressible in the Boolean algebra.
We then demonstrate composition in a high-dimensional video game environment, where an agent first learns to collect different objects, and then composes these abilities to solve complex tasks immediately.
Our results show that, even when function approximation is required, an agent can leverage its existing skills to solve new tasks without further learning.

\section{Preliminaries} \label{sec:pre}

We consider tasks modelled by Markov Decision Processes (MDPs).
An MDP is defined by the tuple $(\mathcal{\state}, \action, \dynamics, \reward)$, where 
\begin{enumerate*}[label=(\roman*)]
  \item $\state$ is the state space,
  \item $\action$ is the action space,
  \item $\dynamics$ is a Markov transition kernel $(s, a) \mapsto \rho_{(s, a)}$ from $\state \times \action$ to $\state$, and
  \item $\reward$ is the real-valued reward function bounded by $[\rmin, \rmax]$.
\end{enumerate*}
In this work, we focus on stochastic shortest path problems \citep{bertsekas91}, which model tasks in which an agent must reach some goal.
We therefore consider the class of undiscounted MDPs with an absorbing set $\goals \subseteq \state$.

The goal of the agent is to compute a Markov policy $\pi$ from $\state$ to $\action$ that optimally solves a given task.
A given policy $\pi$ induces a value function $\vpi(s) = \E_\pi \left[ \sum_{t=0}^{\infty} r(s_t, a_t) \right]$, specifying the expected return obtained under $\pi$ starting from state $s$.\footnote{Since we consider undiscounted MDPs, we can ensure the value function is bounded by augmenting the state space with a virtual state $\virtual$ such that $\dynamics_{(s,a)}({\virtual}) = 1$ for all $(s, a) \in \goals \times \action$, and $\reward = 0$ after reaching $\virtual$.}
The \textit{optimal} policy $\pistar$ is the policy that obtains the greatest expected return at each state: $\v^{\pistar}(s) = \vstar(s) = \max_\pi \vpi (s)$ for all $s \in \state$.
A related quantity is the $Q$-value function, $\qpi(s, a)$, which defines the expected return obtained by executing $a$ from $s$, and thereafter following $\pi$.
Similarly, the optimal $Q$-value function is given by $\qstar(s, a) = \max_\pi \qpi(s, a)$ for all $(s, a) \in \state \times \action$.
Finally, we denote a \textit{proper policy} to be a policy that is guaranteed to eventually reach the absorbing set $\goals$ \citep{james06,vanniekerk19}. 
We assume the value functions for improper policies---those that never reach absorbing states---are unbounded from below.

\section{Boolean Algebras for Tasks and Value Functions}

In this section, we develop the notion of a Boolean task algebra.
This  formalises the notion of task conjunction  ($\wedge$) and disjunction ($\vee$) introduced in previous work \citep{haarnoja18,vanniekerk19,hunt19}, while additionally introducing the concept of negation ($\neg$).  
We then show that, having solved a series of base tasks, an agent can use its knowledge to solve tasks expressible as a Boolean expression over those tasks, without any further learning.\footnote{Owing to space constraints, all proofs are presented in the supplementary material.}

We consider a family of related MDPs $\tasks$ restricted by the following assumption:
\begin{assump}[\cite{vanniekerk19}] 
For all tasks in a set of tasks $\tasks$,
\begin{enumerate*}[label=(\roman*)]
    \item the tasks share the same state space, action space and transition dynamics,
    \item the transition dynamics are deterministic, and
    \item the reward functions between tasks differ only on the absorbing set $\goals$.  For all non-terminal states, we denote the reward $\constantr$ to emphasise that it is constant across tasks.
\end{enumerate*}
\label{assump:1}
\end{assump}

Assumption~\ref{assump:1} represents the family of tasks where the environment remains the same but the goals and their desirability may vary. This is typically true for robotic navigation and manipulation tasks where there are multiple achievable goals, the goals we want the robot to achieve may vary, and how desirable those goals are may also vary. 
Although we have placed restrictions on the reward functions, the above formulation still allows for a large number of tasks to be represented. 
Importantly, sparse rewards can be formulated under these restrictions. 
In practice, however, all of these assumptions can be violated with minimal impact.
In particular, additional experiments in the supplementary material show that even for tasks with stochastic transition dynamics and dense rewards, and which differ in their terminal states, our composition approach still results in policies that are either identical or very close to optimal.

\subsection{A Boolean Algebra for Tasks}

An abstract Boolean algebra is a set $\mathcal{B}$ equipped with operators $\neg, \vee, \wedge$ that satisfy the Boolean axioms of 
\begin{enumerate*}[label=(\roman*)]
    \item idempotence,
    \item commutativity,
    \item associativity,
    \item absorption,
    \item distributivity,
    \item identity, and
    \item complements.\footnote{We provide a description of these axioms in the supplementary material.}
\end{enumerate*}

We first define the $\neg, \vee$, and $\wedge$ operators over a set of tasks.

\begin{mydef}
Let $\tasks$ be a set of tasks which adhere to Assumption~\ref{assump:1},  with $\mbig,\msmall \in \tasks$ such that 

\noindent\begin{minipage}{.5\linewidth}
\begin{align*}
  \reward_{\mbig}:~ \state \times \action &\to \sR \\
  (s, a) &\mapsto \max\limits_{M \in \tasks} r_{M}(s, a)
\end{align*}
\end{minipage}%
\begin{minipage}{.5\linewidth}
\begin{align*}
  \reward_{\msmall}:~\state \times \action &\to \sR \\
  (s, a) &\mapsto \min\limits_{M \in \tasks} r_{M}(s, a)
\end{align*}
\end{minipage}

Define the $\neg, \vee$, and $\wedge$ operators over $\tasks$ as

\begin{fleqn}
\begin{ceqn}
\setlength{\abovedisplayskip}{2pt}
\setlength{\belowdisplayskip}{2pt}
\begin{alignat*}{3}
  \neg :~\tasks &\to \tasks              \\
   M &\mapsto (\state, \action, \dynamics, \rneg{M}), \text{ where }    & \rneg{M}:~\state \times \action &\to \sR  \\
    & & (s, a)  &\mapsto \left(\reward_{\mbig}(s, a) + \reward_{\msmall}(s, a) \right) - \reward_M(s, a)
\end{alignat*}
\begin{alignat*}{3}
  \vee :~\tasks \times \tasks &\to \tasks              \\
   (M_1, M_2) &\mapsto (\state, \action, \dynamics, \ror{M_1}{M_2}), \text{ where }    & \ror{M_1}{M_2}:~\state \times \action &\to \sR  \\
    & & (s, a)  &\mapsto \max\{r_{M_1}(s, a), r_{M_2}(s, a)\}
\end{alignat*} 
\begin{alignat*}{3}
  \wedge :~\tasks \times \tasks &\to \tasks              \\
   (M_1, M_2) &\mapsto (\state, \action, \dynamics, \rand{M_1}{M_2}), \text{ where }    & \rand{M_1}{M_2}:~\state \times \action &\to \sR  \\
    & & (s, a)  &\mapsto \min\{r_{M_1}(s, a), r_{M_2}(s, a)\}
\end{alignat*} 
\end{ceqn}
\end{fleqn}
\end{mydef}

In order to formalise the logical composition of tasks under the Boolean algebra structure, it is necessary that the tasks have a Boolean nature. 
This is enforced by the following sparseness assumption:

\begin{assump} 
For all tasks in a set of tasks $\tasks$ which adhere to Assumption~\ref{assump:1}, the set of possible terminal rewards consists of only two values. That is, for all $(g, a)$ in $\goals \times \action$, we have that $\reward(g, a) \in \{ \rsmall, \rbig \} \subset [\rmin, \rmax]$ with $\rsmall \leq \rbig$.\footnote{While Assumption 2 is necessary to establish the Boolean algebra, we show in Theorem~\ref{thm:3} that it is not required for zero-shot negation, disjunction, and conjunction.}
\label{assump:2}
\end{assump}

Given the above definitions and the restrictions placed on the set of tasks we consider, we can now define a Boolean algebra over a set of tasks.

\begin{theorem}
Let $\tasks$ be a set of tasks which adhere to Assumption~\ref{assump:2}. Then $(\tasks,\vee,\wedge,\neg,\mbig,\msmall)$ is a Boolean algebra.

\label{thm:taskalgebra}
\end{theorem}

Theorem~\ref{thm:taskalgebra} allows us to compose existing tasks together to create new tasks in a principled way. 
Figure~\ref{fig:rewards} illustrates the semantics for each of the Boolean operators in a simple environment.










\subsection{Extended Value Functions}

The reward and value functions described in Section~\ref{sec:pre} are insufficient to solve tasks specified by the Boolean algebra above. To understand why, consider two tasks that have multiple different goals, but at least one common goal.
Clearly, there is a meaningful conjunction between them---namely,  achieving the common goal. 
Now consider an agent that learns standard value functions for both tasks, and which is then required to solve their conjunction without further learning.
Note that this is impossible in general, since the regular value function for each task only represents the value of each state with respect to the \textit{nearest} goal.
That is, for all states where the nearest goal for each task is \textit{not} the common goal, the agent has no information about that common goal. 
We therefore define extended versions of the reward and value function such that the agent is able to learn the value of achieving all goals, and not simply the nearest one. These are given by the following two definitions: 

\begin{mydef}
The \textit{extended} reward function $\rbar: \state \times \goals \times \action \to \mathbb{R}$ is given by the mapping
\begin{equation}
    (s, g, a) \mapsto \begin{cases}
\rbarmin & \text{if } g \neq s \in \goals\\
r(s, a) &\text{otherwise},
\end{cases}
\end{equation}
where $\rbarmin\leq \min \{\rmin, (\rmin - \rmax)D \}$, and $D$ is the diameter of the MDP \citep{jaksch10}.\footnote{The diameter is defined as $D = \max_{s \neq s^\prime \in \state} \min_{\pi} \E \left[ T(s^\prime | \pi, s) \right] $, where $T$ is the number of timesteps required to first reach $s^\prime$ from $s$ under $\pi$.} 
\end{mydef}

Because we require that tasks share the same transition dynamics, we also require that the absorbing set of states is shared. 
Thus the extended reward function adds the extra constraint that, if the agent enters a terminal state for a \textit{different} task, it should receive the largest penalty possible.
In practice, we can simply set $\rbarmin$ to be the lowest finite value representable by the data type used for the value function.

\begin{mydef}
The \textit{extended} Q-value  function $\qbar: \state \times \goals \times \action \to \mathbb{R}$ is given by the mapping
\begin{equation}
    (s, g, a) \mapsto \rbar(s, g, a) + \int_{\state} \bar{\v}^{\bar{\pi}}(s^\prime, g) \dynamics_{(s, a)} (ds^\prime),
\end{equation}
\end{mydef}
where $\bar{\v}^{\bar{\pi}}(s, g) = \E_{\bar{\pi}} \left[ \sum_{t=0}^{\infty} \rbar(s_t, g, a_t) \right]$.

The extended Q-value function is similar to DG functions \citep{kaelbling93} which also learn how to achieve all goals, except here we use task-dependent reward functions as opposed to measuring distance between states. \citet{veeriah2018many} refers to this idea of learning to achieve all goals in an environment as ``mastery''. We can see that the definition of extended Q-value functions encapsulates this notion for arbitrary task rewards. 


The standard reward functions and value functions can be recovered from their extended versions through the following lemma.

\begin{lemma}
Let $\reward_{M}, \rbar_{M}, \qstar_M, \qstarbar_M$ be the reward function, extended reward function, optimal Q-value function, and optimal extended Q-value function for a task $M$ in $\tasks$. 
Then for all $(s, a)$ in $\state \times \action$, we have
\begin{enumerate*}[label=(\roman*)]
    \item $\reward_M(s, a) = \max\limits_{g \in \goals} \rbar_M(s, g, a)$, and
    \item $\qstar_M(s, a) = \max\limits_{g \in \goals} \qstarbar_M(s, g, a)$.
\end{enumerate*}
\label{lem:1}
\end{lemma}

\begin{figure*}[t!]
    \centering
    \begin{subfigure}[t]{0.15\textwidth}
        \centering
        \includegraphics[height=0.9in]{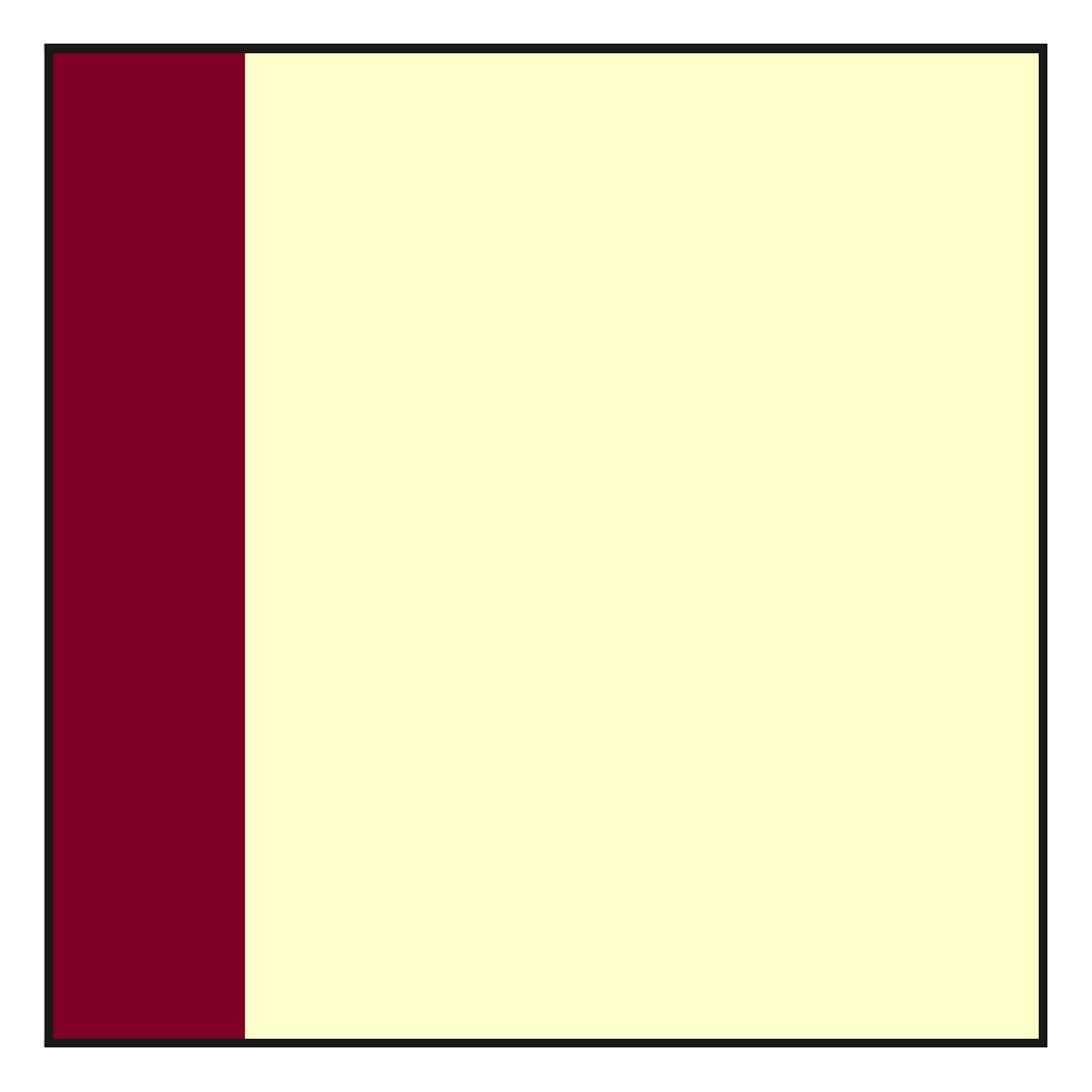}
        \caption{$\reward_{M_\text{LEFT}}$}
    \end{subfigure}%
    ~ 
        \begin{subfigure}[t]{0.15\textwidth}
        \centering
        \includegraphics[height=0.9in]{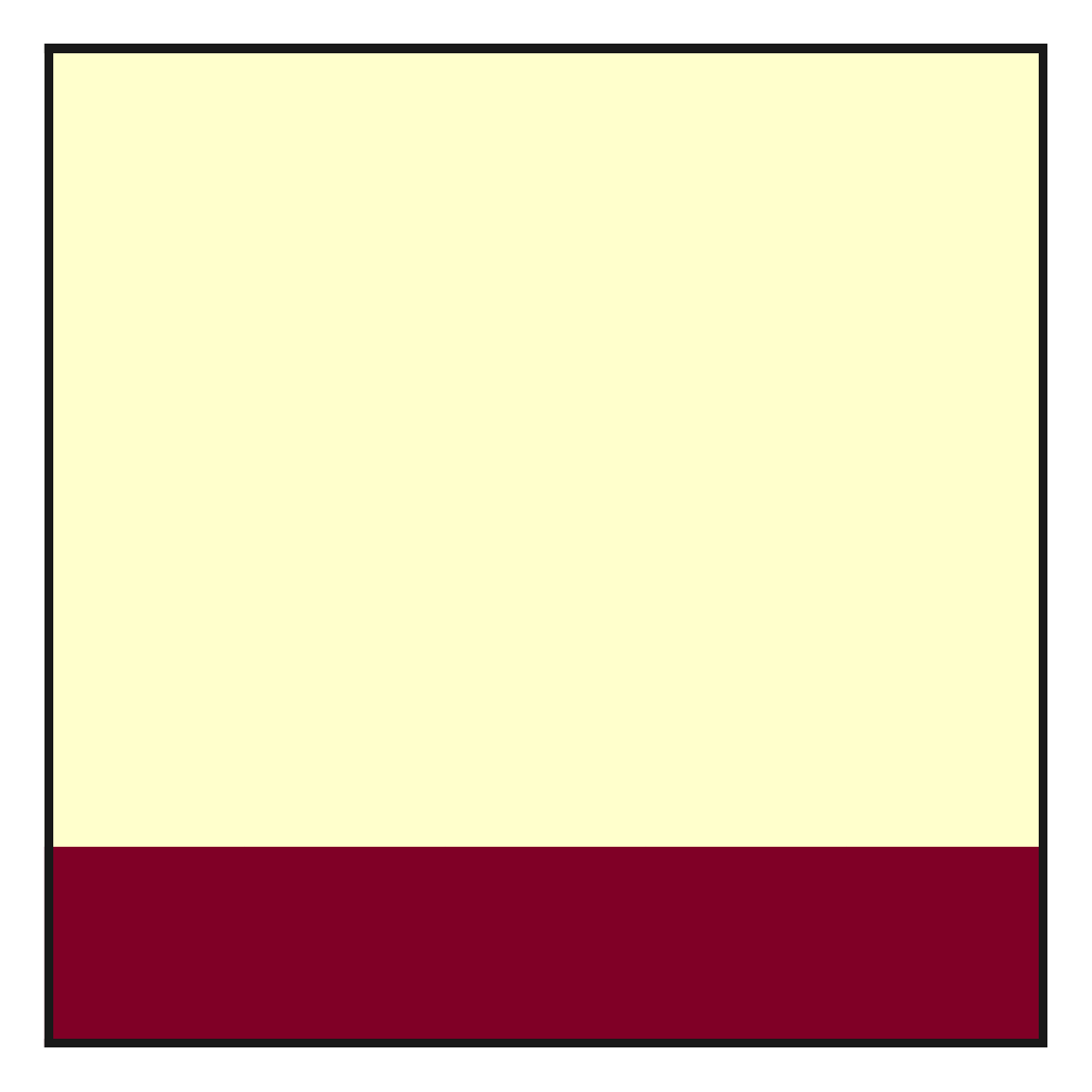}
        \caption{$\reward_{M_\text{DOWN}}$}
    \end{subfigure}%
    ~ 
    \begin{subfigure}[t]{0.15\textwidth}
        \centering
        \includegraphics[height=0.9in]{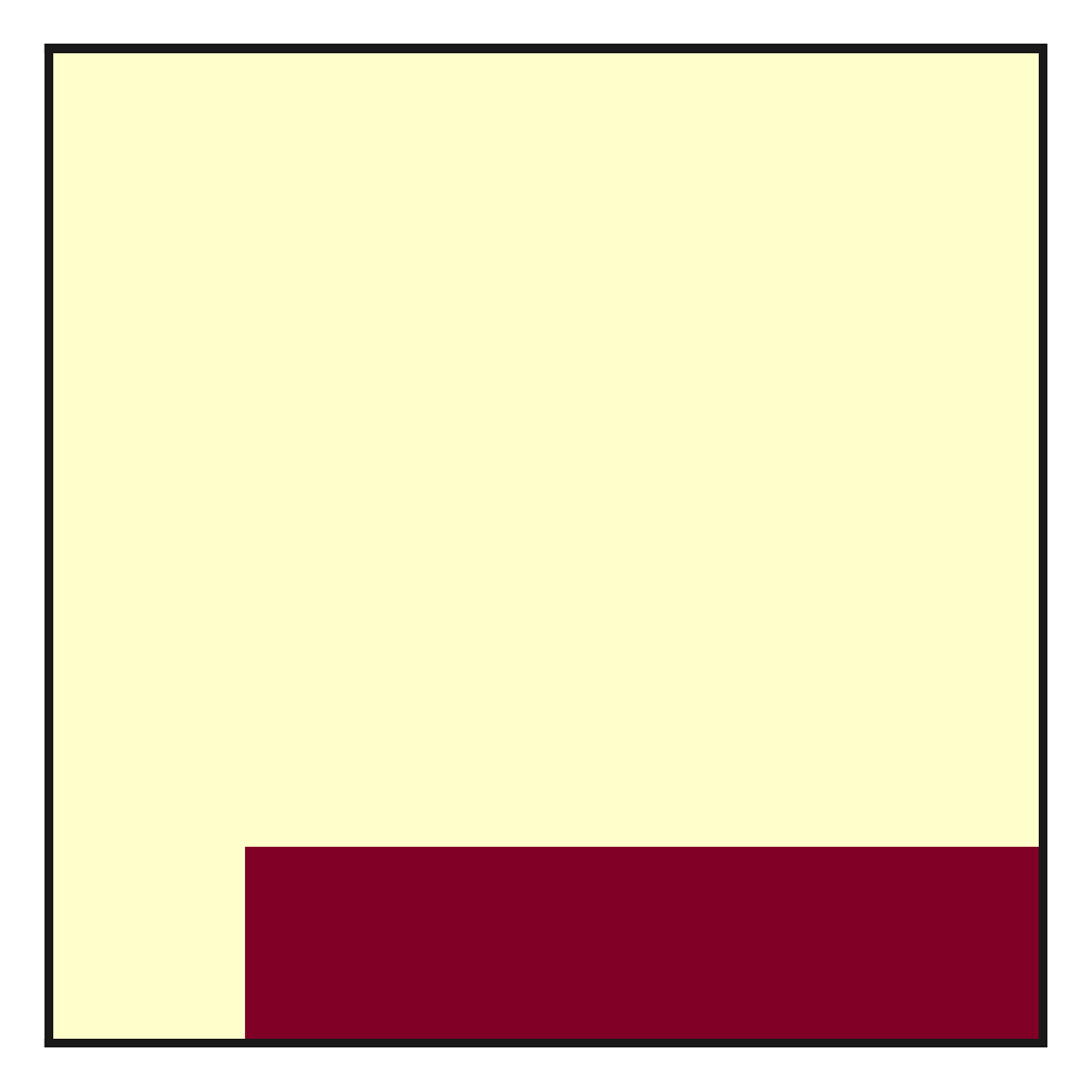}
        \caption{$\reward_{M_{\neg \text{LEFT}}}$}
    \end{subfigure}%
    ~ 
    \begin{subfigure}[t]{0.15\textwidth}
        \centering
        \includegraphics[height=0.9in]{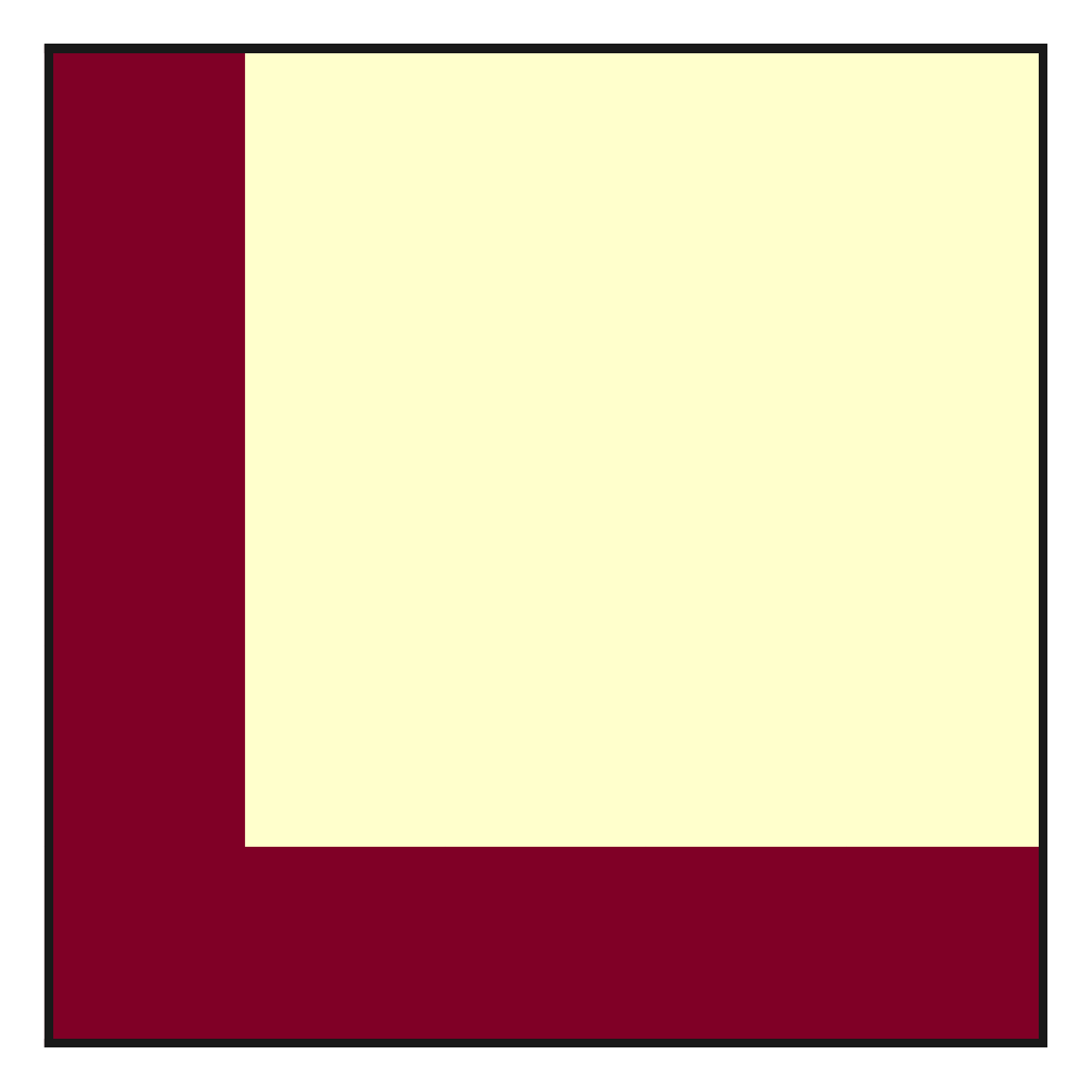}
        \caption{Disjunction}
    \end{subfigure}%
    ~ 
    \begin{subfigure}[t]{0.15\textwidth}
        \centering
        \includegraphics[height=0.9in]{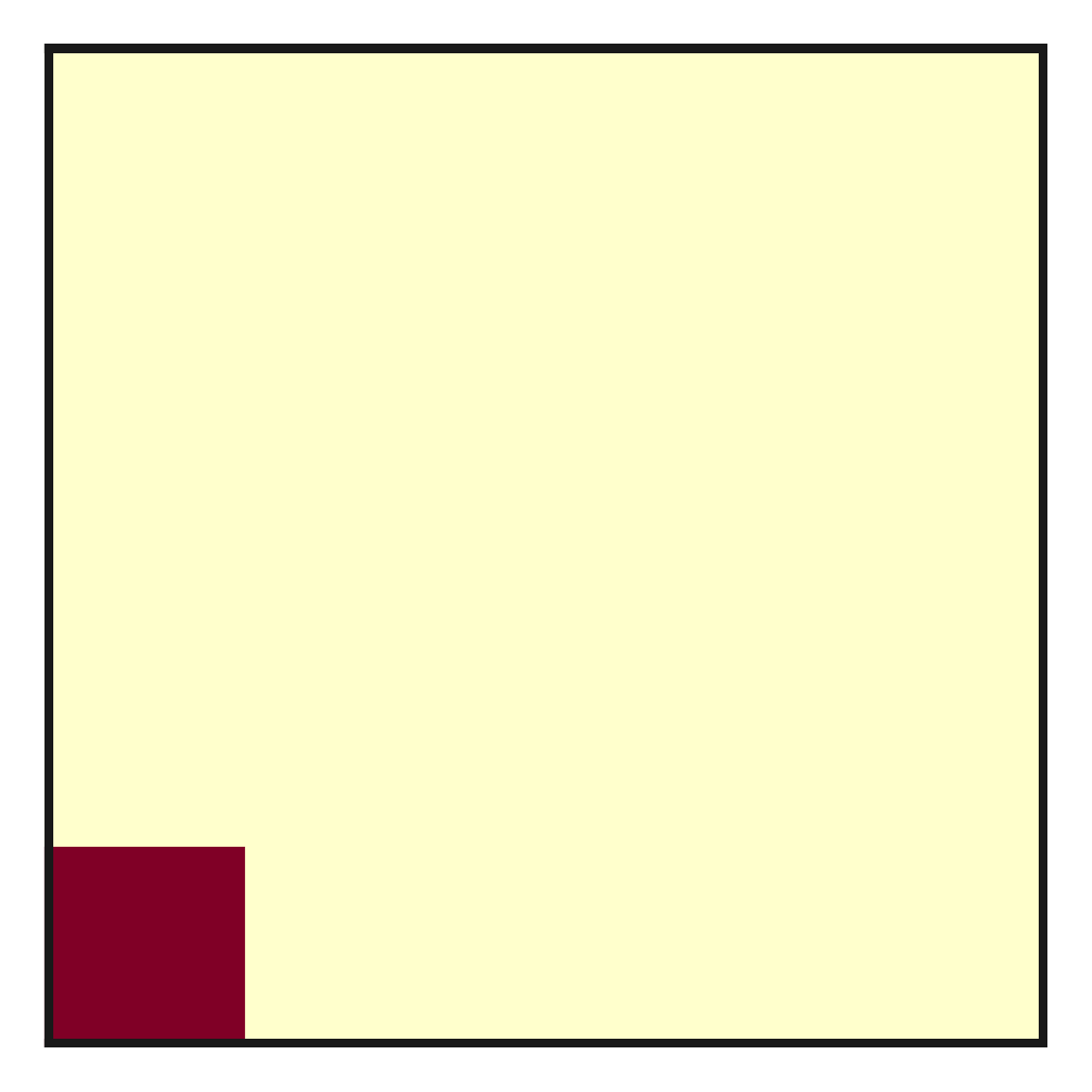}
        \caption{Conjunction}
    \end{subfigure}%
    ~ 
    \begin{subfigure}[t]{0.15\textwidth}
        \centering
        \includegraphics[height=0.9in]{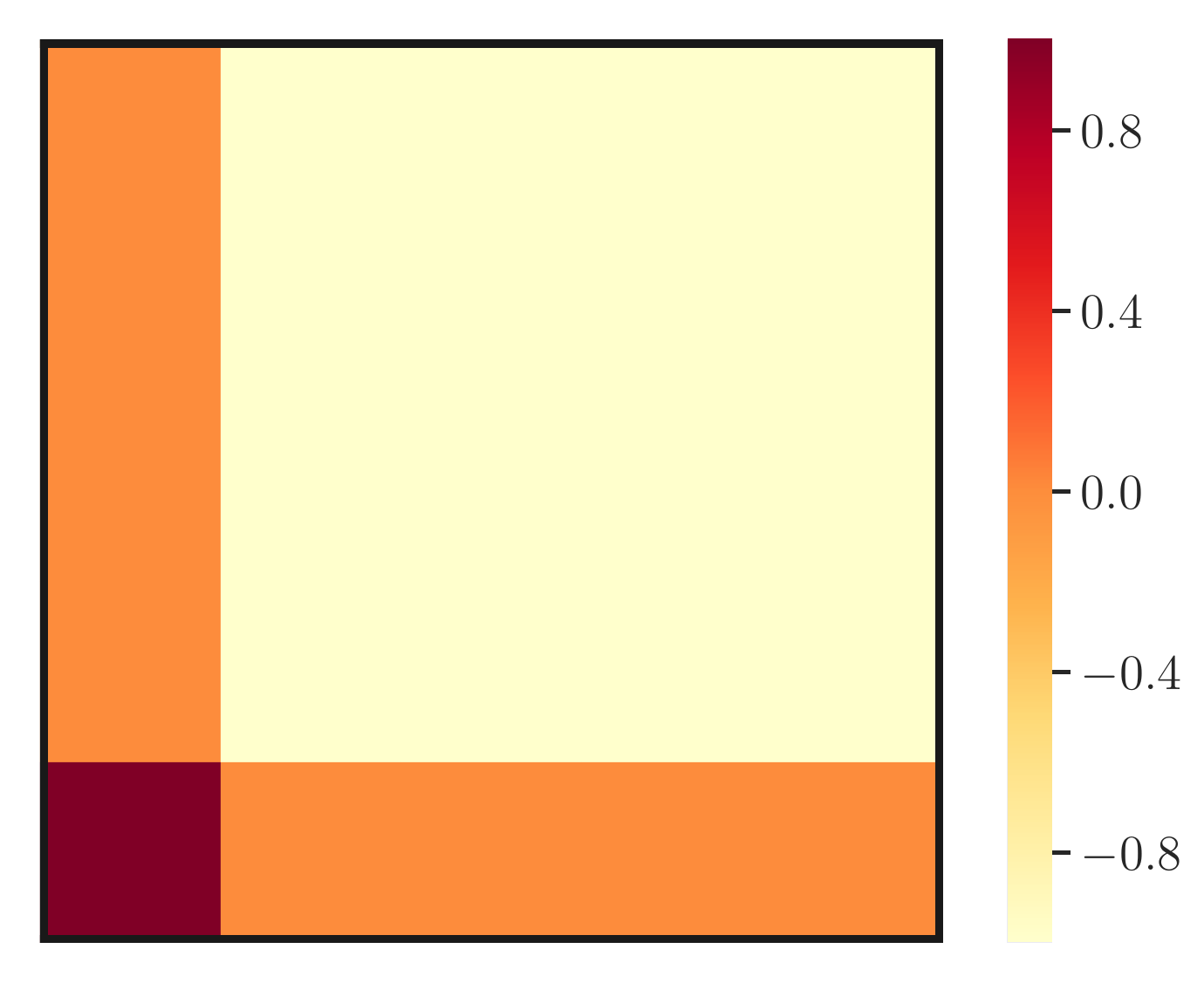}
        \caption{Average}
    \end{subfigure}%
    \caption{Consider two tasks, $M_\text{LEFT}$ and $M_\text{DOWN}$, in which an agent must navigate to the left and bottom regions of an $xy$-plane respectively. From left to right we plot the reward for entering a region of the state space for the individual tasks, the negation of $M_\text{LEFT}$, and  the union (disjunction) and intersection (conjunction) of tasks. For reference, we also plot the average reward function, which has been used in previous work to approximate the conjunction operator \citep{haarnoja18,hunt19,vanniekerk19}. Note that by averaging reward, terminal states that are not in the intersection are erroneously given rewards.}
    \label{fig:rewards}
\end{figure*}

In the same way, we can also recover the optimal policy from these extended value functions by first applying Lemma~\ref{lem:1}, and acting greedily with respect to the resulting value function.

\begin{lemma}
Denote $\state^{-} = \state \setminus \goals  $ as the non-terminal states of $\tasks$.  Let $M_1, M_2 \in \tasks$, and let each $g$ in $\goals$ define MDPs $M_{1,g}$ and $M_{2,g}$ with reward functions 
\[
r_{M_{1,g}} \vcentcolon = \rbar_{M_1}(s,g,a) \text{ and } r_{M_{2,g}} \vcentcolon = \rbar_{M_2}(s,g,a) \text{ for all } (s,a) \text{ in } \state \times \action.
\]
Then for all $g$ in $\goals$ and $s$ in $\state^{-}$,
\[
\pistar_g(s) \in \argmax\limits_{a \in \action}\qstar_{M_{1, g}}(s, a) \text{ iff } \pistar_g(s) \in \argmax\limits_{a \in \action}\qstar_{M_{2, g}}(s, a).
\]
\label{lem:2}
\end{lemma}

Combining Lemmas \ref{lem:1} and \ref{lem:2}, we can extract the greedy action from the extended value function by first maximising over goals, and then selecting the maximising action: $\pistar(s) \in \argmax_{a \in \action} \max_{g \in \goals} \qstarbar(s, g, a)$.
If we consider the extended value function to be a set of standard value functions (one for each goal), then this is equivalent to first performing generalised policy improvement \citep{barreto17}, and then selecting the greedy action.

Finally, much like the regular definition of value functions, the extended Q-value function can be written as the sum of rewards received by the agent until first encountering a terminal state.

\begin{corollary}
Denote $\gstar$ as the sum of rewards starting from $s$ and taking action $a$ up until, but not including, $g$. Then let $M \in \tasks$ and $\qstarbar_M$ be the extended Q-value function. Then for all $s \in \state, g \in \goals, a \in \action$, there exists a $\gstar \in \R$ such that
\begin{equation*}
    \qstarbar_M(s, g, a) = \gstar + \rbar_M(s^\prime,g, a^\prime), \text{ where } s^\prime \in \goals \text{ and } a^\prime = \argmax_{b \in \action} \rbar_M(s^\prime,g, b).
\end{equation*}
\label{cor:1}
\end{corollary}

\subsection{A Boolean Algebra for Value Functions}

In the same manner we constructed a Boolean algebra over a set of tasks, we can also do so for a set of optimal extended Q-value functions for the corresponding tasks.

\begin{mydef}
Let $\goalq$ be the set of optimal extended $\bar{Q}$-value functions for tasks in $\tasks$ which adhere to Assumption~\ref{assump:1},  with $\qstarbarsmall, \qstarbarbig \in \goalq$ the optimal $\bar{Q}$-functions for the tasks $\msmall, \mbig \in \tasks$ .Define the $\neg, \vee$, and $\wedge$ operators over $\goalq$ as,

\begin{fleqn}
\begin{ceqn}
\setlength{\abovedisplayskip}{2pt}
\setlength{\belowdisplayskip}{2pt}
\begin{alignat*}{3}
  \neg :~\goalq &\to \goalq              \\
   \qstarbar &\mapsto \neg \qstarbar, \text{ where }    & \neg \qstarbar:~\state \times \goals \times \action &\to \sR  \\
    & & (s, g, a)  &\mapsto \left(\qstarbarbig(s,g, a) + \qstarbarsmall(s, g, a) \right) - \qstarbar(s, g, a)
\end{alignat*} 
\begin{alignat*}{3}
  \vee :~\goalq \times \goalq &\to \goalq              \\
   (\qstarbar_1, \qstarbar_2) &\mapsto \qstarbar_1 \vee \qstarbar_2, \text{ where }    &  \qstarbar_1 \vee \qstarbar_2:~\state \times \goals \times \action &\to \sR \\
    & & (s, g, a)  &\mapsto \max\{\qstarbar_{1}(s, g, a), \qstarbar_{2}(s, g, a)\}
\end{alignat*} 
\begin{alignat*}{3}
  \wedge :~\goalq  \times \goalq &\to \goalq              \\
   (\qstarbar_1, \qstarbar_2) &\mapsto \qstarbar_1 \wedge \qstarbar_2, \text{ where }    & \qstarbar_1 \wedge \qstarbar_2:~\state \times \goals \times \action &\to \sR  \\
    & & (s, g, a)  &\mapsto \min\{\qstarbar_{1}(s, g, a), \qstarbar_{2}(s, g, a)\}
\end{alignat*} 
\end{ceqn}
\end{fleqn}
\end{mydef}

\begin{theorem}
Let $\goalq$ be the set of optimal extended $\bar{Q}$-value functions for tasks in $\tasks$ which adhere to Assumption~\ref{assump:2}. Then $(\goalq,\vee,\wedge,\neg,\qstarbarbig,\qstarbarsmall)$ is a Boolean Algebra.
\label{thm:valuealgebra}
\end{theorem}

\subsection{Between Task and Value Function Algebras}

Having established a Boolean algebra over tasks and extended value functions, we finally show that there exists an equivalence between the two. 
As a result, if we can write down a task under the Boolean algebra, we can immediately write down the optimal value function for the task.

\begin{theorem}
Let $\goalq$ be the set of optimal extended $\bar{Q}$-value functions for tasks in $\tasks$ which adhere to Assumption~\ref{assump:1}. Then for all $M_1, M_2 \in \tasks$, we have 
\begin{enumerate*}[label=(\roman*)]
    \item $\qstarbar_{\neg M_1} = \neg \qstarbar_{M_1}$,
    \item $\qstarbar_{M_1 \vee M_2} = \qstarbar_{M_1} \vee \qstarbar_{M_2}$, and 
    \item $\qstarbar_{M_1 \wedge M_2} = \qstarbar_{M_1} \wedge \qstarbar_{M_2}$.
\end{enumerate*}

\label{thm:3}
\end{theorem}

\begin{corollary}
Let $\mathcal{F}~: \tasks \to \goalq$ be any map from $\tasks$ to $\goalq$ such that $\mathcal{F}(M) = \qstarbar_M$ for all $M$ in $\tasks$. Then $\mathcal{F}$ is a homomorphism between $(\tasks,\vee,\wedge,\neg,\mbig,\msmall)$ and $(\goalq,\vee,\wedge,\neg,\qstarbarbig,\qstarbarsmall)$.
\label{cor:2}
\end{corollary}

Theorem~\ref{thm:3} shows that we can provably achieve zero-shot negation, disjunction, and conjunction provided Assumption~\ref{assump:1} is satisfied. Corollary~\ref{cor:2} extends this result by showing that the task and value function algebras are in fact homomorphic, which implies zero-shot composition of arbitrary combinations of negations, disjunctions, and conjunctions.

\section{Zero-shot Transfer Through Composition} \label{sec:rooms}

We can use the theory developed in the previous sections to perform zero-shot transfer by first learning extended value functions for a set of base tasks, and then composing them to solve new tasks expressible under the Boolean algebra. 
To demonstrate this, we conduct a series of experiments in the Four Rooms domain \citep{sutton99}, where an agent must navigate a grid world to a particular location.
The agent can move in any of the four cardinal directions at each timestep, but colliding with a wall leaves the agent in the same location. We add a 5th action for ``stay'' that the agent chooses to achieve goals. A goal position only becomes terminal if the agent chooses to stay in it. 
The transition dynamics are deterministic, and rewards are $-0.1$ for all non-terminal states, and $2$ at the goal.

\subsection{Learning Base Tasks}

We use a modified version of Q-learning \citep{watkins89} to learn the extended Q-value functions described previously.
Our algorithm differs in a number of ways from standard Q-learning: we keep track of the set of terminating states seen so far, and at each timestep we update the extended Q-value function with respect to both the current state and action, as well as all goals encountered so far.
We also use the definition of the extended reward function, and so if the agent encounters a terminal state of a different task, it receives reward $\rbarmin$.  
The full pseudocode is listed in the supplementary material.

If we know the set of goals (and hence potential base tasks) upfront, then it is easy to select a minimal set of base tasks that can be composed to produce the largest number of composite tasks.
We first assign a Boolean label to each goal in a table, and then use the columns of the table as base tasks. The goals for each base task are then those goals with value $1$ according to the table.
In this domain, the two base tasks we select are $M_{\text{T}}$, which requires that the agent visit either of the top two rooms, and $M_{\text{L}}$, which requires visiting the two left rooms. 
We illustrate this selection procedure in the supplementary material.

\subsection{Boolean Composition}

Having learned the optimal extended value functions for our base tasks, we can now leverage Theorems~\ref{thm:taskalgebra}--\ref{thm:3} to solve new tasks with no further learning. 
Figure~\ref{fig:composition} illustrates this composition, where an agent is able to immediately solve complex tasks such as exclusive-or.
We illustrate a few composite tasks here, but note that in general, if we have $K$ base tasks, then a Boolean algebra allows for $2^{2^K}$  new tasks to be constructed.
Thus having trained on only two tasks, our agent has enough information to solve a total of 16 composite tasks.

\begin{figure*}[h!]
    \centering
    \begin{subfigure}[t]{0.15\textwidth}
        \centering
        \includegraphics[height=0.85in]{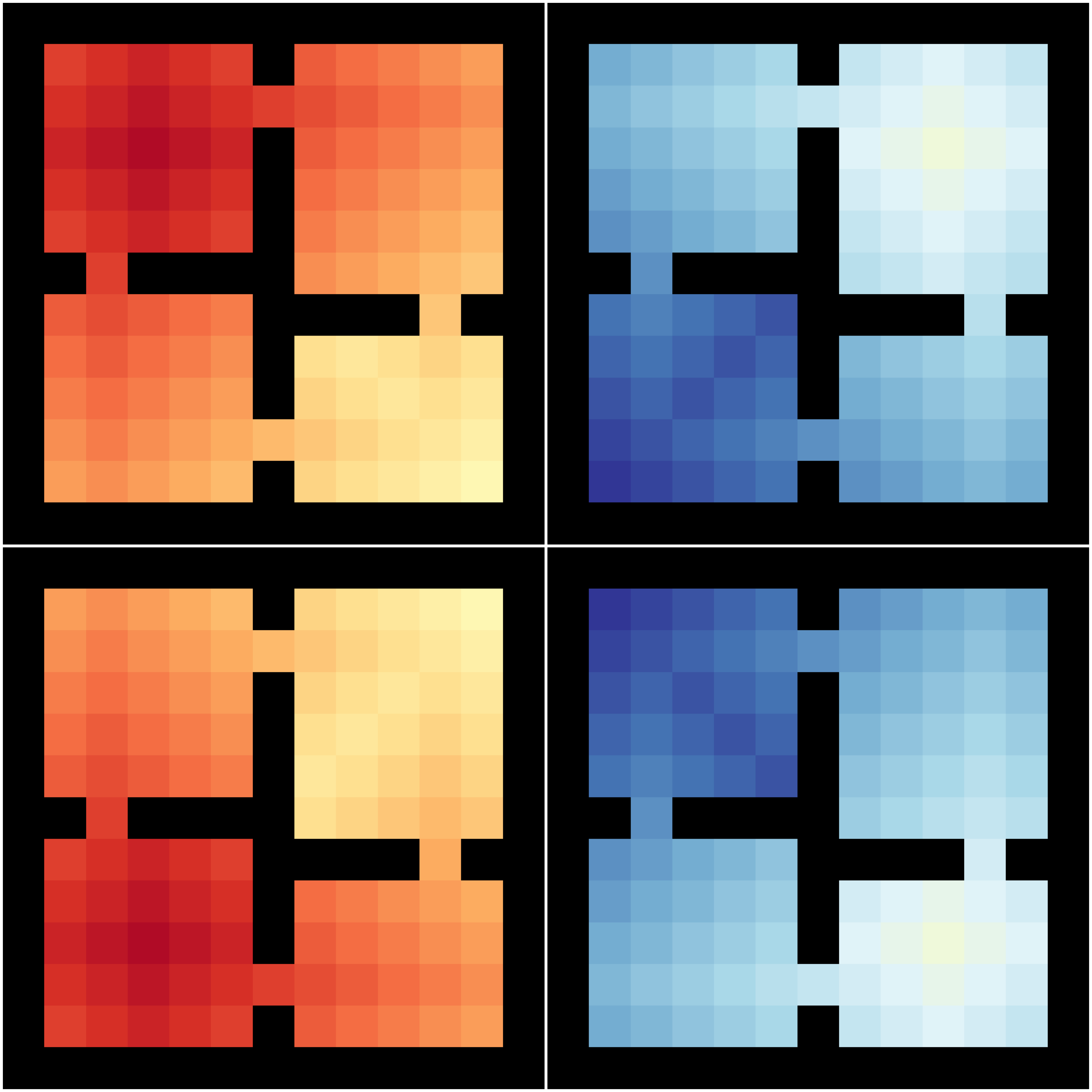}
        \includegraphics[height=0.85in]{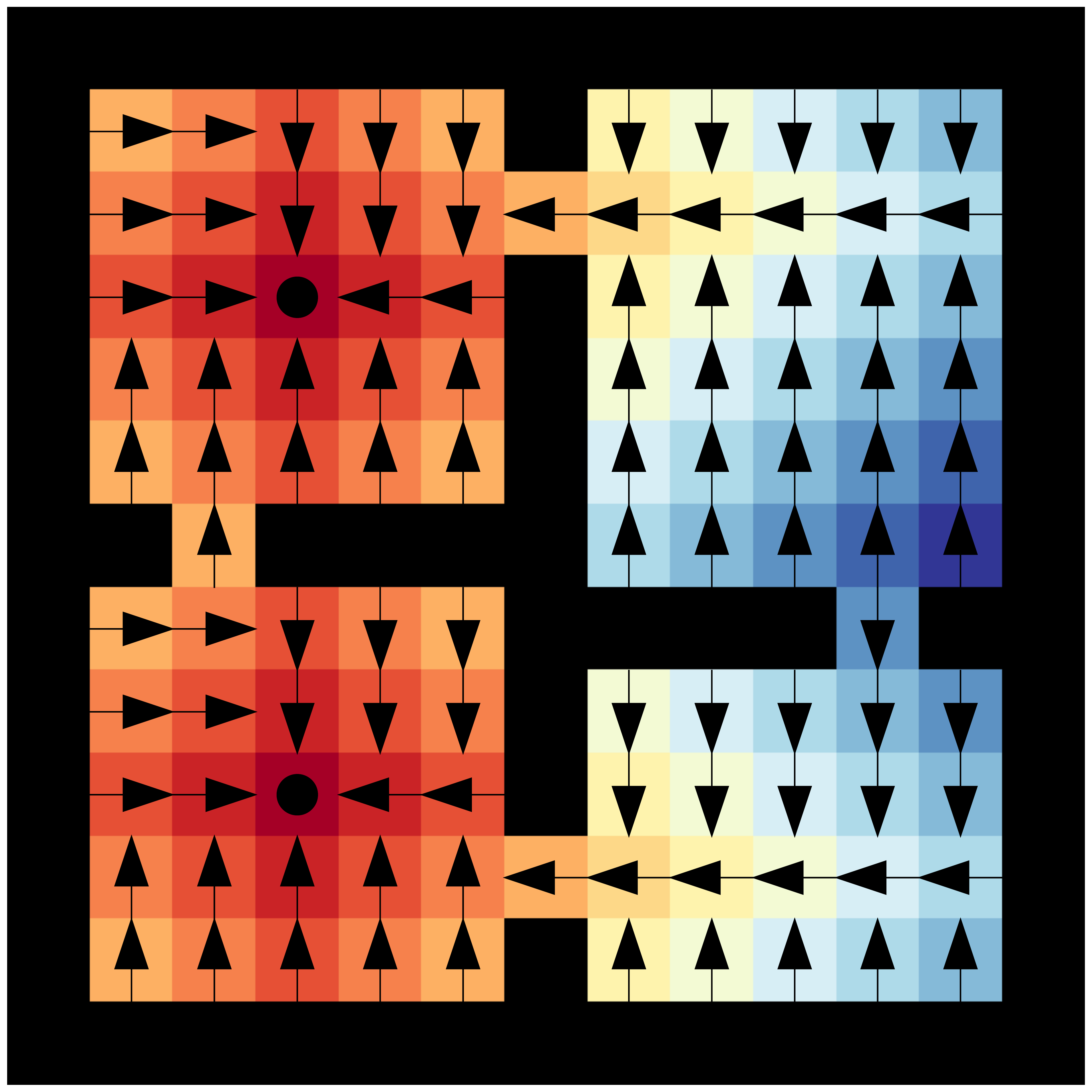}
        \caption{$M_{\text{L}}$}
        \label{fig:a}
    \end{subfigure}%
    ~ 
        \begin{subfigure}[t]{0.15\textwidth}
        \centering
        \includegraphics[height=0.85in]{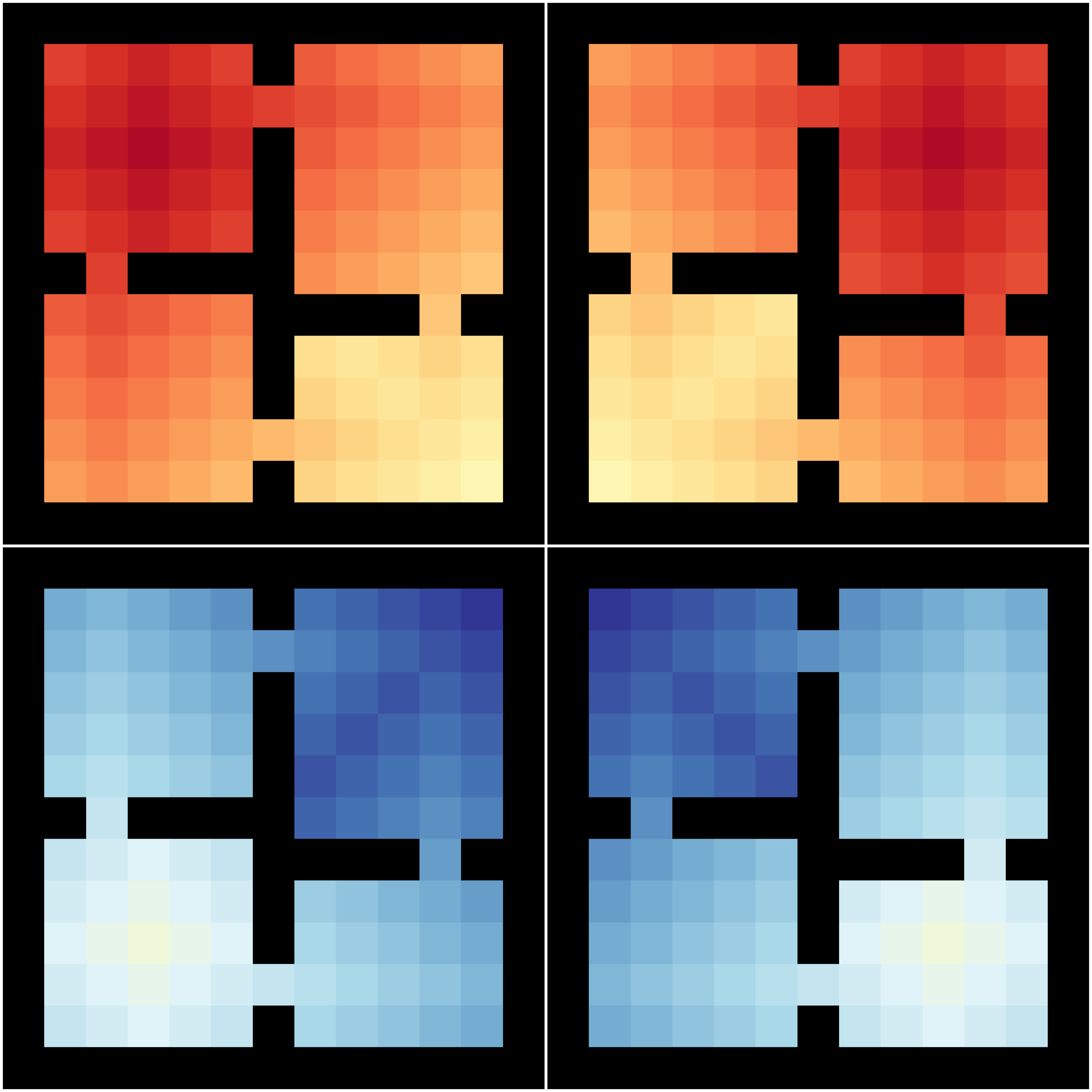}
        \includegraphics[height=0.85in]{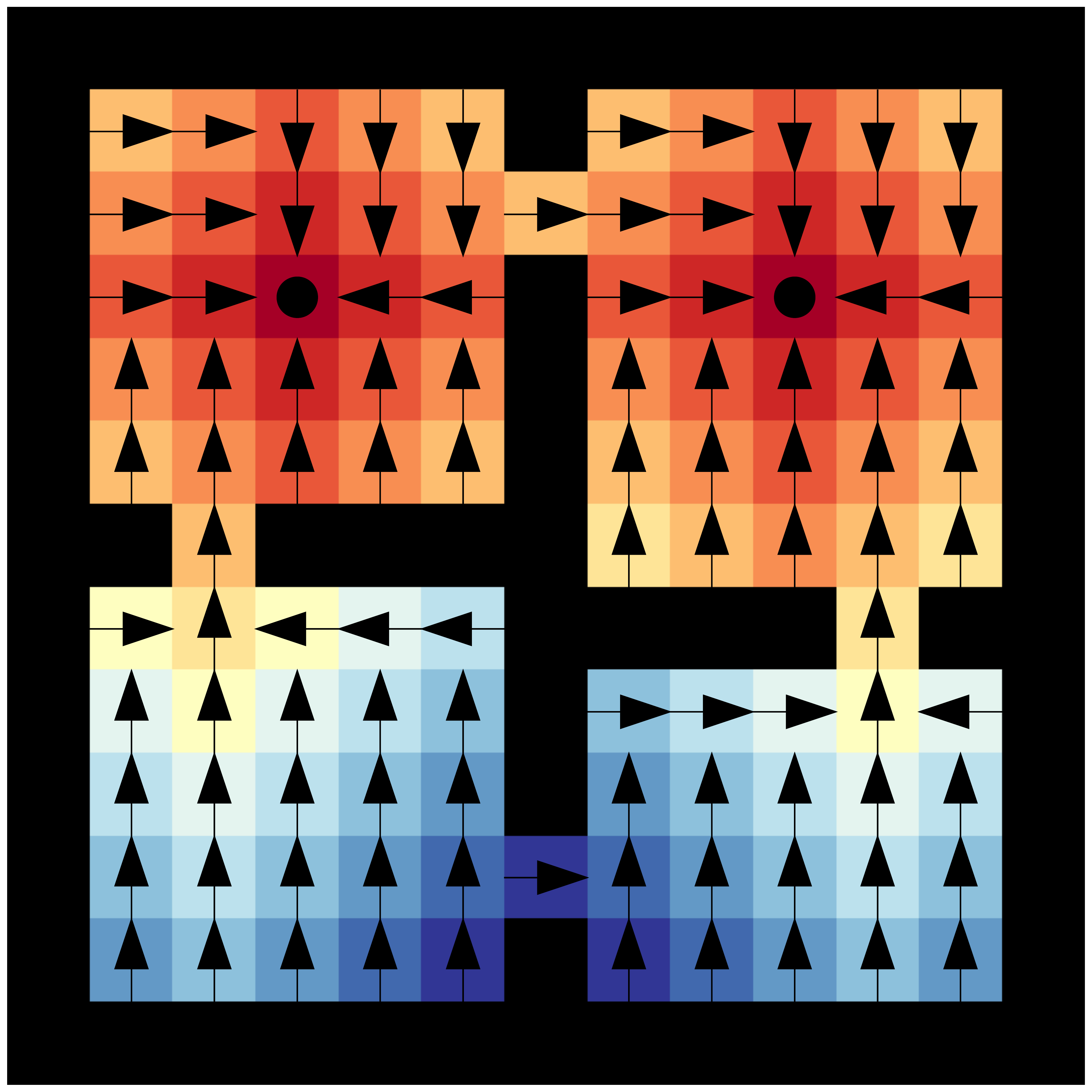}
        \caption{$M_{\text{T}}$}
        \label{fig:b}
    \end{subfigure}%
    ~ 
    \begin{subfigure}[t]{0.15\textwidth}
        \centering
        \includegraphics[height=0.85in]{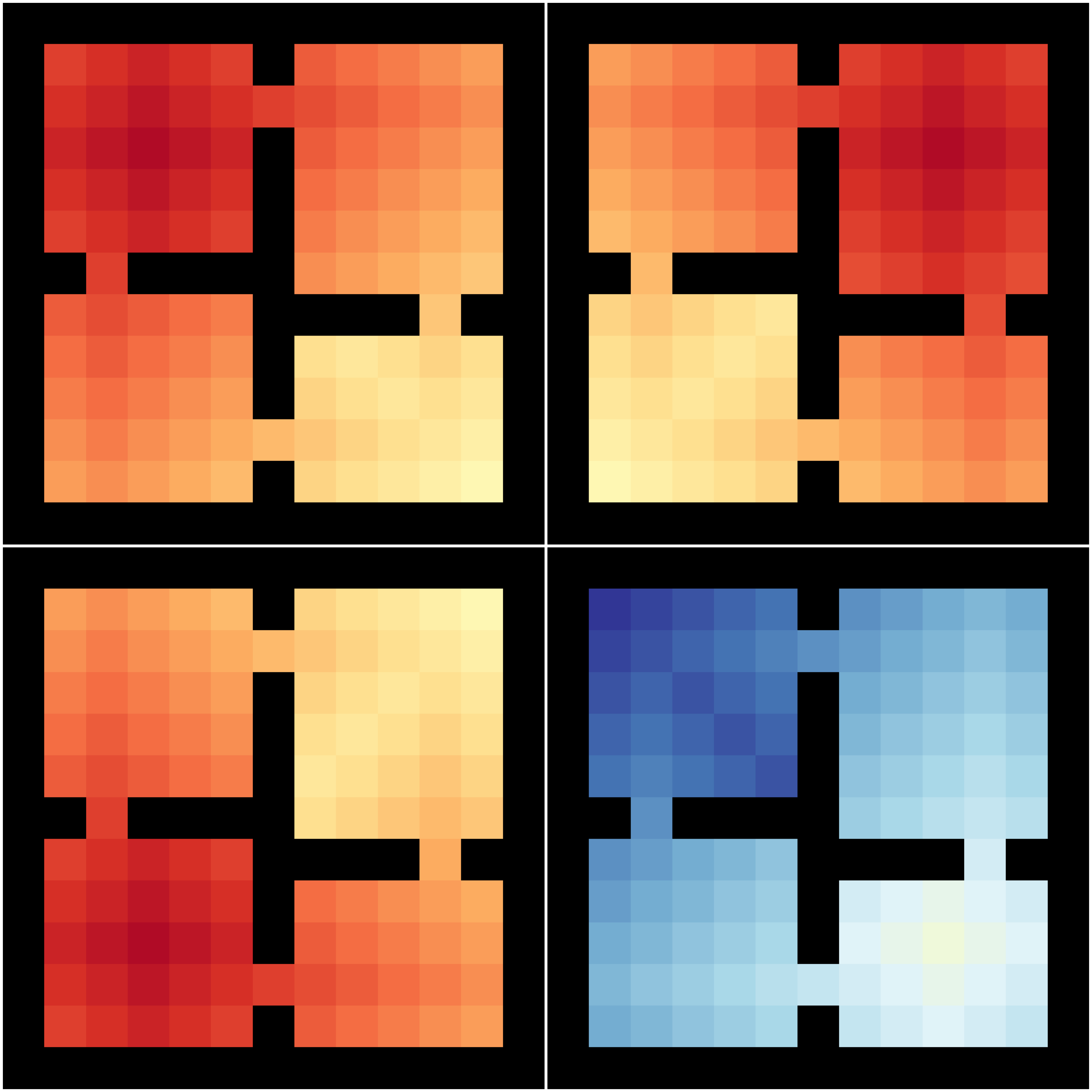}
        \includegraphics[height=0.85in]{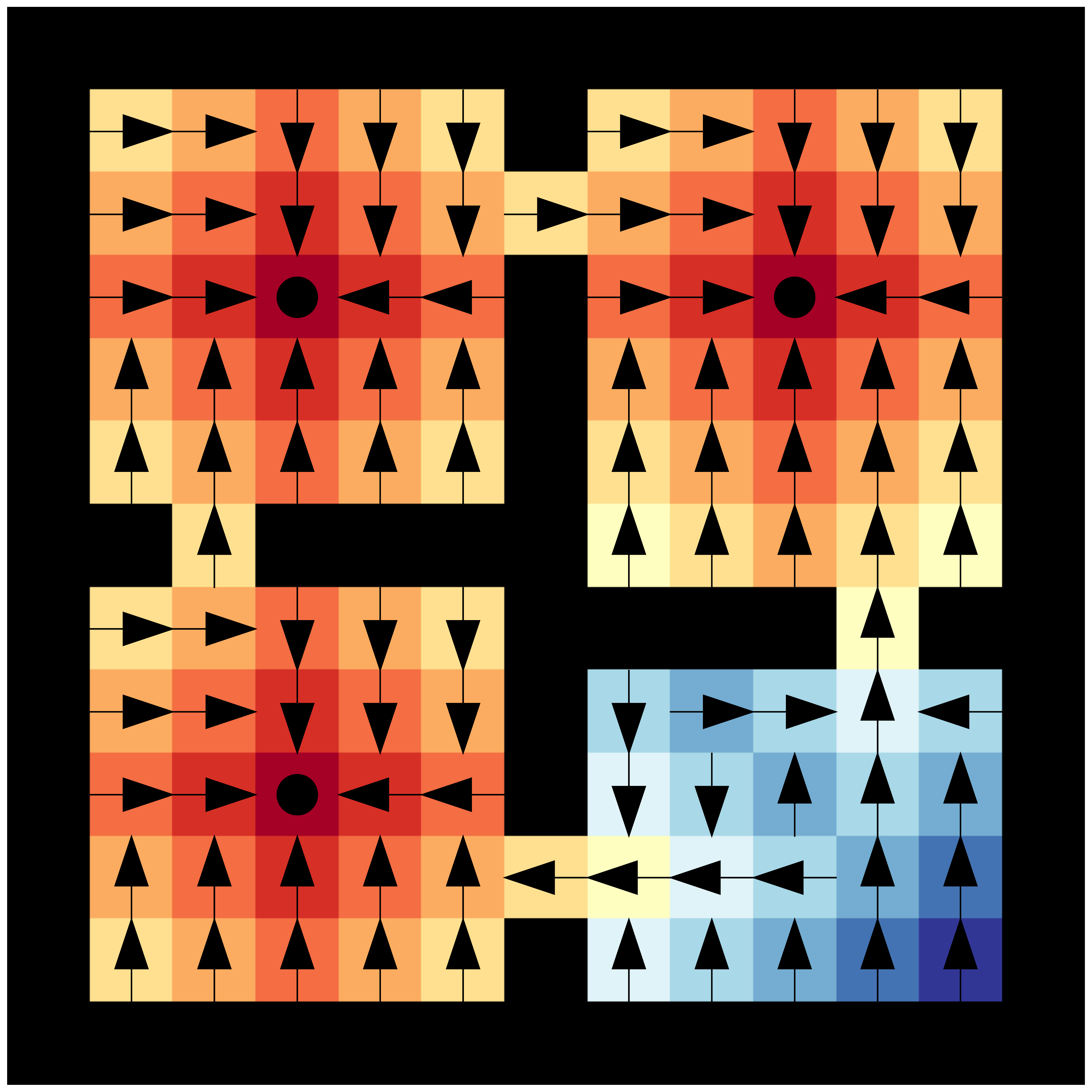}
        \caption{$M_{\text{L}} \vee M_{\text{T}}$}
        \label{fig:c}
    \end{subfigure}%
    ~
    \begin{subfigure}[t]{0.15\textwidth}
        \centering
        \includegraphics[height=0.85in]{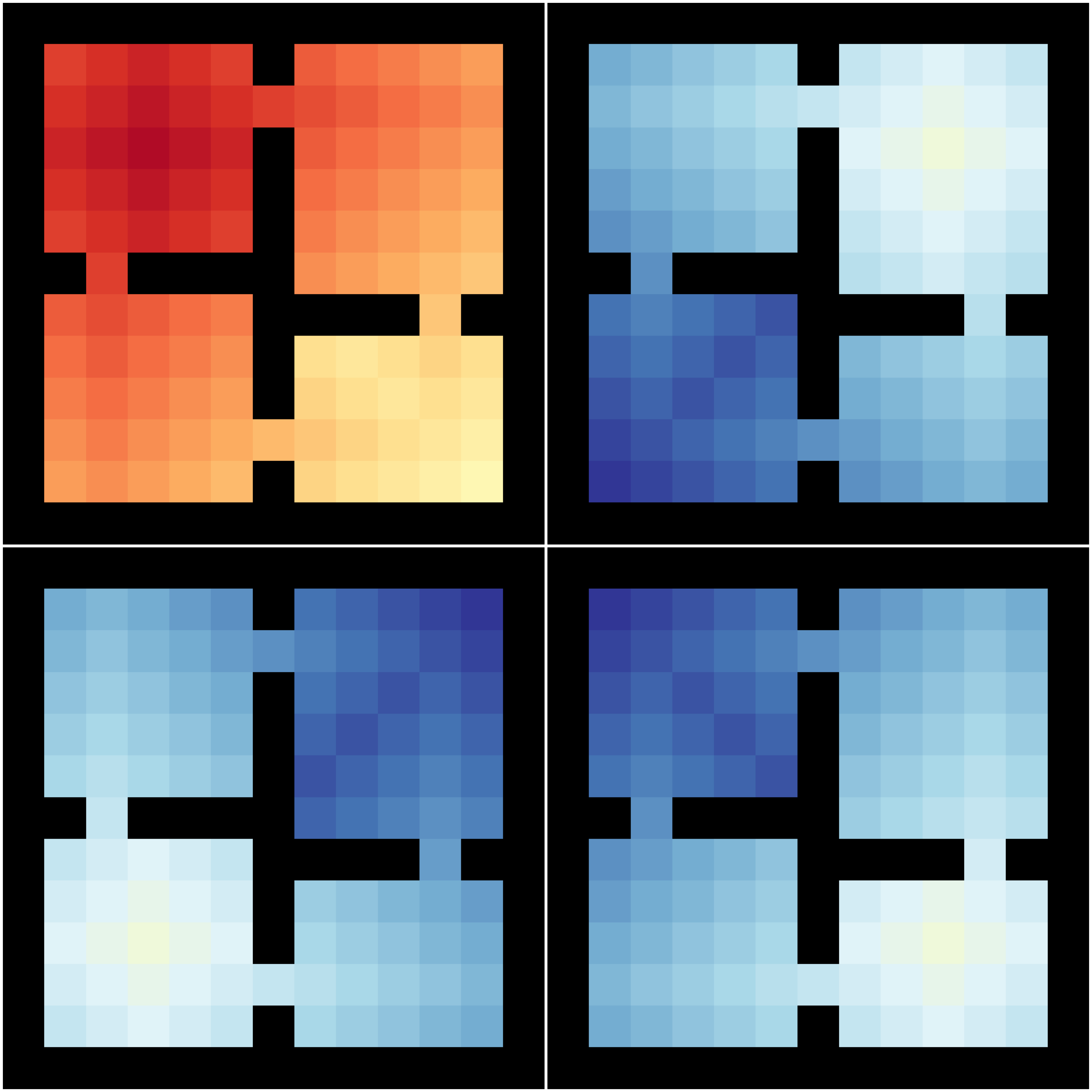}
        \includegraphics[height=0.85in]{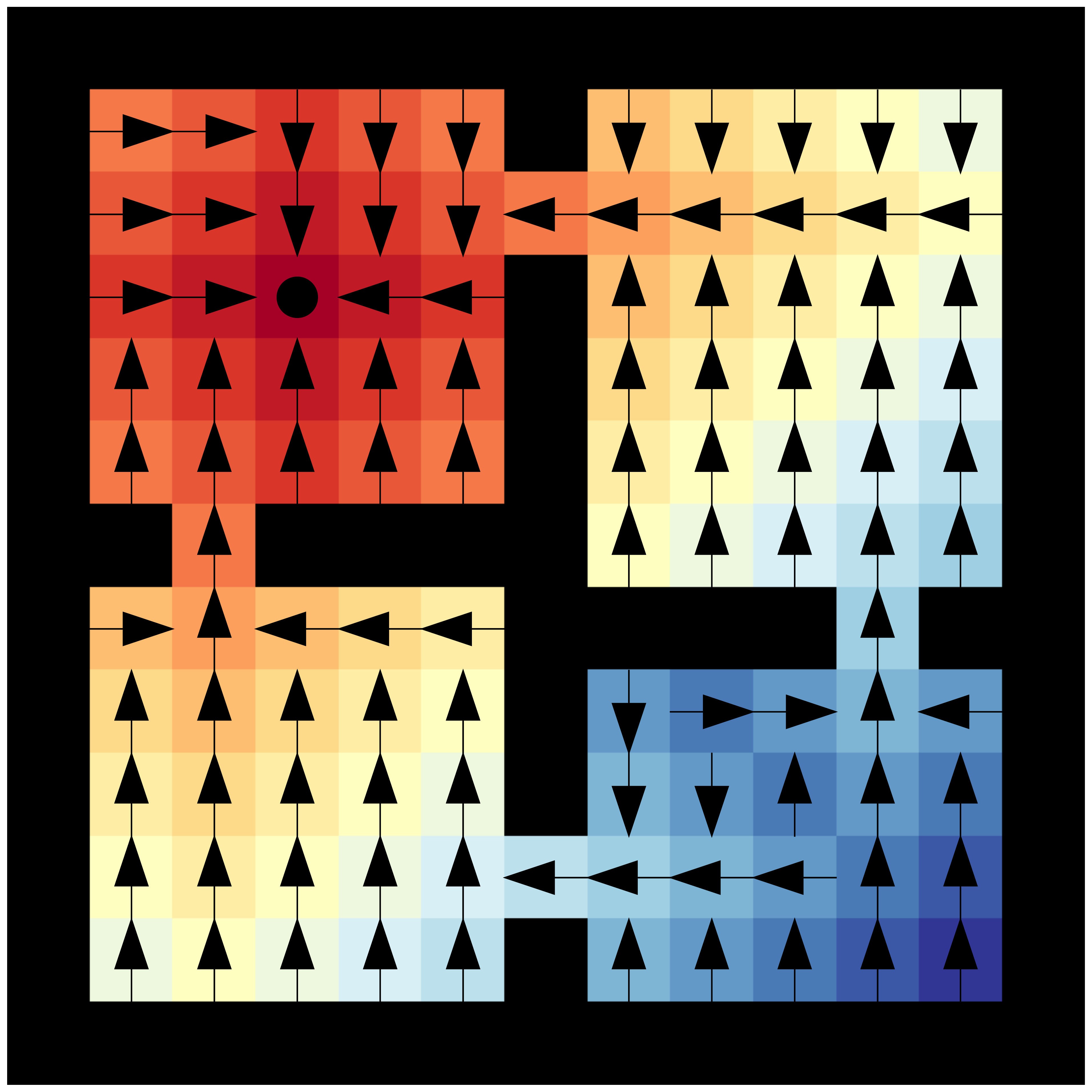}
        \caption{$M_{\text{L}} \wedge M_{\text{T}}$}
        \label{fig:d}
    \end{subfigure}%
    ~ 
    \begin{subfigure}[t]{0.15\textwidth}
        \centering
        \includegraphics[height=0.85in]{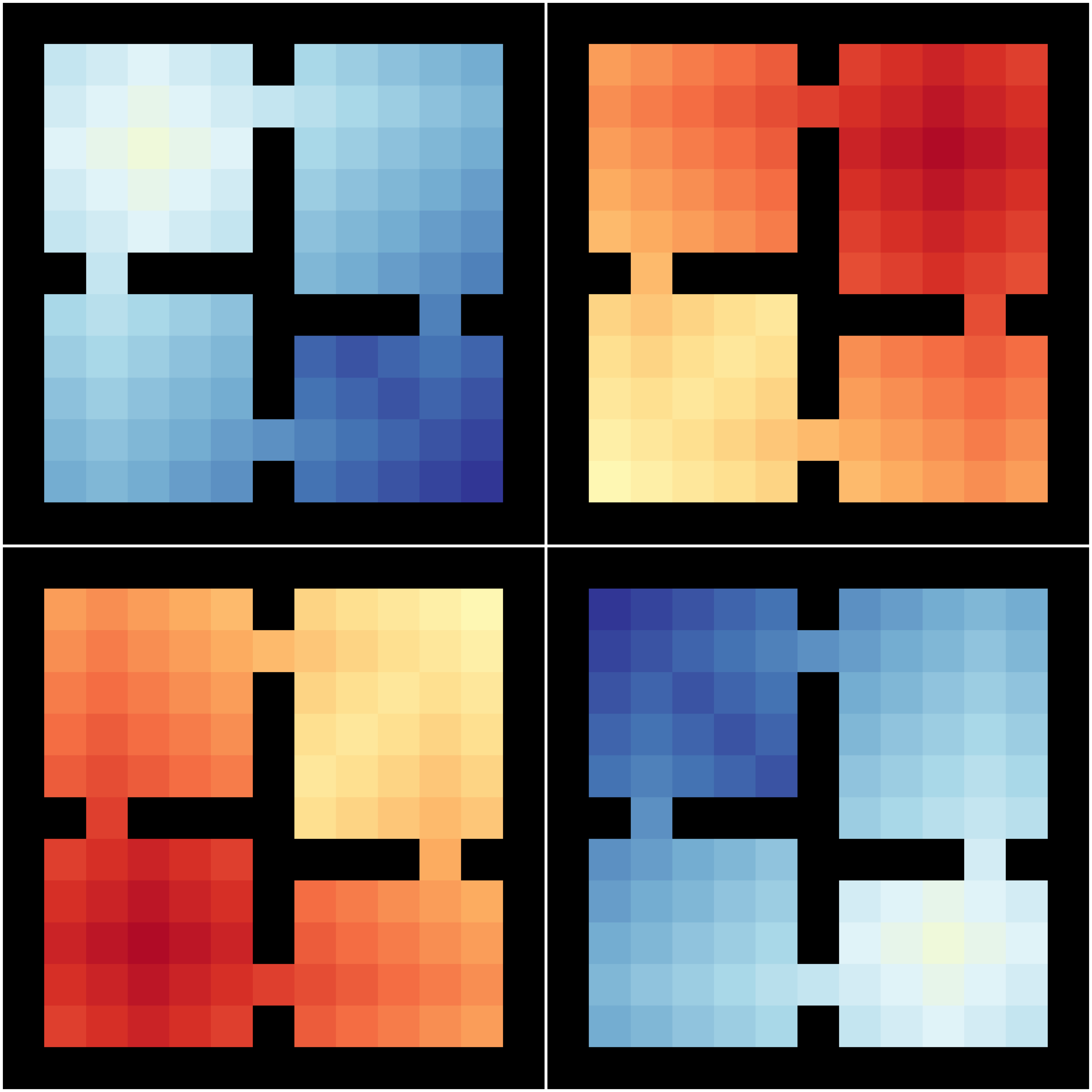}
        \includegraphics[height=0.85in]{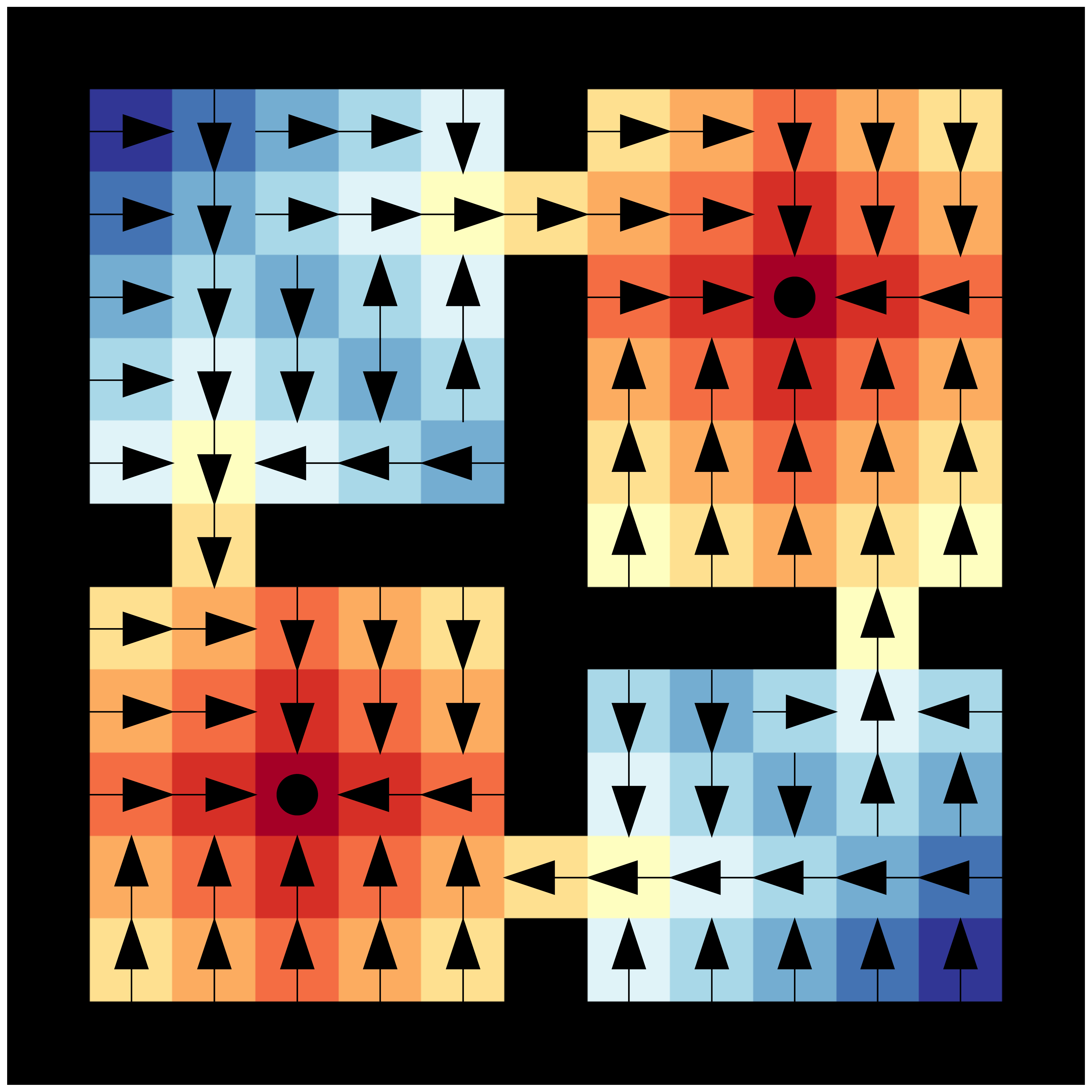}
        \caption{$M_{\text{L}} \veebar M_{\text{T}}$}
        \label{fig:e}
    \end{subfigure}%
    ~ 
    \begin{subfigure}[t]{0.15\textwidth}
        \centering
        \includegraphics[height=0.85in]{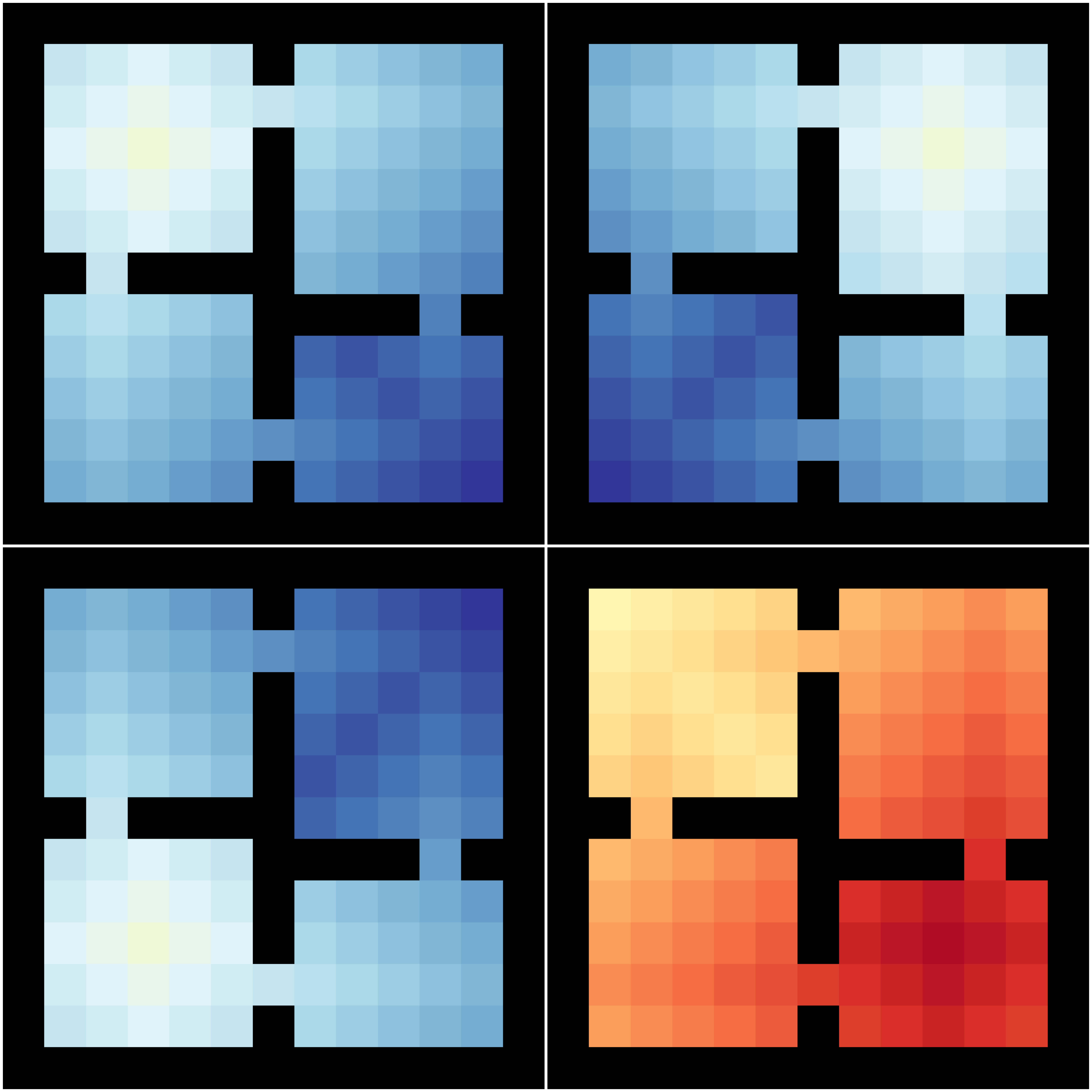}
        \includegraphics[height=0.85in]{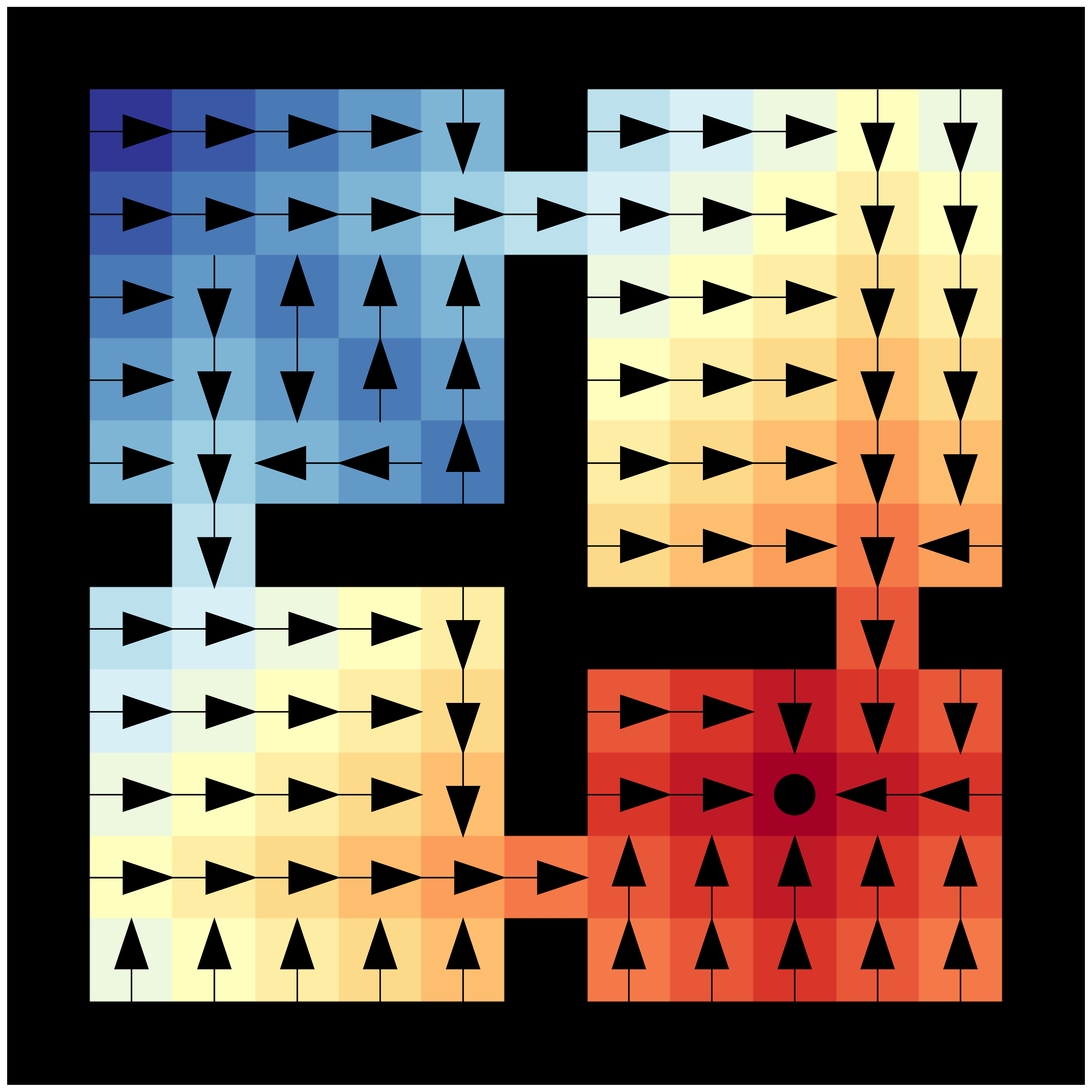}
        \caption{$M_{\text{L}} \barvee M_{\text{T}}$}
        \label{fig:f}
    \end{subfigure}%
    \caption{An example of zero-shot Boolean algebraic composition using the learned extended value functions. The top row shows the extended value functions. For each, the plots show the value of each state with respect to the four goals (the centre of each room). The bottom row shows the recovered regular value functions obtained by maximising over goals. Arrows represent the optimal action in a given state. (\subref{fig:a}--\subref{fig:b}) The learned optimal extended value functions for the base tasks. (\subref{fig:c}) Zero-shot disjunctive composition.  (\subref{fig:d}) Zero-shot conjunctive composition.  (\subref{fig:e}) Combining operators to model exclusive-or composition.  (\subref{fig:f}) Composition that produces logical nor. Note that the resulting optimal value function can attain a goal not explicitly represented by the base tasks. }
    \label{fig:composition}
\end{figure*}

By learning extended value functions, an agent can subsequently solve a massive number of tasks; however, the upfront cost of learning is likely to be higher. 
We investigate the trade-off between the two approaches by quantifying how the sample complexity scales with the number of tasks.
We compare to \citet{vanniekerk19}, who use regular value functions to demonstrate optimal disjunctive composition.
We note that while the upfront learning cost is therefore lower, the number of tasks expressible using only disjunction is $2^K-1$, which is significantly less than the full Boolean algebra. 
We also conduct  a test using an extended version of the Four Rooms domain, where additional goals are placed along the sides of all walls, resulting in a total of 40 goals.
Empirical results are illustrated by Figure~\ref{fig:comparison}.

\begin{figure*}[h!]
    \centering
    \begin{subfigure}[t]{0.3\textwidth}
        \centering
        \includegraphics[height=1.3in]{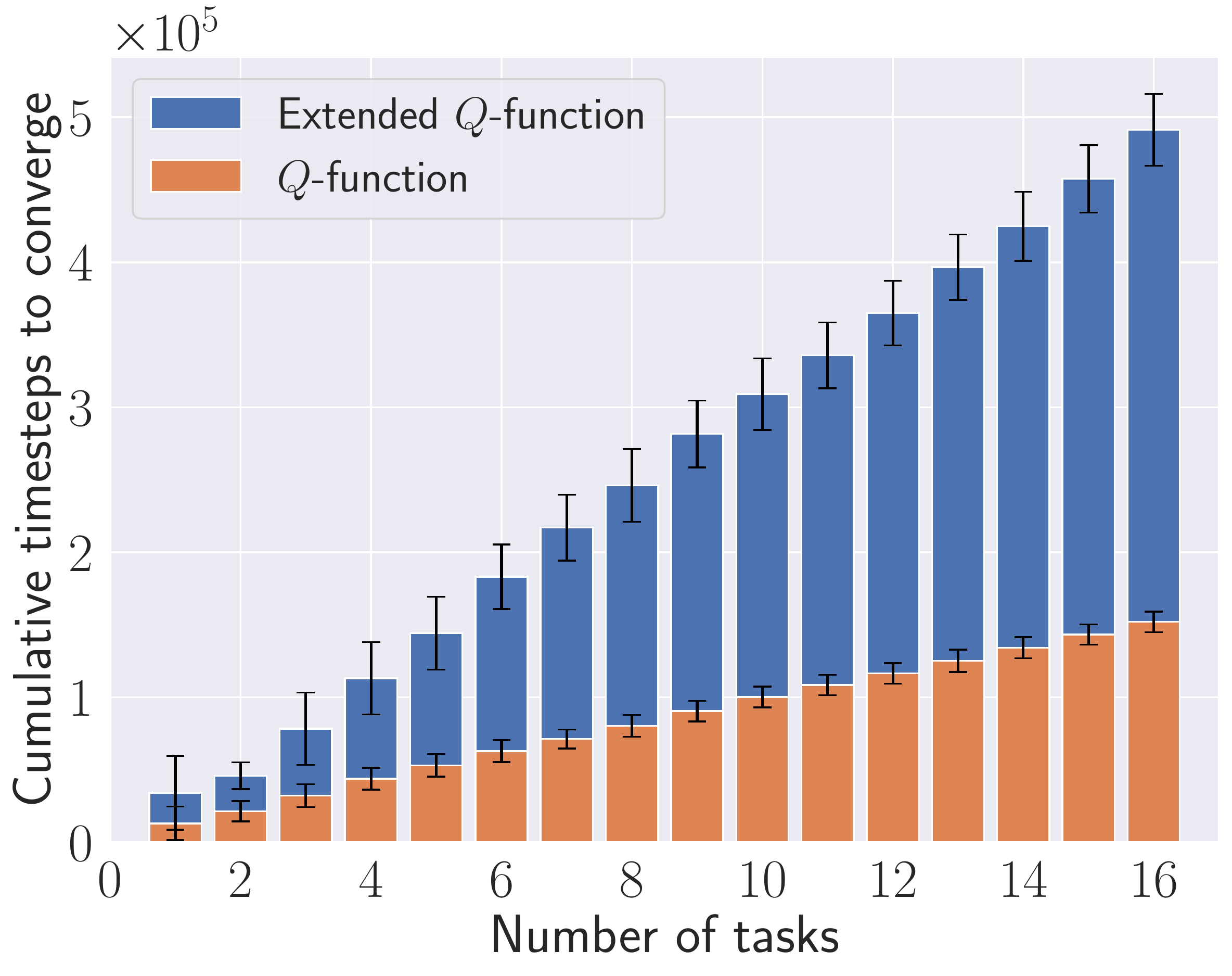}
        \caption{Cumulative number of samples required to learn optimal extended and regular value functions. Error bars represent standard deviations over 100 runs.}
        \label{fig:p}
    \end{subfigure}%
    \quad
    \begin{subfigure}[t]{0.3\textwidth}
        \centering
        \includegraphics[height=1.3in]{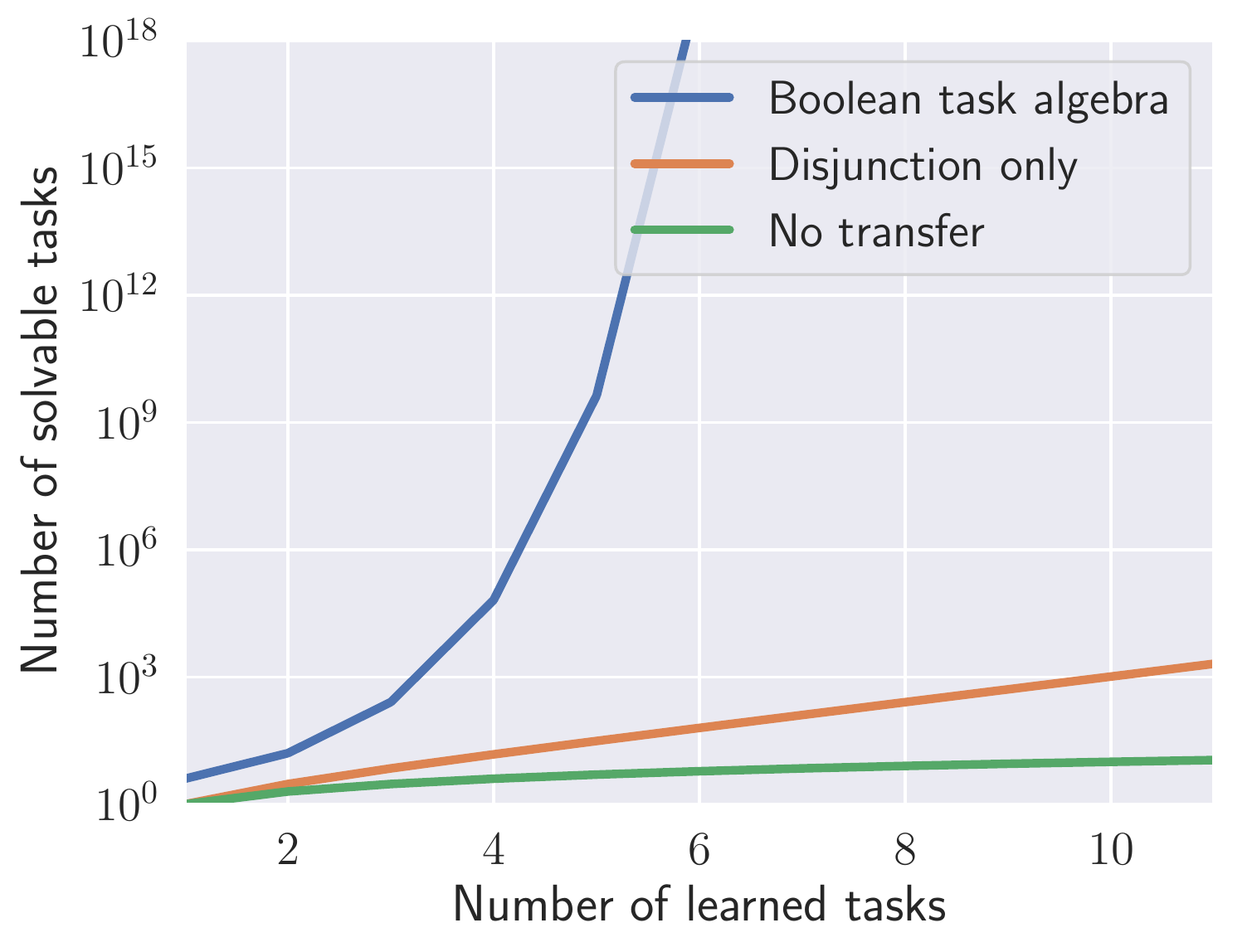}
        \caption{Number of tasks that can be solved as a function of the number of existing tasks solved. Results are plotted on a log-scale.}
        \label{fig:q}
    \end{subfigure}%
    \quad
    \begin{subfigure}[t]{0.3\textwidth}
        \centering
        \includegraphics[height=1.3in]{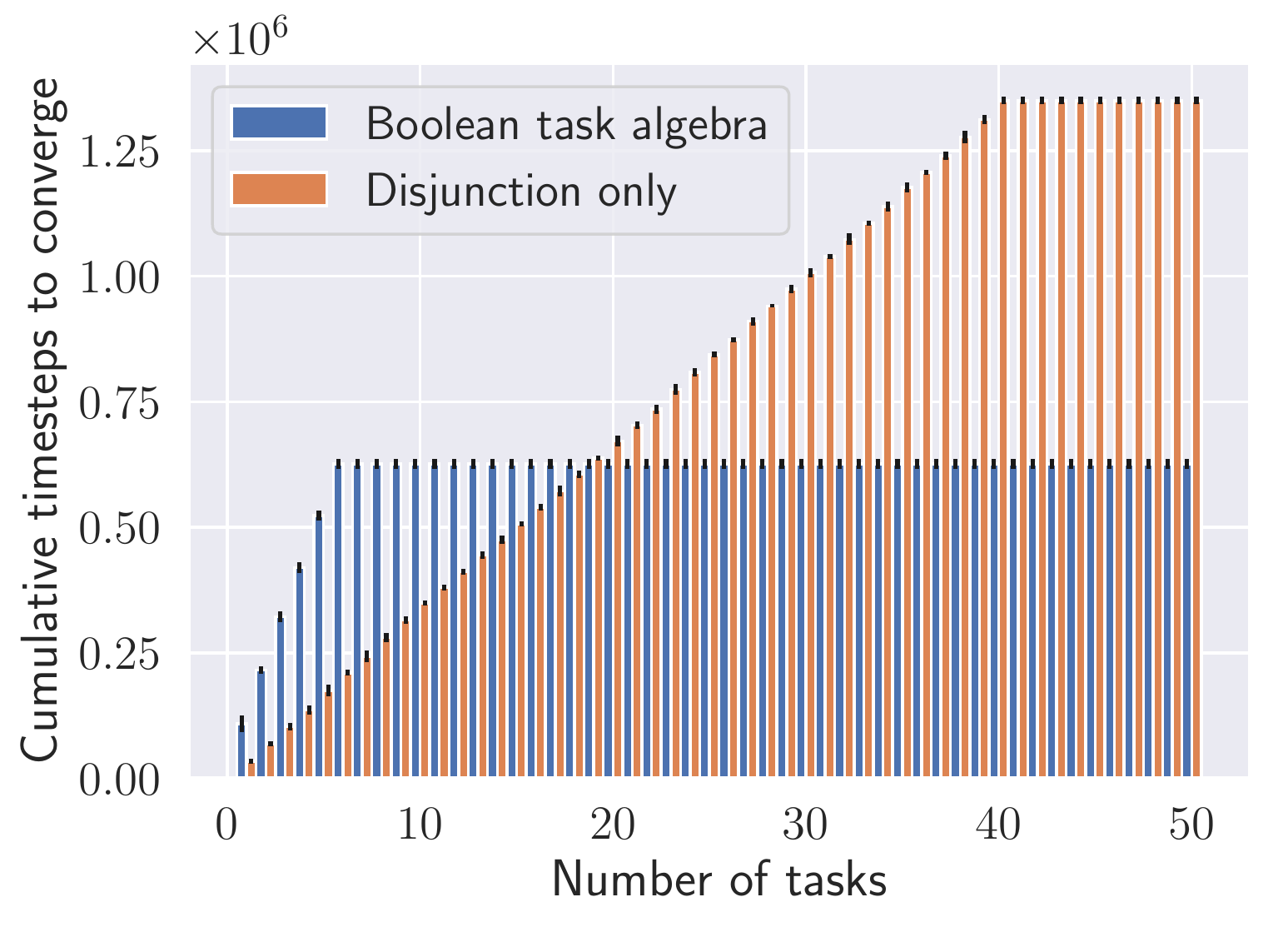}
        \caption{Cumulative number of samples required to solve tasks in a 40-goal Four Rooms domain. Error bars represent standard deviations over 100 runs.}
        \label{fig:r}
    \end{subfigure}%
    \caption{Results in comparison to the disjunctive composition of \citet{vanniekerk19}. (\subref{fig:p})  The number of samples required to learn the extended value function is greater than learning a standard value function. However, both scale linearly and differ only by a constant factor. (\subref{fig:q}) The extended value functions allow us to solve exponentially more tasks than the disjunctive approach without further learning.   (\subref{fig:r}) In the modified task with 40 goals, we need to learn only 7 base tasks, as opposed to 40 for the disjunctive case.}
    \label{fig:comparison}
\end{figure*}

Our results show that while additional samples are needed to learn an extended value function, the agent is able to expand the tasks it can solve super-exponentially. 
Furthermore, the number of base tasks we need to solve is only logarithmic in the number of goal states. 
For an environment with $K$ goals, we need to learn only $\floor{\log_2 K} + 1$ base tasks, as opposed to the disjunctive approach which requires $K$ base tasks.   
Thus by sacrificing sample efficiency initially, we achieve an exponential increase in abilities compared to previous work \citep{vanniekerk19}.

\section{Composition with Function Approximation} \label{sec:boxman}


Finally, we demonstrate that our compositional approach can also be used to tackle high-dimensional domains where function approximation is required.
We use the same video game environment as \citet{vanniekerk19}, where an agent must navigate a 2D world and collect objects of different shapes and colours from any initial position.
The state space is an $84 \times 84$ RGB image, and the agent is able to move in any of the four cardinal directions.
The agent also possesses a \texttt{pick-up} action, which allows it to collect an object when standing on top of it.
There are two shapes (squares and circles) and three colours (blue, beige and purple) for a total of six unique objects.

To learn the extended action-value functions, we modify deep Q-learning \citep{mnih15} similarly to the many-goals update method of \citet{veeriah2018many}. Here, a universal value function approximator (UVFA) \citep{schaul15} is used to represent the action values for each state and goal (both specified as RGB images).\footnote{The hyperparameters and network architecture are listed in the supplementary material}
Additionally, when a terminal state is encountered, it is added to the collection of goals seen so far, and when learning updates occur, these goals are sampled randomly from a replay buffer. 
We first learn to solve two base tasks: collecting blue objects and collecting squares.
As shown in Figure~\ref{fig:collect-evf}, by training the UVFA for each task using the extended rewards definition, the agent learns not only how to achieve all goals, but also how desirable each of those goals are for the current task.
These UVFAs can now be composed to solve new tasks with no further learning.

\begin{figure*}[h!]
    \centering
    \begin{subfigure}[t]{0.35\textwidth}
        \centering
        \includegraphics[height=1.4in]{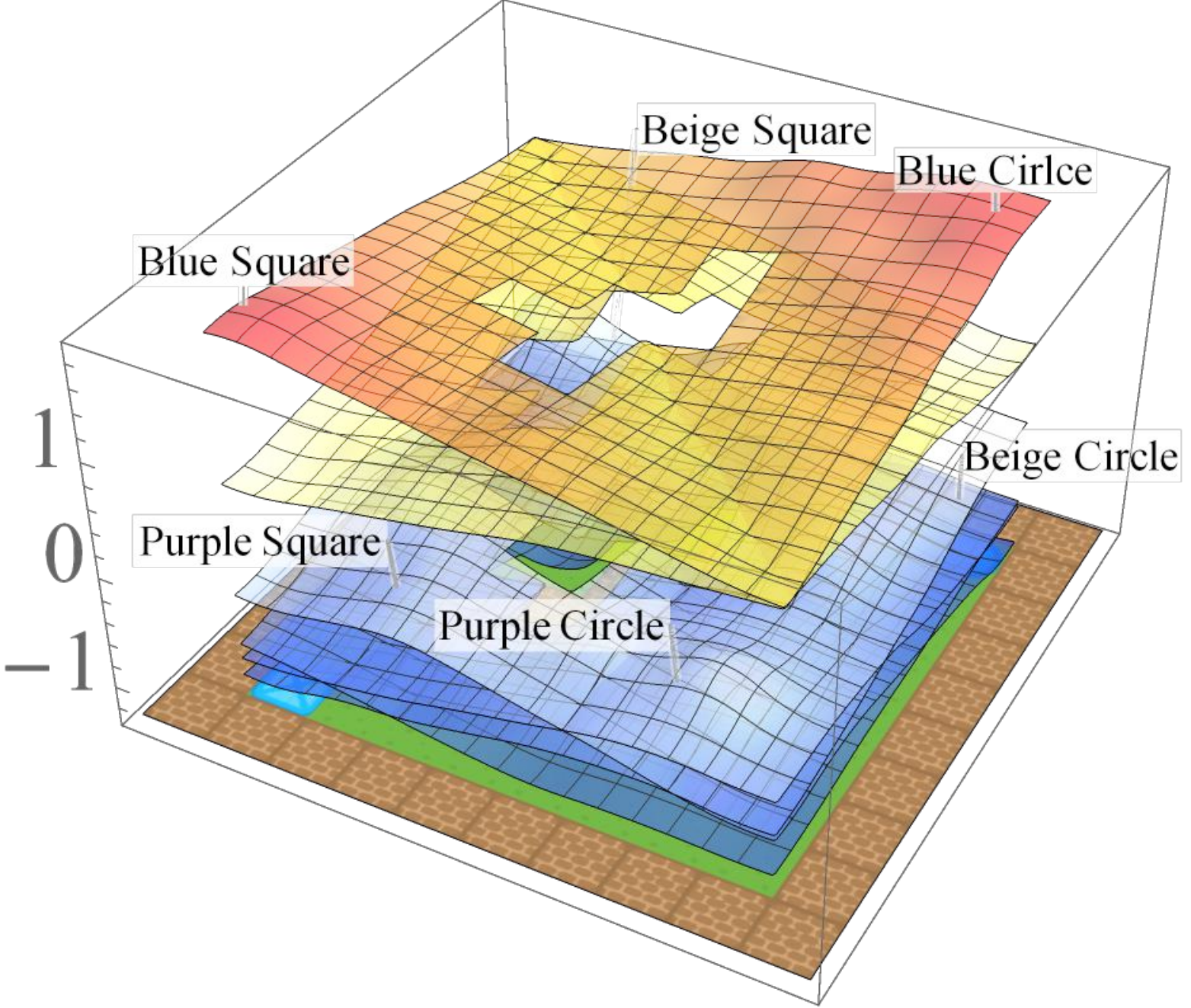}
    \end{subfigure}%
    \quad
    \begin{subfigure}[t]{0.35\textwidth}
        \centering
        \includegraphics[height=1.4in]{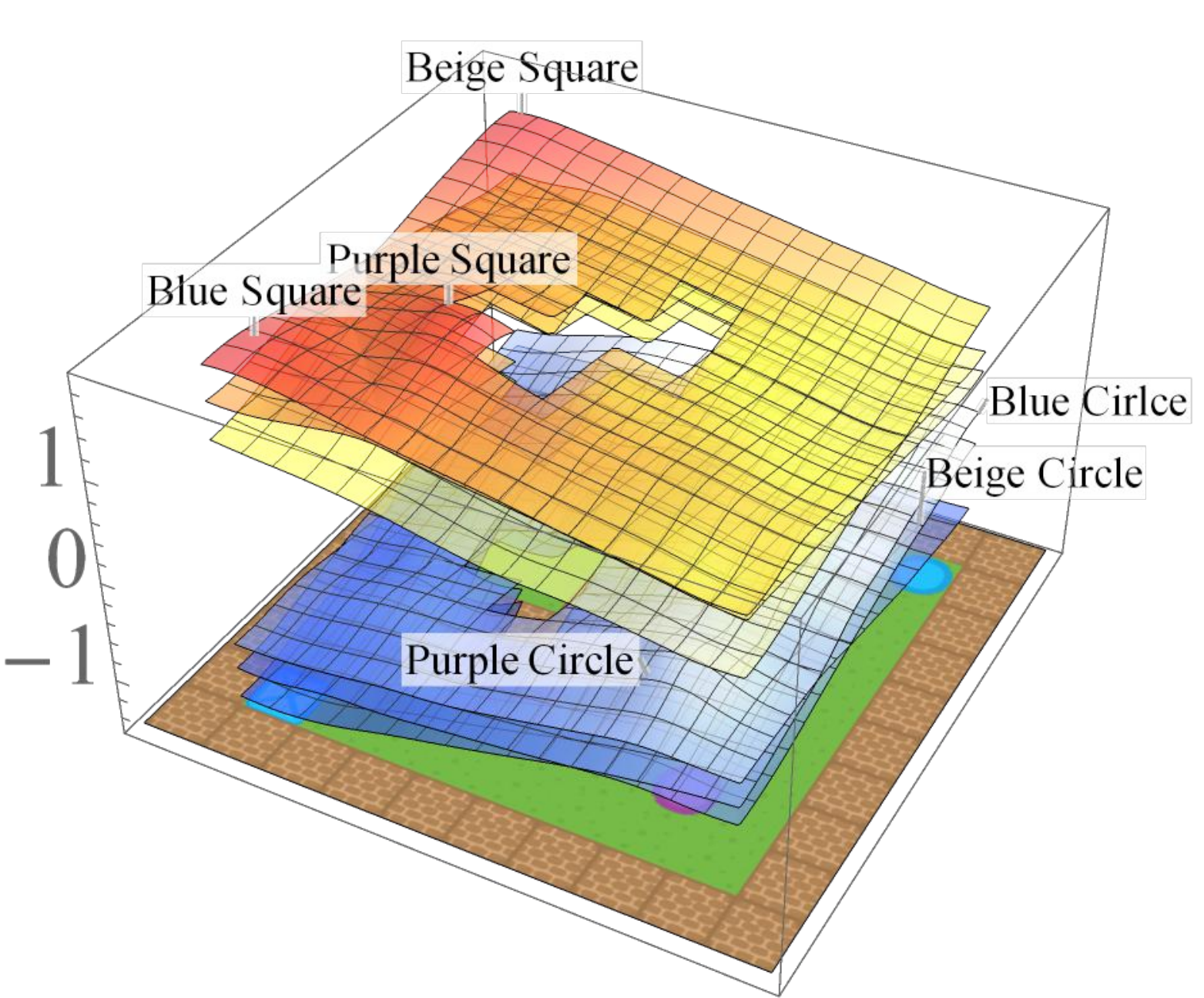}
    \end{subfigure}%
    \caption{ Extended value function for collecting blue objects (left) and squares (right). To generate the value functions, we place the agent at every location and compute the maximum output of the network over all goals and actions.
We then interpolate between the points to smooth the graph. Any error in the visualisation is due to the use of non-linear function approximation.}
\label{fig:collect-evf}
\end{figure*}


We demonstrate composition characterised by disjunction, conjunction and exclusive-or.
This corresponds to tasks where the target items are: 
\begin{enumerate*}[label=(\roman*)]
  \item blue or square,
  \item blue squares, and
  \item blue or squares, but not blue squares.
\end{enumerate*}
Figure~\ref{fig:repoman} illustrates the composed value functions and samples of the subsequent trajectories for the respective tasks. Figure~\ref{fig:Boxman_returns_1} shows the average returns across random  initial positions of the agent.\footnote{Experiments involving randomised object positions are included in the supplementary material.}

\begin{figure*}[h!]
        \centering
    \begin{subfigure}[t]{0.3\textwidth}
        \centering
        \includegraphics[height=1.4in]{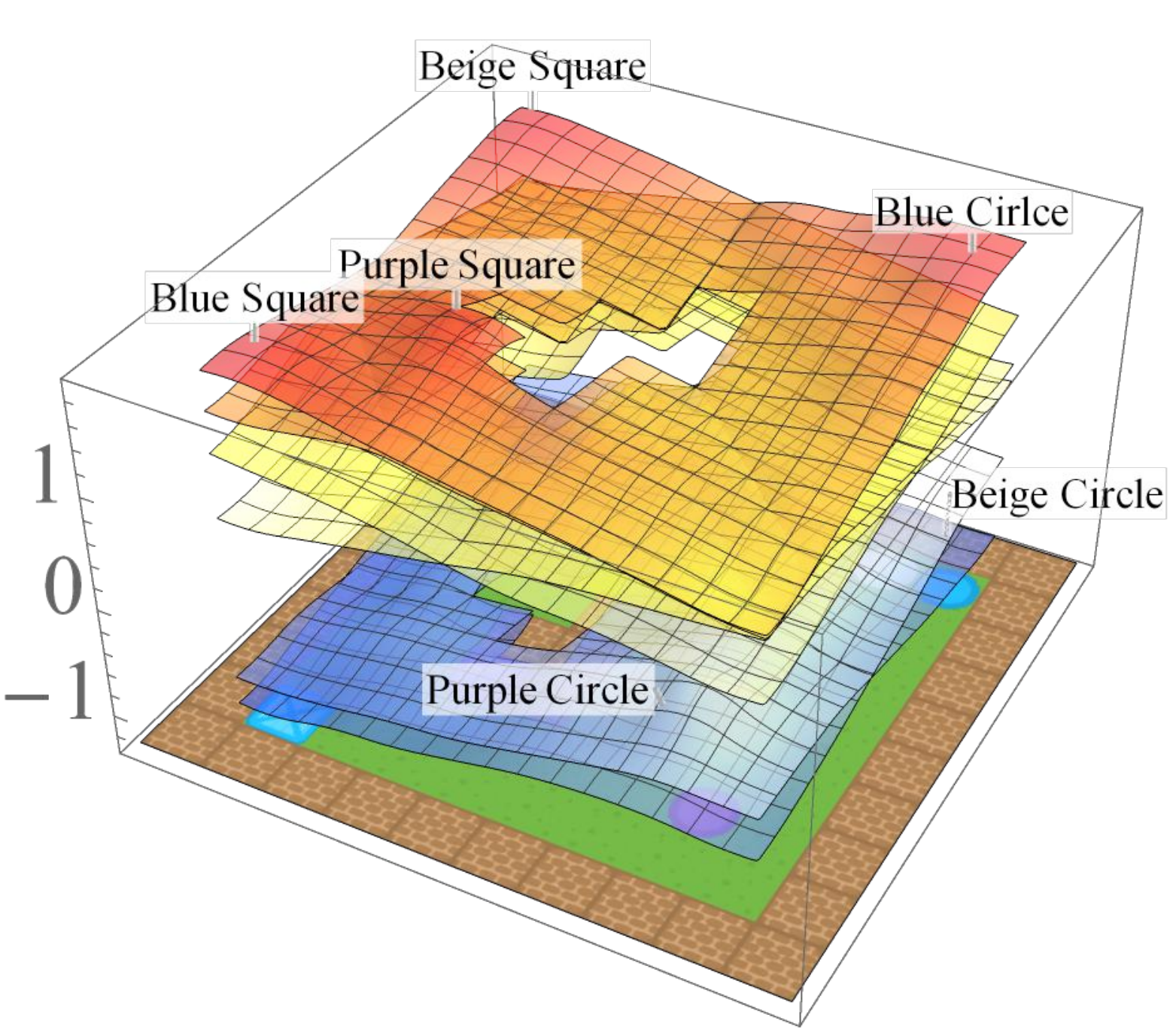}
        \caption{Extended value function for disjunctive composition.}
        \label{fig:u_}
    \end{subfigure}%
    \quad
    \begin{subfigure}[t]{0.3\textwidth}
        \centering
        \includegraphics[height=1.4in]{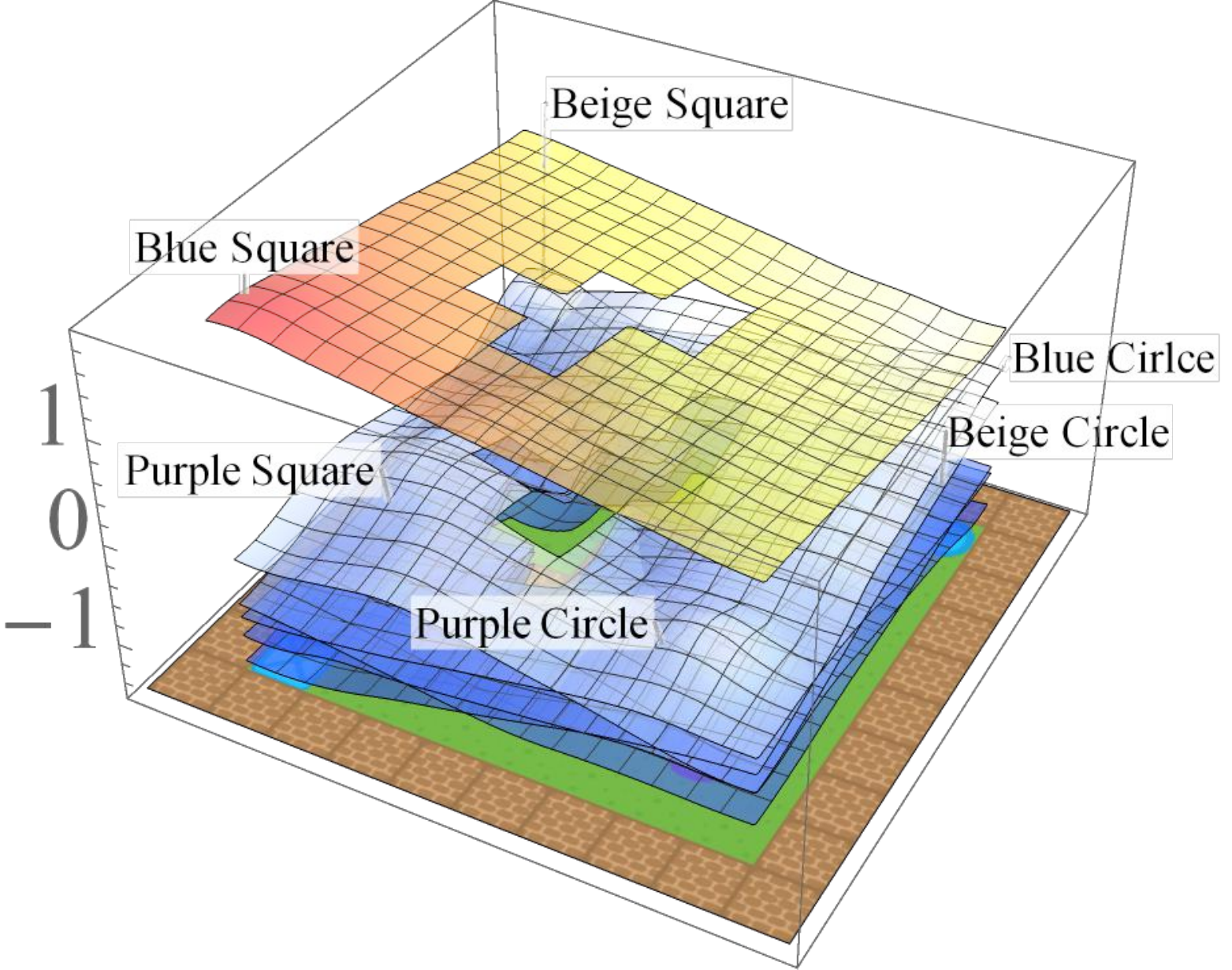}
        \caption{Extended value function for conjunctive composition.}
        \label{fig:v_}
    \end{subfigure}%
    \quad
    \begin{subfigure}[t]{0.3\textwidth}
        \centering
        \includegraphics[height=1.4in]{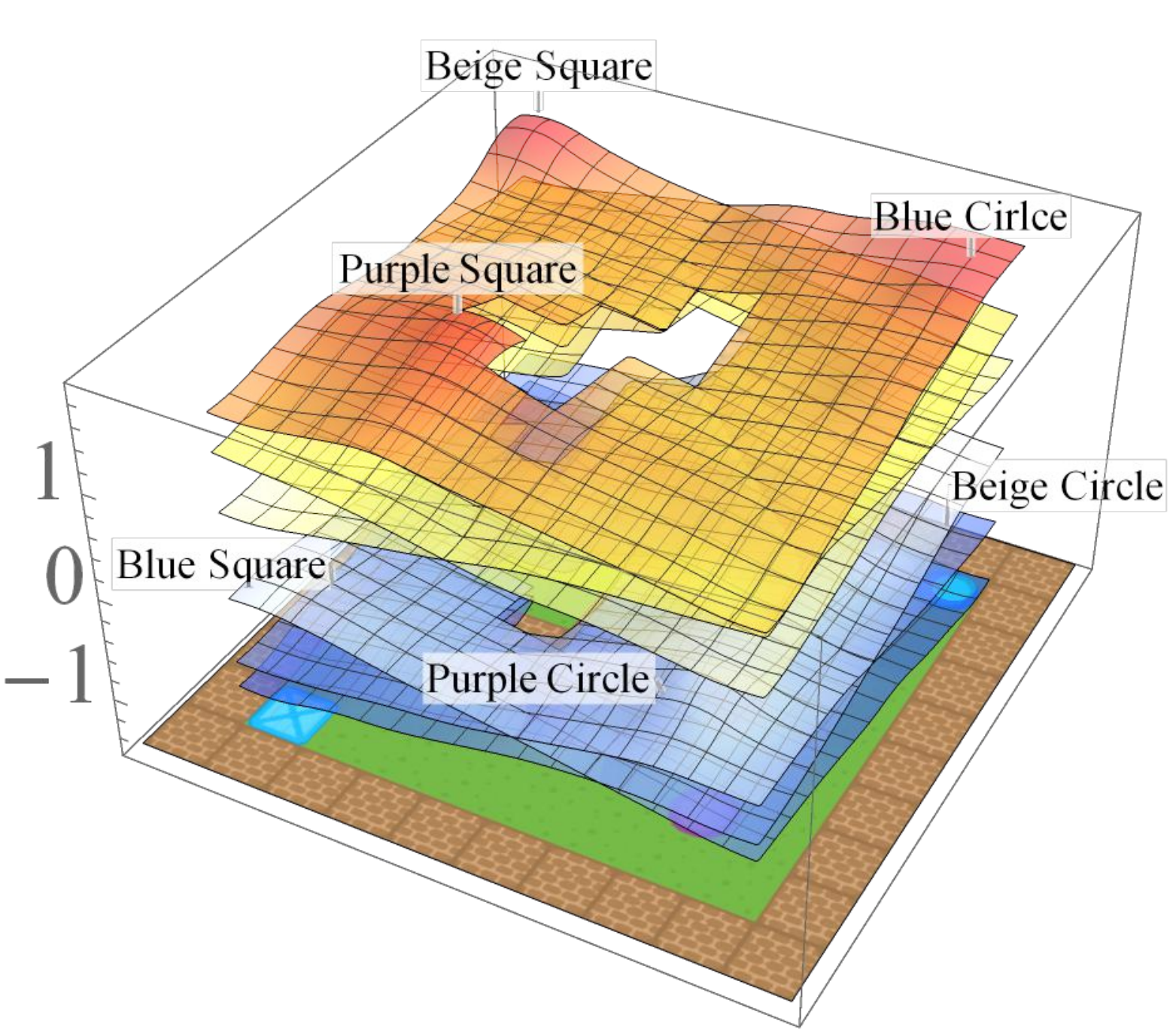}
        \caption{Extended value function for exclusive-or composition.}
        \label{fig:w_}
    \end{subfigure}%
    \\
        \centering
    \begin{subfigure}[t]{0.3\textwidth}
        \centering
        \includegraphics[height=1.3in]{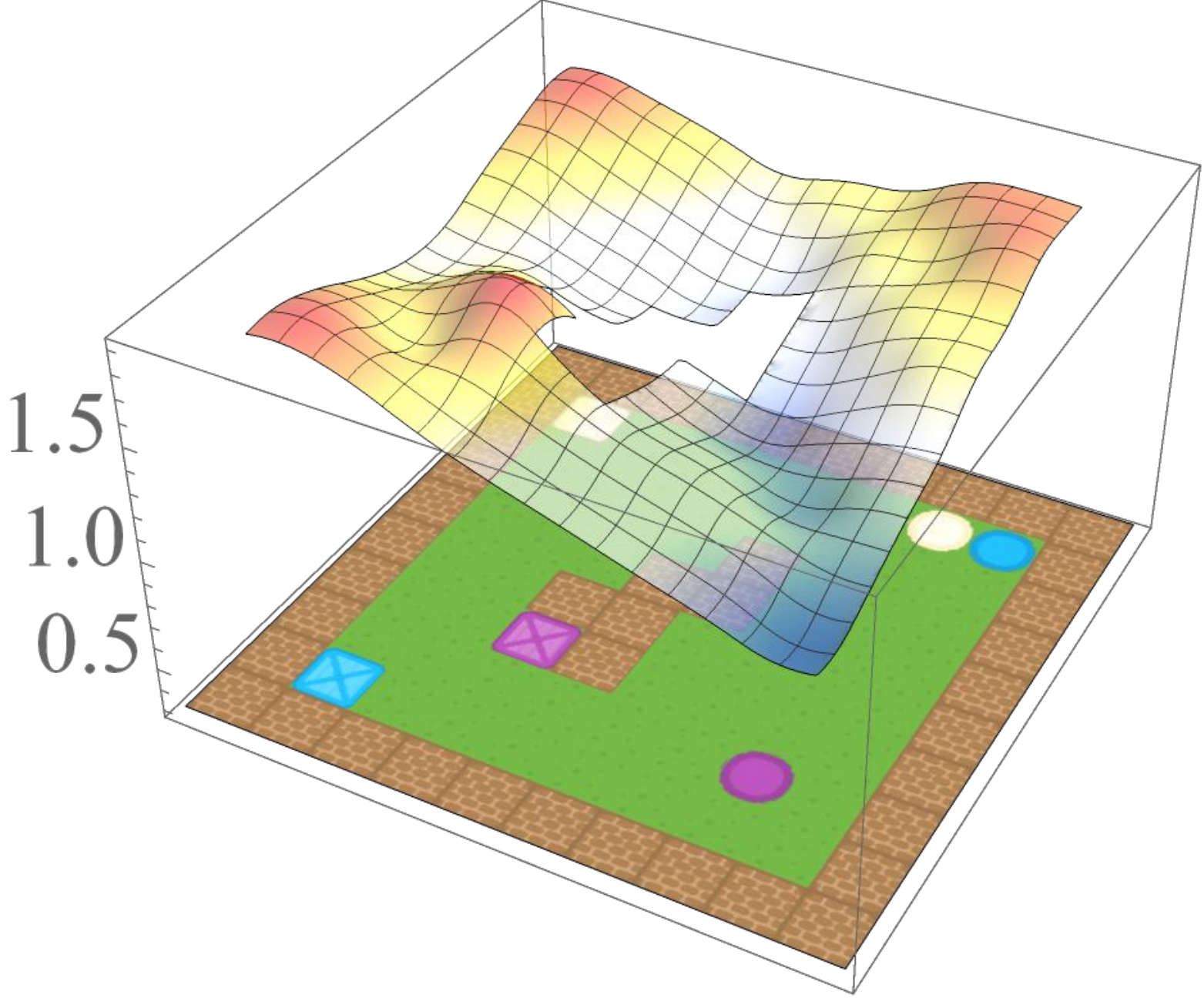}
        \caption{Value function for disjunctive composition.}
        \label{fig:u_}
    \end{subfigure}%
    \quad
    \begin{subfigure}[t]{0.3\textwidth}
        \centering
        \includegraphics[height=1.3in]{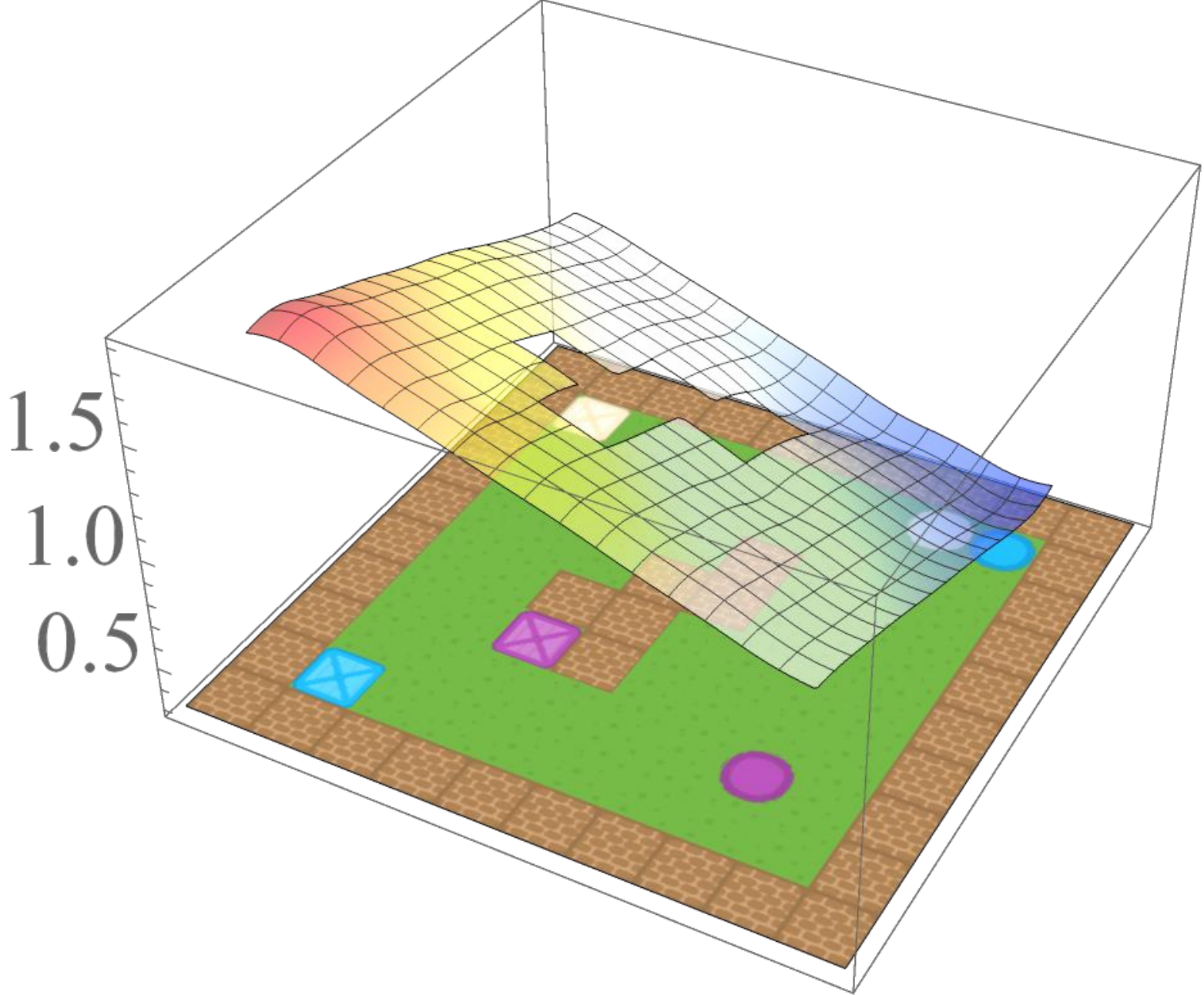}
        \caption{Value function for conjunctive composition.}
        \label{fig:v_}
    \end{subfigure}%
    \quad
    \begin{subfigure}[t]{0.3\textwidth}
        \centering
        \includegraphics[height=1.3in]{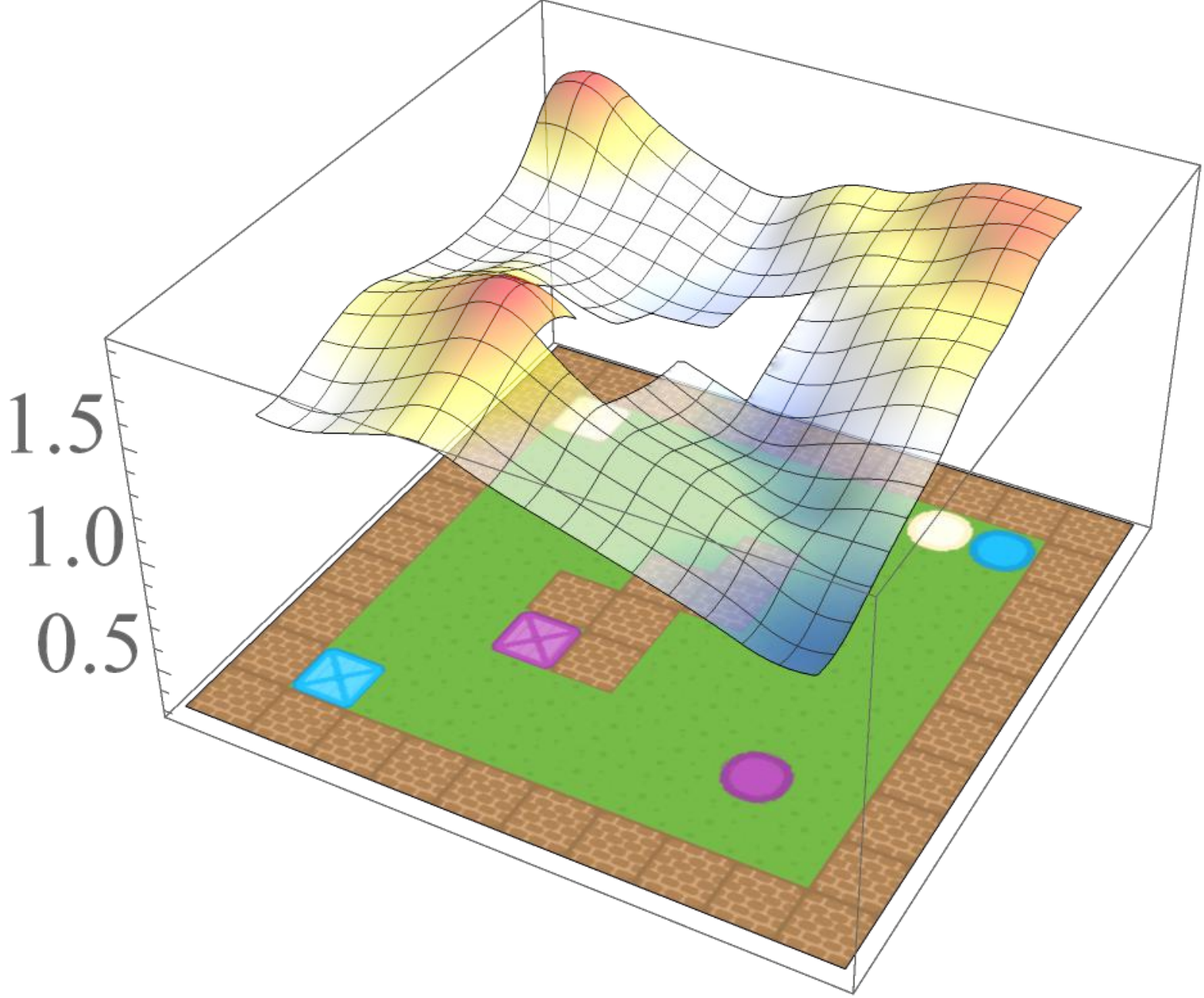}
        \caption{Value function for exclusive-or composition.}
        \label{fig:w_}
    \end{subfigure}%
    \\
    \centering
    \begin{subfigure}[t]{0.3\textwidth}
        \centering
        \includegraphics[height=1.3in]{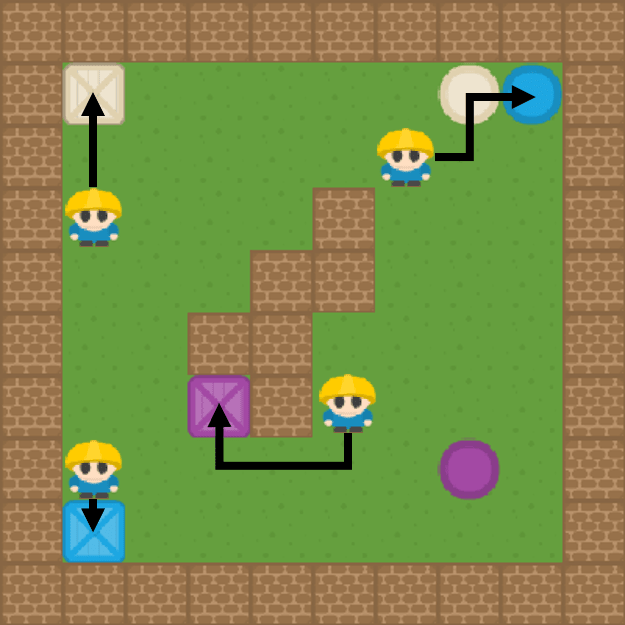}
        \caption{Trajectories for disjunctive composition.}
        \label{fig:u}
    \end{subfigure}%
    \quad
    \begin{subfigure}[t]{0.3\textwidth}
        \centering
        \includegraphics[height=1.3in]{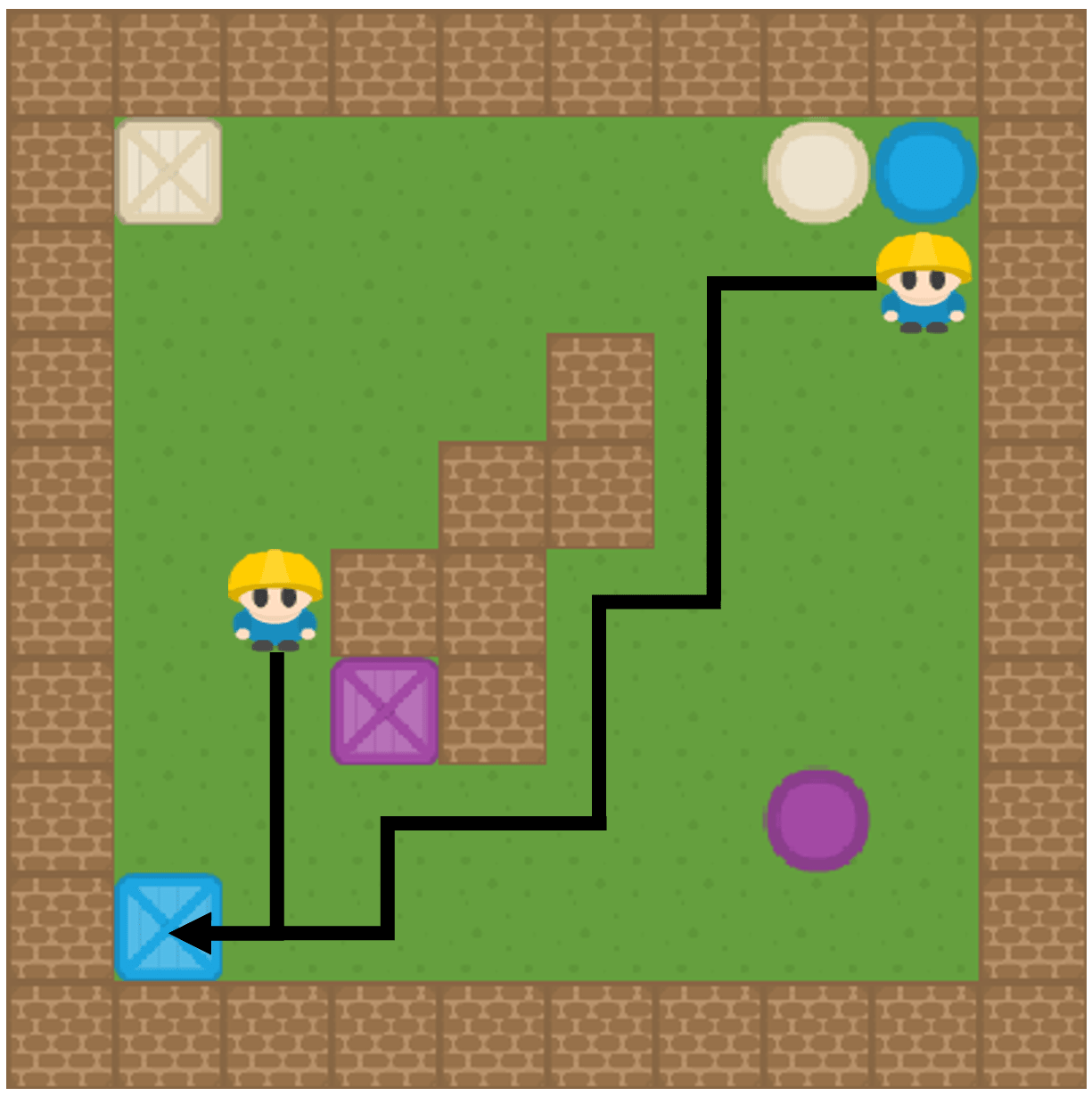}
        \caption{Trajectories for conjunctive composition.}
        \label{fig:v}
    \end{subfigure}%
    \quad
    \begin{subfigure}[t]{0.3\textwidth}
        \centering
        \includegraphics[height=1.3in]{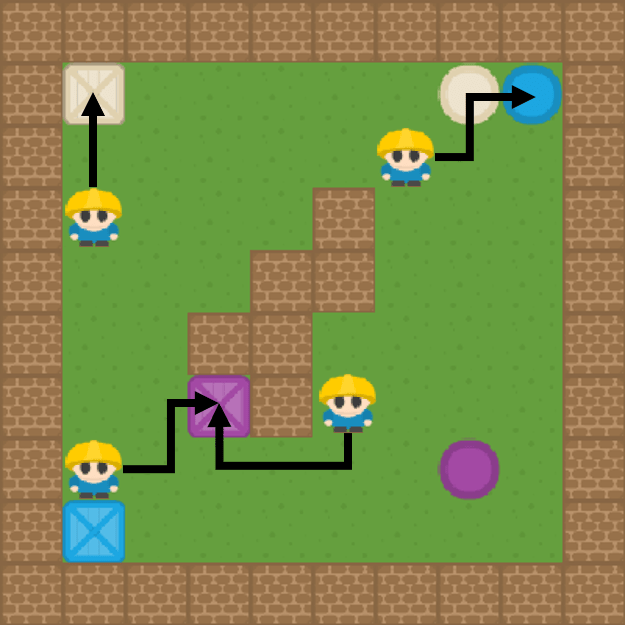}
        \caption{Trajectories for exclusive-or composition.}
        \label{fig:w}
    \end{subfigure}%
    \caption{By composing extended value functions from the base tasks (collecting blue objects, and collecting squares), we can act optimally in new tasks with no further learning.}
    \label{fig:repoman}
\end{figure*}

\begin{figure}[h!]
    \centering
    \includegraphics[width=0.5\linewidth]{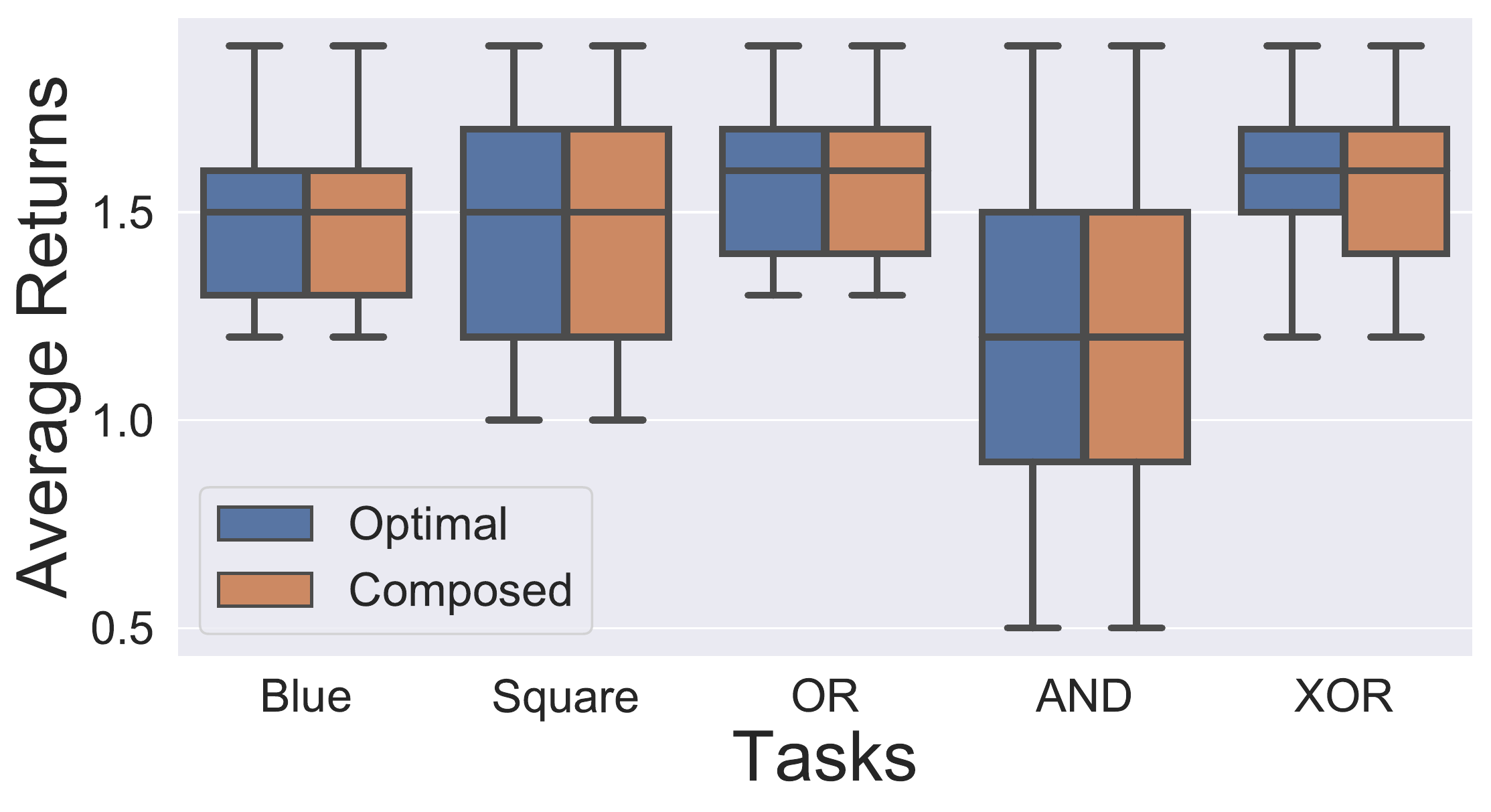}
    \caption{Average returns over 1000 episodes for the \textit{Blue} and \textit{Square} tasks, and their disjunction (\textit{OR}), conjunction (\textit{AND}) and exlusive-or (\textit{XOR}).}
    \label{fig:Boxman_returns_1}
\end{figure}

\section{Related Work}

The ability to compose value functions was first demonstrated using the linearly-solvable MDP framework \citep{todorov07}, where value functions could be composed to solve tasks similar to the disjunctive case \citep{todorov09}.
\citet{vanniekerk19} show that the same kind of composition can be achieved using entropy-regularised RL \citep{fox15}, and extend the results to the standard RL setting, where agents can optimally solve the disjunctive case.
Using entropy-regularised RL, \citet{haarnoja18} approximates the conjunction of tasks by averaging their reward functions, and demonstrates that by averaging the optimal value functions of the respective tasks, the agent can achieve performance close to optimal. 
\citet{hunt19} extends this result by composing value functions to solve the average reward task exactly, which approximates the true conjunctive case.
More recently, \citet{peng19} introduce a few-shot learning approach to compose policies multiplicatively. 
Although lacking theoretical foundations, their results show that an agent can learn a weighted composition of existing base skills to solve a new complex task.  
By contrast, we show that zero-shot optimal composition can be achieved for all Boolean operators.  

\section{Conclusion}

We have shown how to compose tasks using the standard Boolean algebra operators. 
These composite tasks can be solved without further learning by first learning goal-oriented value functions, and then composing them in a similar manner. 
Finally, we note that there is much room for improvement in learning the extended value functions for the base tasks. 
In our experiments, we learned each extended value function from scratch, but it is likely that having learned one for the first task, we could use it to initialise the extended value function for the second task to improve convergence times.
One area for improvement lies in efficiently learning the extended value functions, as well as developing better algorithms for solving tasks with sparse rewards.  
For example, it is likely that approaches such as hindsight experience replay \citep{andrychowicz2017hindsight} could reduce the number of samples required to learn extended value functions, while \citet{mirowski2016learning} provides a method for learning complex tasks with sparse rewards using auxiliary tasks. 
We leave incorporating these approaches to future work, but note that our framework is agnostic to the value-function learning algorithm.
Our proposed approach is a step towards both interpretable RL---since both the tasks and optimal value functions can be specified using Boolean operators---and the ultimate goal of lifelong learning agents, which are able to solve combinatorially many tasks in a sample-efficient manner.

\newpage

\section*{Broader Impact}

Our work is mainly theoretical, but is a step towards creating agents that can solve tasks specified using human-understandable Boolean expressions, which could one day be deployed in practical RL systems. 
We envisage this as an avenue for overcoming the problem of reward misspecification, and for developing safer agents whose goals are readily interpretable by humans.

\begin{ack}

The authors wish to thank the anonymous reviewers for their helpful comments, and Pieter Abbeel, Marc Deisenroth and Shakir Mohamed for their assistance in reviewing a final draft of this paper. 
This work is based on the research supported in part by the National Research Foundation of South Africa (Grant Number: 17808).

\end{ack}

\bibliography{neurips2020}
\bibliographystyle{neurips2020}

\end{document}


\maketitle

\section{Boolean Algebra Definition}

\begin{mydef}

A Boolean algebra is a set $\mathcal{B}$ equipped with the binary operators $\vee$ (disjunction) and $\wedge$ (conjunction), and the unary operator $\neg$ (negation), which satisfies the following \textit{Boolean algebra axioms} for $a, b, c$ in $\mathcal{B}$:

\begin{enumerate}[label=(\roman*)]
    \item Idempotence: $a \wedge a = a \vee a = a$.
    \item Commutativity: $a \wedge b = b \wedge a $ and $a \vee b = b \vee a $.
    \item Associativity: $a \wedge (b \wedge c) = (a \wedge b) \wedge c$ and $a \wedge (b \vee c) = (a \vee b) \vee c$.
    \item Absorption: $a \wedge (a \vee b) = a \vee (a  \wedge b) = a $.
    \item Distributivity: $a \wedge (b \vee c) = (a \wedge b) \vee (a \wedge c) $ and $a \vee (b \wedge c) = (a \vee b) \wedge (a \vee c)$.
    \item Identity: there exists $\mathbf{0}, \mathbf{1}$ in $\mathcal{B}$ such that
    \begin{align*}
        \mathbf{0} \wedge a &= \mathbf{0} \\
        \mathbf{0} \vee a &= a \\
        \mathbf{1} \wedge a &= a \\
        \mathbf{1} \vee a &= \mathbf{1}
    \end{align*}
    \item Complements: for every $a$ in $\mathcal{B}$, there exists an element $a^\prime$ in $\mathcal{B}$ such that $a \wedge a^\prime = \mathbf{0}$ and $a \vee a^\prime = \mathbf{1}$.

\end{enumerate}

\end{mydef}

\section{Proof for Boolean Task Algebra}

\begin{theorem}
Let $\tasks$ be a set of tasks which adhere to Assumption 2. Then $(\tasks,\vee,\wedge,\neg,\mbig,\msmall)$ is a Boolean algebra.

\end{theorem}
\begin{proof}
Let $M_1, M_2 \in \tasks$. We show that $\neg, \vee, \wedge$ satisfy the Boolean properties (i) -- (vii).

\begin{description}[align=right,leftmargin=*,labelindent=\widthof{(i)--(v):}]
    \item [(i)--(v):] These easily follow from the fact that the $\min$ and $\max$ functions satisfy the idempotent, commutative, associative, absorption and distributive laws.
    \item [(vi):] Let $\rand{\mbig}{M_1}$ and $\reward_{M_1}$ be the reward functions for $\mbig \wedge M_1$ and $M_1$ respectively. Then for all $(s, a)$ in $\state \times \action$, 
    \begin{align*}
        \rand{\mbig}{M_1}(s, a) &= 
            \begin{cases}
                    \min\{\rbig, \reward_{M_1}(s, a)\}, & \text{if } s \in \goals \\
                    \min \{\constantr, \constantr \}, & \text{otherwise.}
            \end{cases} \\
        &= 
            \begin{cases}
                    \reward_{M_1}(s, a), & \text{if } s \in \goals \\
                    \constantr, & \text{otherwise.}
            \end{cases} && \text{(}\reward_{M_1}(s, a) \in \{ \rsmall, \rbig \} \text{ for } s \in \goals \text{)}  \\
        &= \reward_{M_1}(s, a).
    \end{align*}
Thus $\mbig \wedge M_1 = M_1$.
Similarly $\mbig \vee M_1 = \mbig$, $\msmall \wedge M_1 = \msmall$, and $\msmall \vee M_1 = M_1$ . 
Hence $\msmall$ and $\mbig$ are the universal bounds of $\tasks$.
\item [(vii):] Let $\rand{M_1}{\neg M_1}$ be the reward function for $M_1 \wedge \neg M_1$. Then for all $(s,a)$ in $\state \times \action$,
    \begin{align*}
        \rand{M_1}{\neg M_1}(s, a) &= 
            \begin{cases}
                    \min\{\reward_{M_1}(s,a), (\rbig + \rsmall) - r_{M_1}(s,a) \}, & \text{if } s \in \goals \\
                    \min \{\constantr, \left(\constantr + \constantr \right) - \constantr\}, & \text{otherwise.}
            \end{cases} \\
        &= 
            \begin{cases}
                    \rsmall, & \text{if } s \in \goals \text{ and } \reward_{M_1}(s, a) = \rbig \\
                    \rsmall, & \text{if } s \in \goals \text{ and } \reward_{M_1}(s, a) = \rsmall \\
                    \constantr, & \text{otherwise.}
            \end{cases}  \\
        &= \reward_{\msmall}(s, a).
    \end{align*}
Thus $M_1 \wedge \neg M_1 = \msmall$, and similarly $M_1 \vee \neg M_1 = \mbig$.
\end{description}
\end{proof}

\section{Proofs of Properties of Extended Value Functions}

\begin{lemma}
Let $\reward_{M}, \rbar_{M}, \qstar_M, \qstarbar_M$ be the reward function, extended reward function, optimal Q-value function, and optimal extended Q-value function for a task $M$ in $\tasks$. 
Then for all $(s, a)$ in $\state \times \action$, we have
\begin{enumerate*}[label=(\roman*)]
    \item $\reward_M(s, a) = \max\limits_{g \in \goals} \rbar_M(s, g, a)$, and
    \item $\qstar_M(s, a) = \max\limits_{g \in \goals} \qstarbar_M(s, g, a)$.
\end{enumerate*}
\label{lem:1}
\end{lemma}
\begin{proof}
~
\begin{description}[align=right,leftmargin=*,labelindent=\widthof{(ii):}]
    \item [(i):]
        \begin{align*}
        \max\limits_{g \in \goals} \rbar_M(s, g, a) &= 
            \begin{cases}
                    \max\{\rbarmin, \reward_{M}(s, a)\}, & \text{if } s \in \goals \\
                    \max\limits_{g \in \goals} \reward_M(s,a), & \text{otherwise.}
            \end{cases} \\
        &= \reward_{M}(s, a) && \text{(} \rbarmin \leq \rmin \leq \reward_{M}(s, a) \text{ by definition)}.
    \end{align*}
    \item [(ii):] Each $g$ in $\goals$ can be thought of as defining an MDP $M_g \vcentcolon = (\state, \action, \dynamics, \reward_{M_g})$ with reward function $\reward_{M_g}(s, a) \vcentcolon = \rbar_M(s, g, a)$ and optimal Q-value function $\qstar_{M_g}(s, a) = \qstarbar_M(s, g, a)$. Then using (i) we have $\reward_M(s, a) =  \max\limits_{g \in \goals} \reward_{M_g}(s, a)$ and from \citet[Corollary 1]{vanniekerk19}, we have that $\qstar_M(s, a) =  \max\limits_{g \in \goals} \qstar_{M_g}(s, a) =  \max\limits_{g \in \goals} \qstarbar_M(s, g, a)$.
\end{description}
\end{proof}

\begin{lemma}
Denote $\state^{-} = \state \setminus \goals  $ as the non-terminal states of $\tasks$.  Let $M_1, M_2 \in \tasks$, and let each $g$ in $\goals$ define MDPs $M_{1,g}$ and $M_{2,g}$ with reward functions 
\[
r_{M_{1,g}} \vcentcolon = \rbar_{M_1}(s,g,a) \text{ and } r_{M_{2,g}} \vcentcolon = \rbar_{M_2}(s,g,a) \text{ for all } (s,a) \text{ in } \state \times \action.
\]
Then for all $g$ in $\goals$ and $s$ in $\state^{-}$,
\[
\pistar_g(s) \in \argmax\limits_{a \in \action}\qstar_{M_{1, g}}(s, a) \text{ iff } \pistar_g(s) \in \argmax\limits_{a \in \action}\qstar_{M_{2, g}}(s, a).
\]
\end{lemma}
\begin{proof}
Let $g \in \goals, s \in \state^-$ and let $\pistar_g$ be defined by
\[ 
\pistar_g(s^\prime) \in \argmax\limits_{a \in \action} \qstar_{M_1, g}(s,a) \text{ for all } s^\prime \in \state. 
\]
If $g$ is unreachable from $s$, then we are done since for all $(s^\prime, a)$ in $\state \times \action$ we have 
\begin{align*}
    g \neq s^\prime &\implies \reward_{M_{1,g}}(s^\prime, a) = \begin{cases}
    \rbarmin, &\text{ if } s^\prime \in \goals \\
    r_{s^\prime,a}, &\text{ otherwise}
    \end{cases}
    = \reward_{M_{2,g}}(s^\prime, a) \\
    &\implies M_{1,g} = M_{2,g}.
\end{align*}
If $g$ \textit{is} reachable from $s$, then we show that following $\pistar_g$ must reach $g$.
Since $\pistar_g$ is proper, it must reach a terminal state $g^\prime \in \goals$.
Assume $g^\prime \neq g$. 
Let $\pi_g$ be a policy that produces the shortest trajectory to $g$. 
Let $G^{\pistar_g}$ and $G^{\pi_g}$ be the returns for the respective policies. Then,
\begin{align*}
    &G^{\pistar_g} \geq G^{\pi_g} \\
    &\implies G^{\pistar_g}_{T-1} + r_{M_{1,g}}(g^\prime, \pistar_g(g^\prime)) \geq G^{\pi_g}, \\ 
    &\text{ where } G^{\pistar_g}_{T-1} = \sum_{t=0}^{T-1}\reward_{M_{1,g}}(s_t,\pistar_g(s_t)) \text{ and } T \text{ is the time at which } g^\prime \text{ is reached.}     \\
    &\implies G^{\pistar_g}_{T-1} + \rbarmin \geq G^{\pi_g}, \text{ since } g \neq g^\prime \in \goals \\
    &\implies \rbarmin \geq G^{\pi_g} - G^{\pistar_g}_{T-1} \\
    &\implies (\rmin - \rmax)D \geq G^{\pi_g} - G^{\pistar_g}_{T-1} , \text{ by definition of } \rbarmin \\
    &\implies G^{\pistar_g}_{T-1} - \rmax D \geq G^{\pi_g} - \rmin D, \text{ since } G^{\pi_g} \geq \rmin D \\
    &\implies G^{\pistar_g}_{T-1} - \rmax D \geq 0 \\
    &\implies G^{\pistar_g}_{T-1} \geq \rmax D.
\end{align*}
But this is a contradiction since the result obtained by following an optimal trajectory up to a terminal state without the reward for entering the terminal state must be strictly less that receiving $\rmax$ for every step of the longest possible optimal trajectory. 
Hence we must have $g^\prime = g$. 
Similarly, all optimal policies of $M_{2,g}$ must reach $g$.
Hence $\pistar_g(s) \in \argmax\limits_{a \in \action} \qstar_{M_{2,g}}(s,a)$.
Since $M_1$ and $M_2$ are arbitrary elements of $\tasks$, the reverse implication holds too. 

\end{proof}

\begin{corollary}
Denote $\gstar$ as the sum of rewards starting from $s$ and taking action $a$ up until, but not including, $g$. Then let $M \in \tasks$ and $\qstarbar_M$ be the extended Q-value function. Then for all $s \in \state, g \in \goals, a \in \action$, there exists a $\gstar \in \R$ such that
\begin{equation*}
    \qstarbar_M(s, g, a) = \gstar + \rbar_M(s^\prime,g, a^\prime), \text{ where } s^\prime \in \goals \text{ and } a^\prime = \argmax_{b \in \action} \rbar_M(s^\prime,g, b).
\end{equation*}
\label{cor:1}
\end{corollary}
\begin{proof}
This follows directly from Lemma 2. Since all tasks $M \in \tasks$ share the same optimal policy $\pistar_g$ up to (but not including) the goal state $g \in \goals$, their return $G^{\pistar_g}_{T-1} = \sum_{t=0}^{T-1}\reward_M(s_t,\pistar_g(s_t))$ is the same up to (but not including) $g$. 
\end{proof}

\section{Proof for Boolean Extendend Value Functions Algebra}

\begin{theorem}
Let $\goalq$ be the set of optimal extended $\bar{Q}$-value functions for tasks in $\tasks$ which adhere to Assumption 2. Then $(\goalq,\vee,\wedge,\neg,\qstarbarbig,\qstarbarsmall)$ is a Boolean Algebra.

\end{theorem}
\begin{proof}
Let $\qstarbar_{M_1}, \qstarbar_{M_2} \in \goalq$ be the optimal $\bar{Q}$-value functions for tasks $M_1, M_2 \in \tasks$ with reward functions $r_{M_1}$ and $r_{M_2}$. We show that $\neg, \vee, \wedge$ satisfy the Boolean properties (i) -- (vii).

\begin{description}[align=right,leftmargin=*,labelindent=\widthof{(i)--(v):}]
    \item [(i)--(v):] These follow directly from the properties of the $\min$ and $\max$ functions.
    \item [(vi):] For all $(s, g, a)$ in $\state \times \goals \times \action$, 
    \begin{align*}
        (\qstarbarbig \wedge \qstarbar_{M_1})(s,g,a) &= \min \{ (\qstarbarbig(s,g,a),  \qstarbar_{M_1}(s,g,a)  \} \\
        &= \min \{ \gstar + \rbar_{\mbig}(s^\prime,g,a^\prime), \gstar + \rbar_{M_1}(s^\prime,g,a^\prime) \} \quad \text{(Corollary 1)} \\ 
     &= \gstar + \min \{\rbar_{\mbig}(s^\prime,g,a^\prime), \rbar_{M_1}(s^\prime,g,a^\prime) \} \\
     &= \gstar + \rbar_{M_1}(s^\prime,g,a^\prime) \quad \quad \quad \text{(since } \rbar_{M_1}(s^\prime,g, a^\prime)  \in \{ \rsmall, \rbig, \rbarmin  \} \text{)}  \\
     &= \qstarbar_{M_1}(s,g,a).
    \end{align*}
    Similarly, $\qstarbarbig \vee \qstarbar_{M_1} = \qstarbarbig, \qstarbarsmall \wedge \qstarbar_{M_1} = \qstarbarsmall$, and $\qstarbarsmall \vee \qstarbar_{M_1} = \qstarbar_{M_1}$.
    \item [(vii):] For all $(s, g, a)$ in $\state \times \goals \times \action$, 
    \begin{align*}
    (\qstarbar_{M_1} \wedge \neg \qstarbar_{M_1})(s,g,a) &= \min\{ \qstarbar_{M_1}(s,g,a), (\qstarbarbig(s,g,a) - \qstarbarsmall(s,g,a)) - \qstarbar_{M_1}(s,g,a) \} \\
    &= \gstar + \min \{ \rbar_{M_1}(s^\prime,g,a^\prime), (\rbar_{\mbig}(s^\prime,g,a^\prime) + \rbar_{\msmall}(s^\prime,g,a^\prime)) \\
    & -  \rbar_{M_1}(s^\prime,g,a^\prime) \} \\
    &= \gstar + \rbar_{\msmall}(s^\prime,g,a^\prime) \\
    &= \qstarbarsmall(s,g,a).
    \end{align*}
    Similarly, $\qstarbar_{M_1} \vee \neg \qstarbar_{M_1}  = \qstarbarbig$.
\end{description}
\end{proof}

\section{Proof for Zero-shot Composition}

\begin{theorem}
Let $\goalq$ be the set of optimal extended $\bar{Q}$-value functions for tasks in $\tasks$ which adhere to Assumption 1. Then for all $M_1, M_2 \in \tasks$, we have 
\begin{enumerate*}[label=(\roman*)]
    \item $\qstarbar_{\neg M_1} = \neg \qstarbar_{M_1}$,
    \item $\qstarbar_{M_1 \vee M_2} = \qstarbar_{M_1} \vee \qstarbar_{M_2}$, and 
    \item $\qstarbar_{M_1 \wedge M_2} = \qstarbar_{M_1} \wedge \qstarbar_{M_2}$.
\end{enumerate*}

\label{thm:3}
\end{theorem}
\begin{proof}
Let $M_1, M_2 \in \tasks$. Then for all $(s,g,a)$ in $\state \times \goals \times \action$,
~
\begin{description}[align=right,leftmargin=*,labelindent=\widthof{(ii):}]
    \item [(i):]
        \begin{align*}
            \qstarbar_{\neg M_1}(s,g,a) &= \gstar + \rbar_{\neg M_1}(s^\prime,g,a^\prime) \quad \text{(from Corollary \ref{cor:1})} \\
            &= \gstar + (\rbar_{\mbig}(s^\prime,g,a^\prime) + \rbar_{\msmall}(s^\prime,g,a^\prime) ) - \rbar_{M_1}(s^\prime,g,a^\prime) \\
            &= \left[ (\gstar + \rbar_{\mbig}(s^\prime,g,a^\prime)) + (\gstar + \rbar_{\msmall}(s^\prime,g,a^\prime)) \right] - (\gstar + \rbar_{M_1}(s^\prime,g,a^\prime)) \\
            &= \left[ \qstarbarbig(s,g,a) + \qstarbarsmall(s,g,a) \right] - \qstarbar_{M_1}(s,g,a) \\
            &= \neg \qstarbar_{M_1}(s,g,a)
        \end{align*}
    \item [(ii):] 
        \begin{align*}
            \qstarbar_{M_1 \vee M_2}(s,g,a) &= \gstar + \rbar_{M_1 \vee M_2}(s^\prime,g,a^\prime) \\
            &= \gstar + \max \{ \rbar_{M_1}(s^\prime,g,a^\prime), \rbar_{M_2}(s^\prime,g,a^{\prime\prime}) \} \\
            &= \max \{ \gstar + \rbar_{M_1}(s^\prime,g,a^\prime), \gstar + \rbar_{M_2}(s^\prime,g,a^{\prime\prime}) \} \\
            &= \max \{ \qstarbar_{M_1}(s,g,a), \qstarbar_{M_2}(s,g,a) \} \\
            &= (\qstarbar_{M_1} \vee \qstarbar_{M_2})(s,g,a)
        \end{align*}
    \item [(iii):] Follows similarly to (ii).
\end{description}
\end{proof}

\begin{corollary}
Let $\mathcal{F}~: \tasks \to \goalq$ be any map from $\tasks$ to $\goalq$ such that $\mathcal{F}(M) = \qstarbar_M$ for all $M$ in $\tasks$. Then $\mathcal{F}$ is a homomorphism between $(\tasks,\vee,\wedge,\neg,\mbig,\msmall)$ and $(\goalq,\vee,\wedge,\neg,\qstarbarbig,\qstarbarsmall)$.
\label{cor:2}
\end{corollary}
\begin{proof}
This follows from Theorem~\ref{thm:3}.
\end{proof}

\section{Goal-oriented Q-learning}

Below we list the pseudocode for the modified Q-learning algorithm used in the four-rooms domain.

\begin{figure*}[h!]
    \centering
{
\SetAlgoNoLine
\begin{algorithm}[H]
\label{alg:q}
\DontPrintSemicolon
    \SetKwInOut{Input}{Input}
 \Input{Learning rate $\alpha$, discount factor $\gamma$, exploration constant $\varepsilon$, lower-bound extended reward $\rbarmin$}

 Initialise $Q: \mathcal{S} \times \mathcal{S} \times \mathcal{A} \rightarrow \mathbb{R}$ arbitrarily\;
 $\mathcal{G} \leftarrow \varnothing$  \;
 \While{$Q$ is not converged}{
   Initialise state $s$ \;
   \While{$s$ is not terminal}{
   
    \uIf{$\mathcal{G} = \varnothing$}{
        Select random action $a$ \;
    }
    \Else{
           $a \leftarrow \begin{cases}
\argmax\limits_{b \in \mathcal{A}} \left( \max\limits_{t \in \mathcal{G}} Q(s, t, b) \right) & \mbox{with probability  } 1 - \varepsilon  \\
\text{a random action} & \mbox{with probability } \varepsilon 
\end{cases}$
    }
   Choose $a$ from $s$ according to policy derived from $Q$ \;
   Take action $a$, observe $r$ and $s^\prime$ \;
    
   \ForEach{$g \in \mathcal{G}$}{
   
       \uIf{$s^\prime$ is terminal}{
            \uIf{$s^\prime \neq g$}{
            $\delta \leftarrow \rbarmin$
            }
            \Else{
            $\delta \leftarrow r - Q(s, g, a)$
            }
       }
       \Else{
        $\delta \leftarrow  r + \gamma \max_b Q(s^\prime, g, b) - Q(s, g, a)$ 
       }
   
        $Q(s, g, a) \leftarrow Q(s, g, a) + \alpha \delta$ \;
    }

   $s \leftarrow s^\prime$ \;
   }
   $\mathcal{G} \leftarrow \mathcal{G} \cup \{s\}$ \;
 }
 \Return{$Q$}
 \caption{Goal-oriented $Q$-learning}
\end{algorithm}
}
    \caption{A $Q$-learning algorithm for learning extended value functions. Note that the greedy action selection step is equivalent to generalised policy improvement \citep{barreto17} over the set of extended value functions. }
    \label{fig:qlearn}
\end{figure*}

\newpage

\section{Investigating Practical Considerations}

The theoretical results presented in this work rely on Assumptions 1 and 2, which restrict the tasks' transition dynamics and reward functions in potentially problematic ways. 
Although this is necessary to prove that Boolean algebraic composition results in optimal value functions, in this section we investigate whether these can be practically ignored.
In particular, we investigate three restrictions: 
\begin{enumerate*}[label=(\roman*)]
    \item the requirement that tasks share the same terminal states,
    \item the impact of using dense rewards, and
    \item the requirement that tasks have deterministic transition dynamics.
\end{enumerate*}

\subsection{Four Rooms Experiments}

We use the same setup as the experiment outlined in Section 4.
We first investigate the difference between using sparse and dense rewards. 
Our sparse reward function is defined as 
\[
r_{\text{sparse}}(s, a) = \begin{cases} 
    2 &\mbox{if } s \in \goals \\
    -0.1 & \mbox{otherwise,} 
\end{cases}
\]
and we use a dense reward function similar to \citet{peng19}:

\[
r_{\text{dense}}(s, a) = \dfrac{0.1}{|\goals|} \sum\limits_{g \in \goals}  \exp(-\dfrac{|s - g|^2}{4}) + r_{\text{sparse}}(s, a)
\]

Using this dense reward function, we again learn to solve the two base task $M_T$ (reaching the centre of the top two rooms) and  $M_L$ (reaching the centre of the left two rooms).
We then compose them to solve a variety of tasks, with the resulting value functions illustrated by Figure~\ref{fig:composition_dense}.

\begin{figure*}[h!]
    \centering
    \begin{subfigure}[t]{0.15\textwidth}
        \centering
        \includegraphics[height=0.85in]{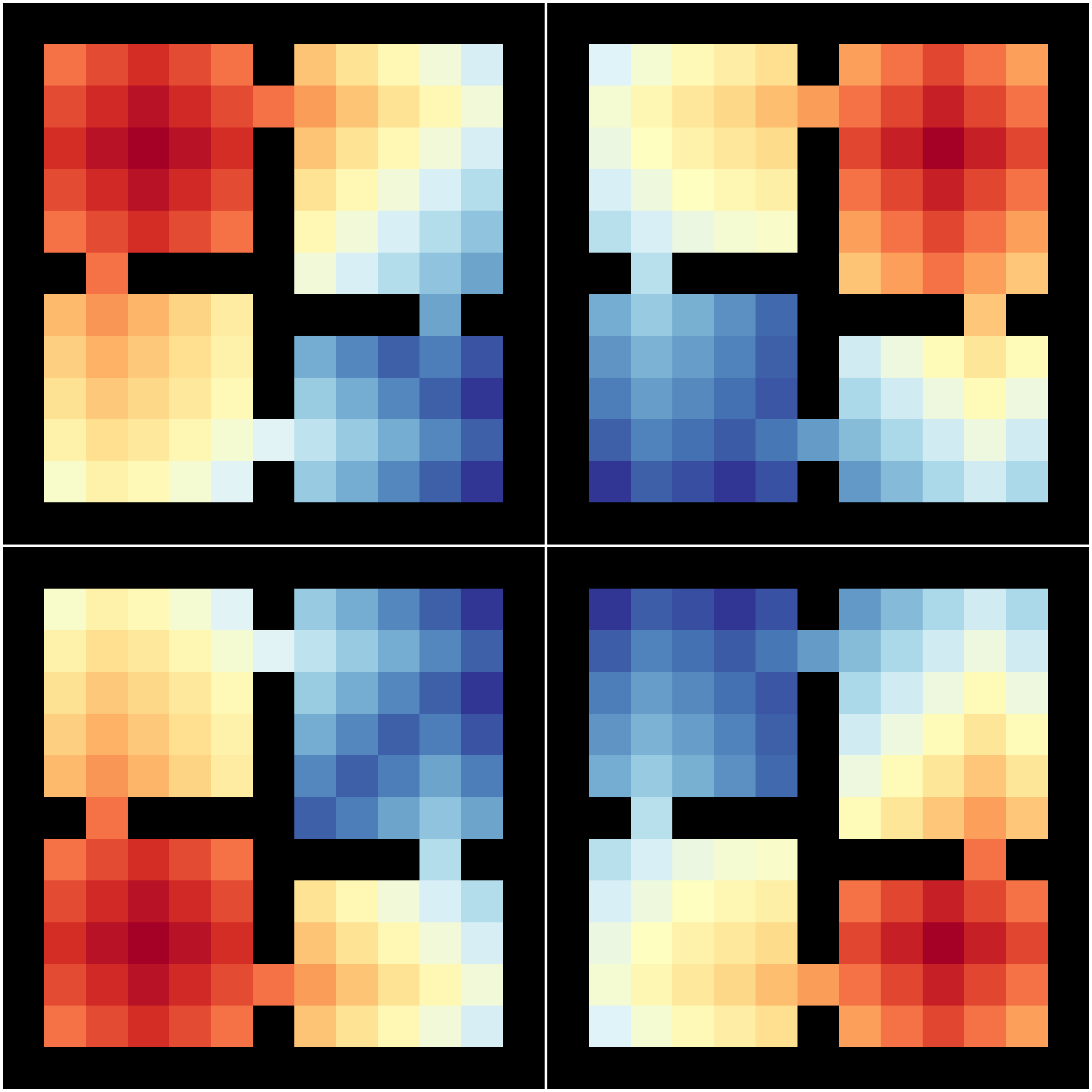}
        \includegraphics[height=0.85in]{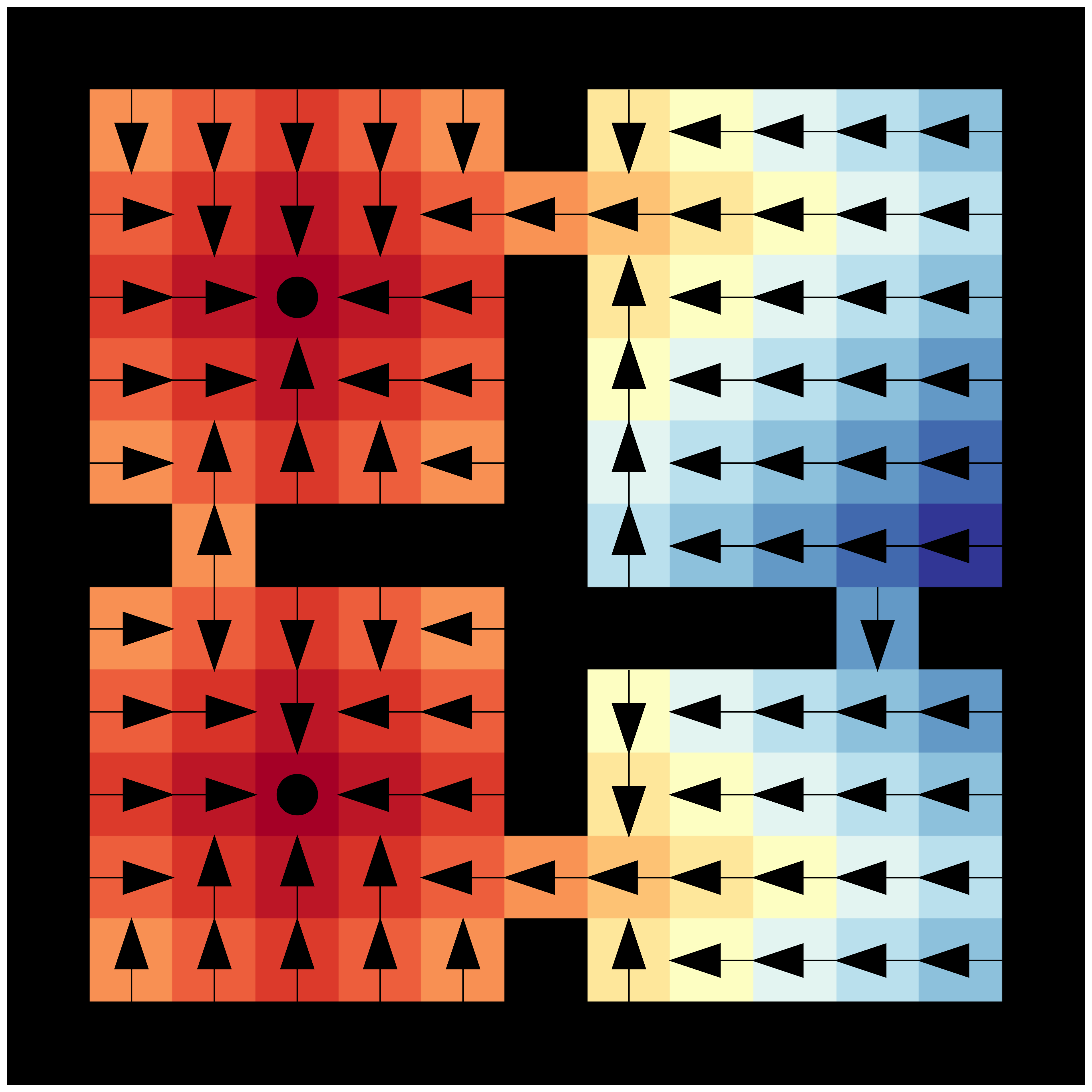}
        \caption{$M_{\text{L}}$}
        \label{fig:a-dense}
    \end{subfigure}%
    ~ 
        \begin{subfigure}[t]{0.15\textwidth}
        \centering
        \includegraphics[height=0.85in]{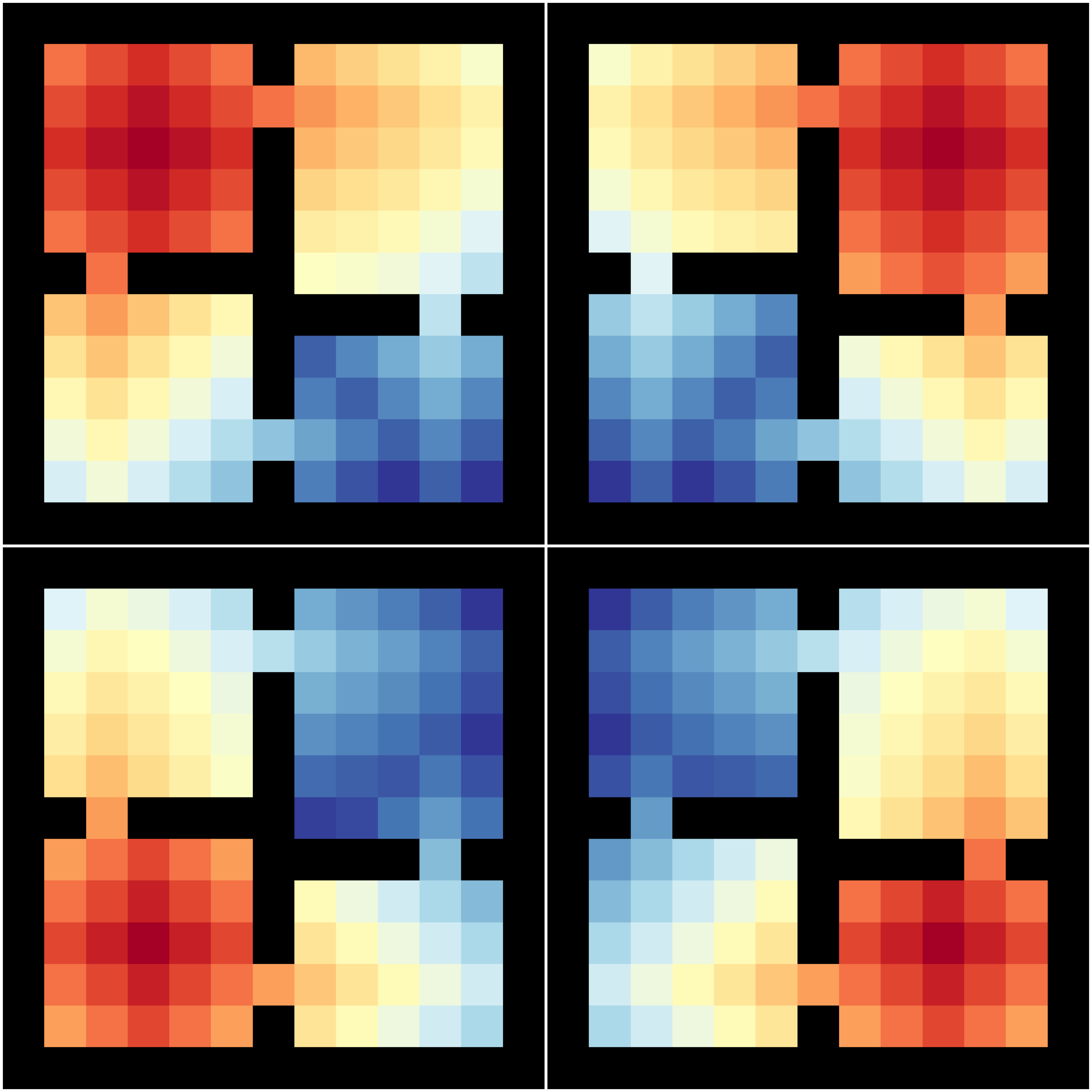}
        \includegraphics[height=0.85in]{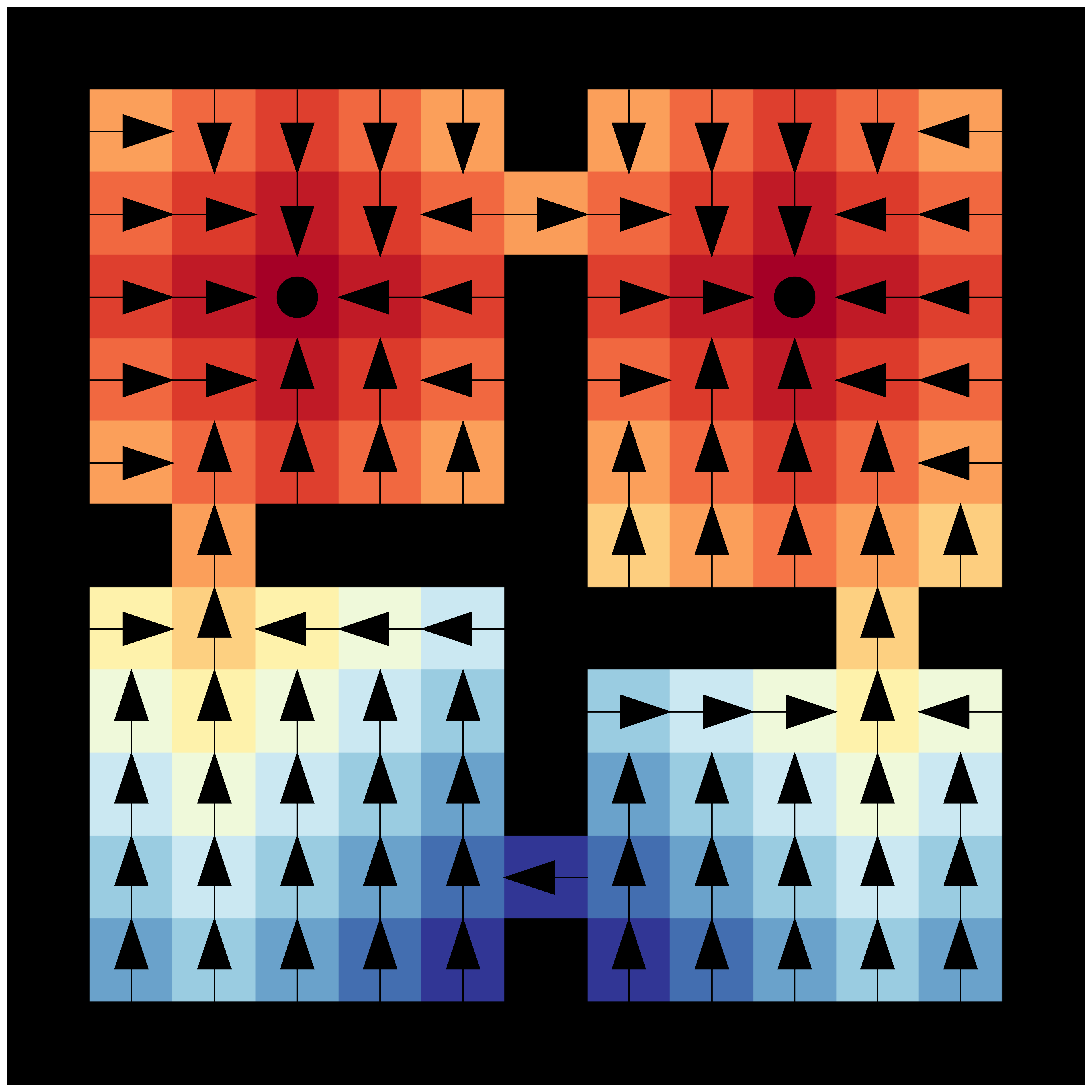}
        \caption{$M_{\text{T}}$}
        \label{fig:b-dense}
    \end{subfigure}%
    ~ 
    \begin{subfigure}[t]{0.15\textwidth}
        \centering
        \includegraphics[height=0.85in]{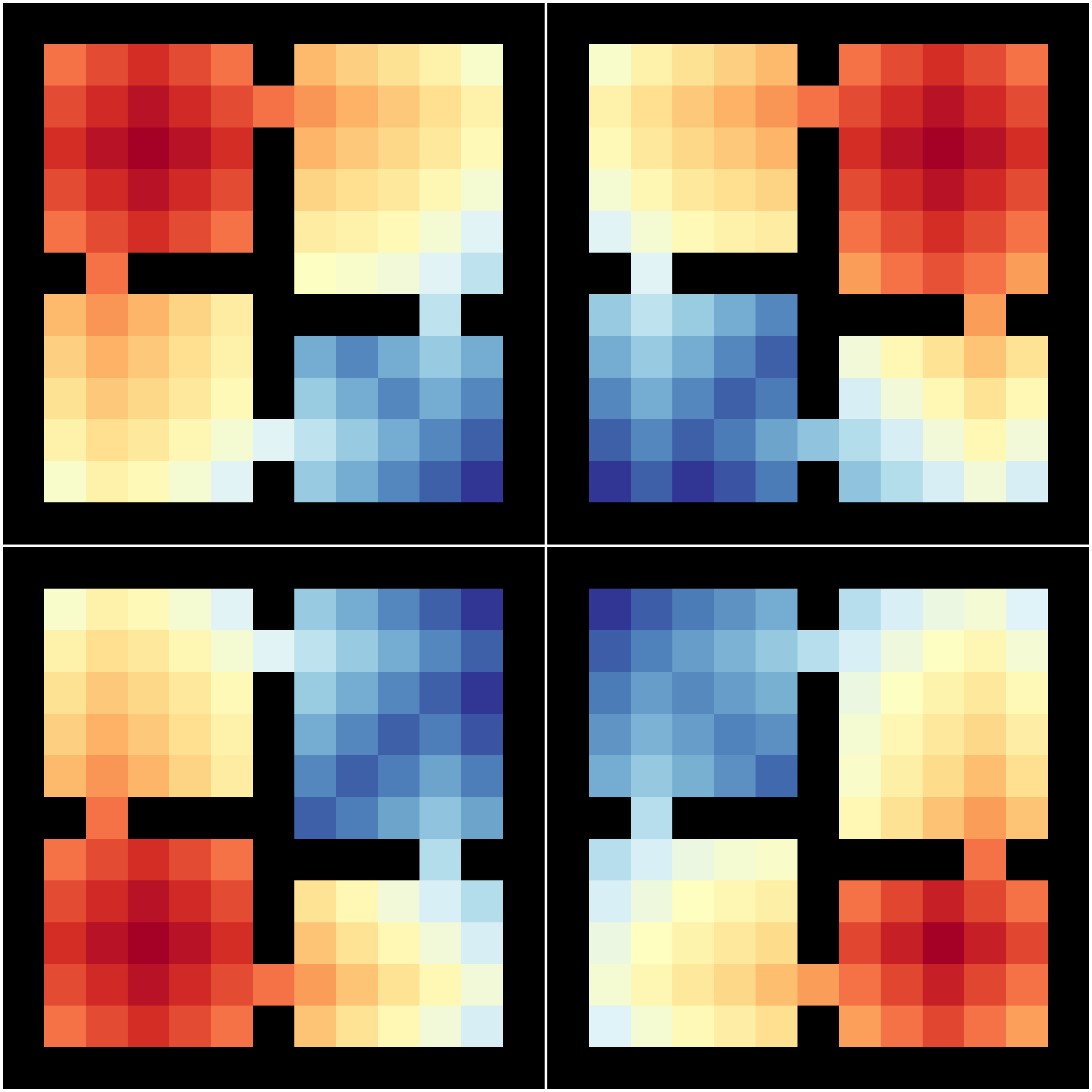}
        \includegraphics[height=0.85in]{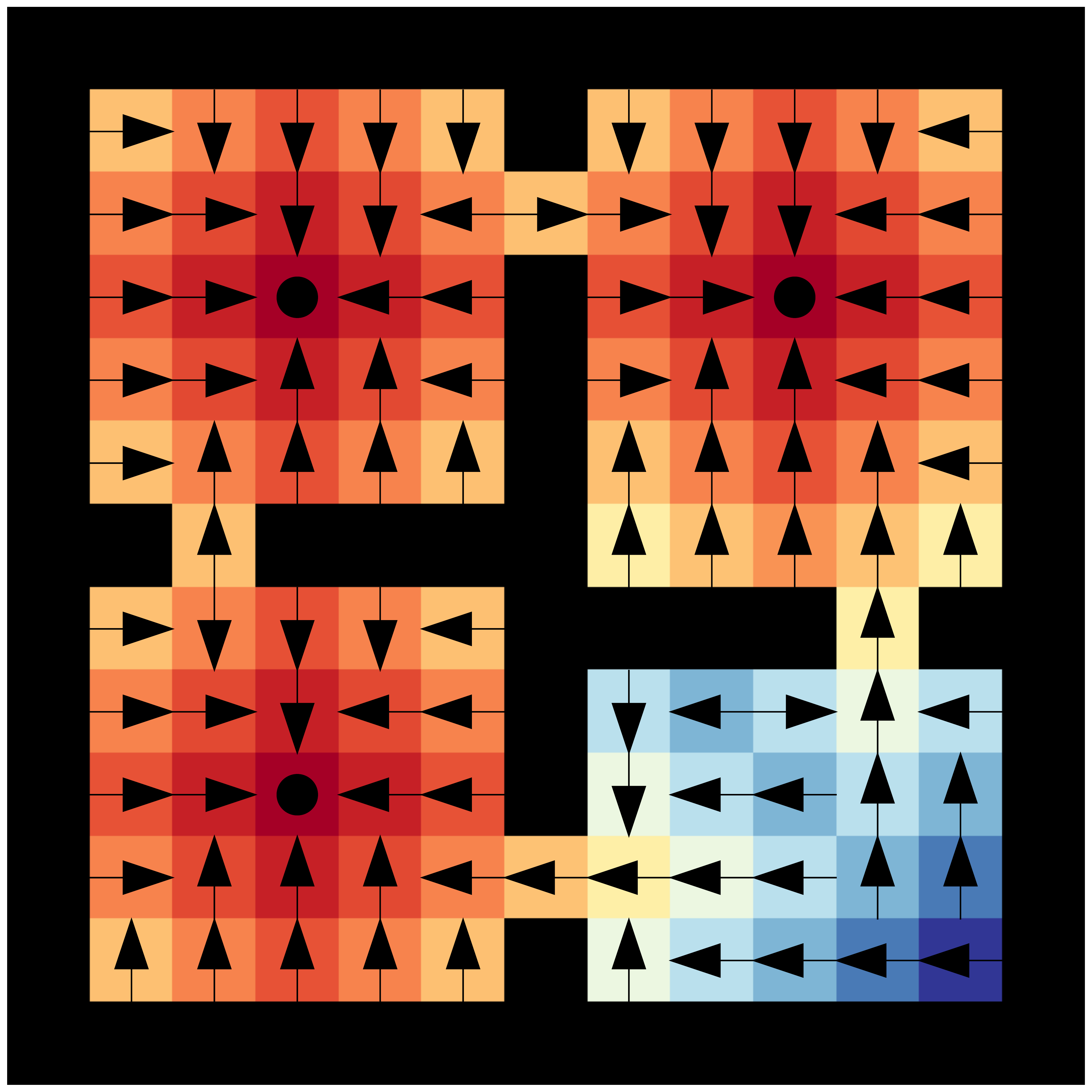}
        \caption{$M_{\text{L}} \vee M_{\text{T}}$}
        \label{fig:c-dense}
    \end{subfigure}%
    ~
    \begin{subfigure}[t]{0.15\textwidth}
        \centering
        \includegraphics[height=0.85in]{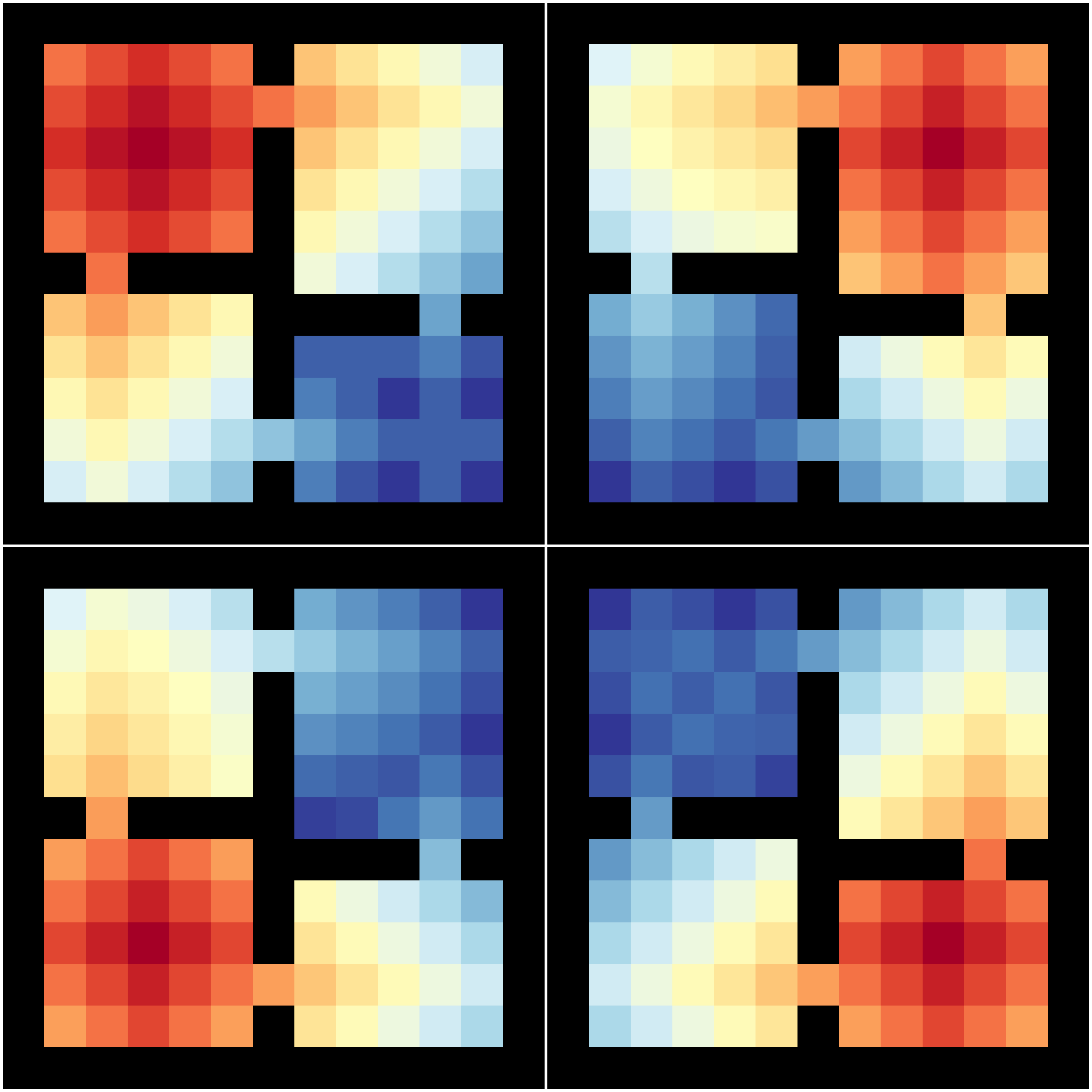}
        \includegraphics[height=0.85in]{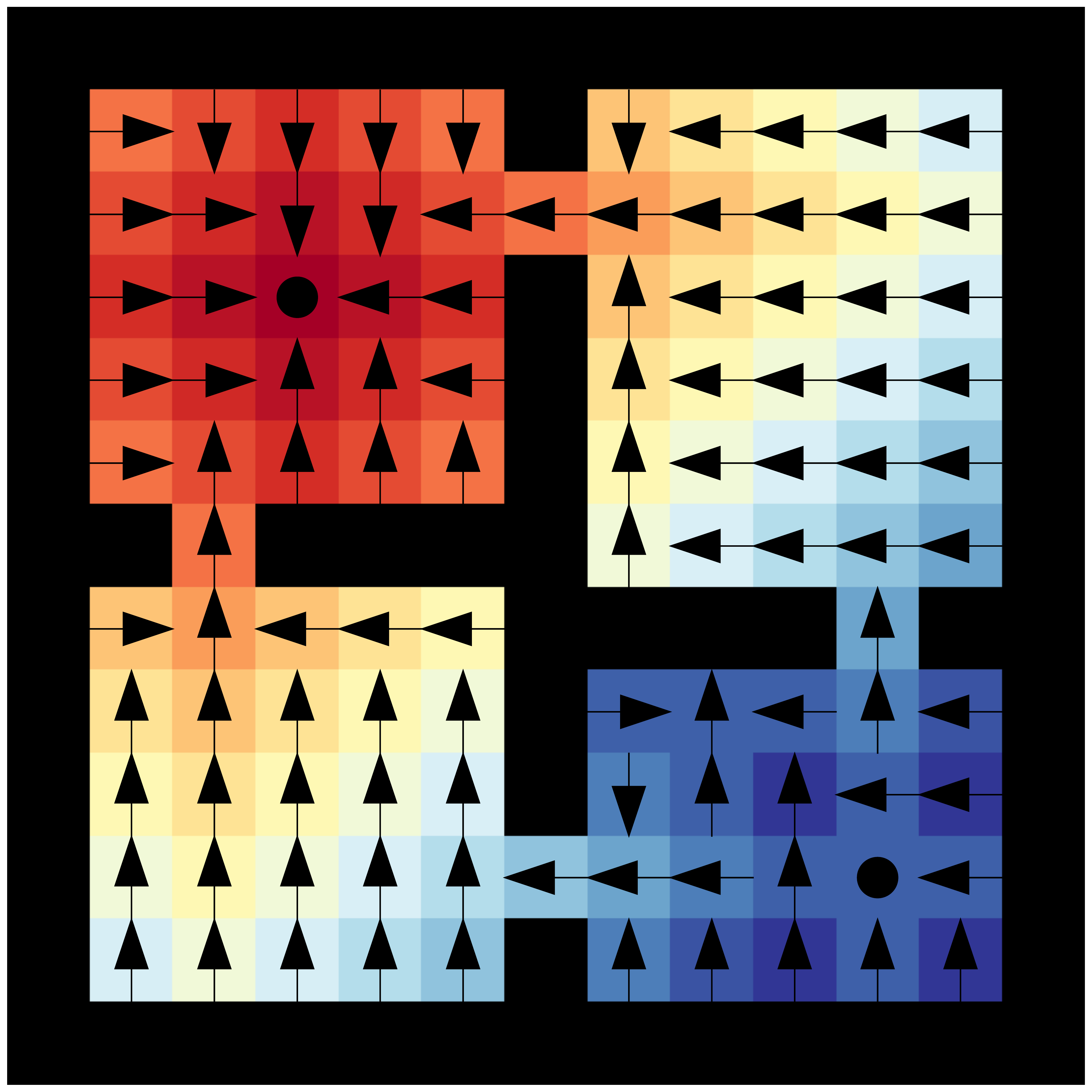}
        \caption{$M_{\text{L}} \wedge M_{\text{T}}$}
        \label{fig:d-dense}
    \end{subfigure}%
    ~ 
    \begin{subfigure}[t]{0.15\textwidth}
        \centering
        \includegraphics[height=0.85in]{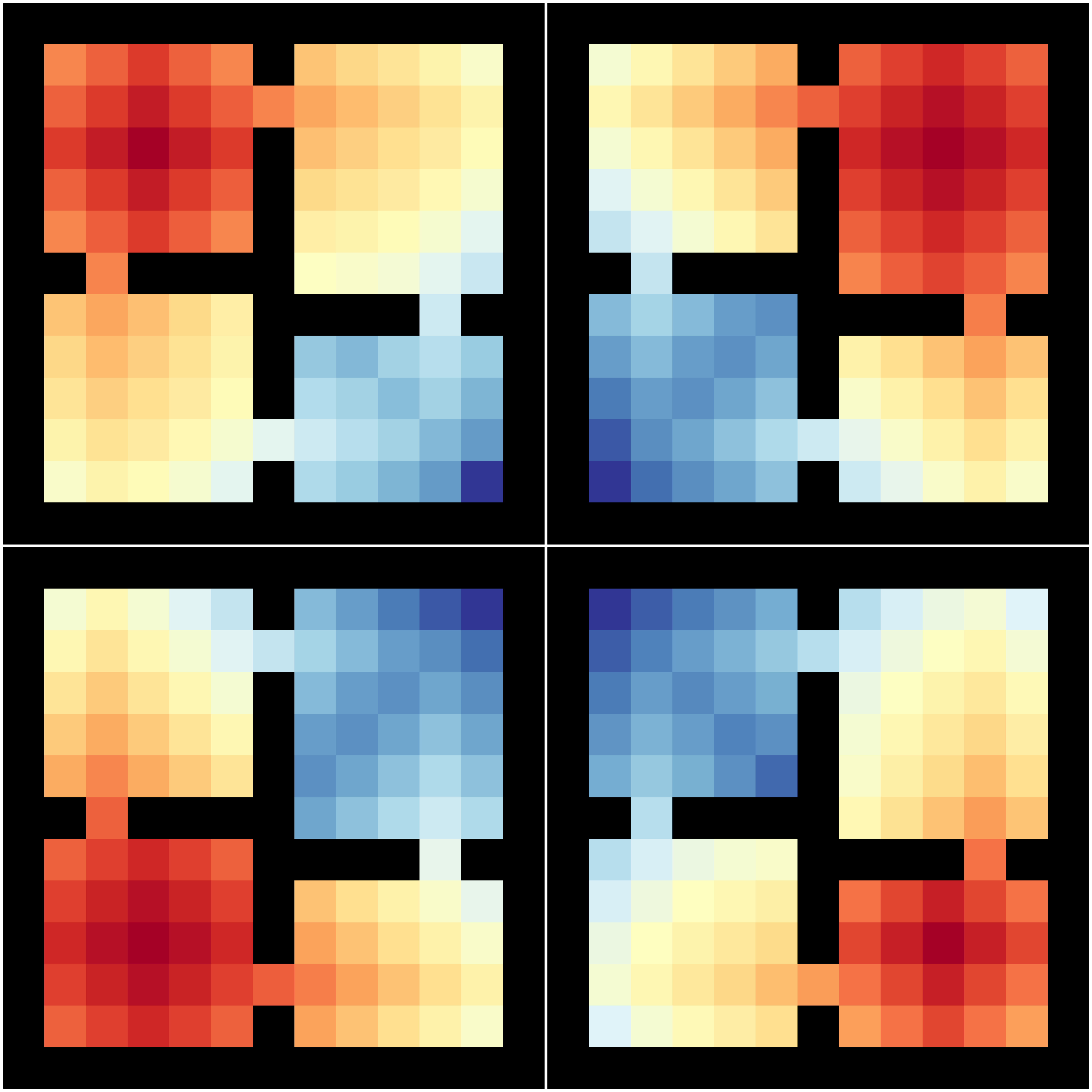}
        \includegraphics[height=0.85in]{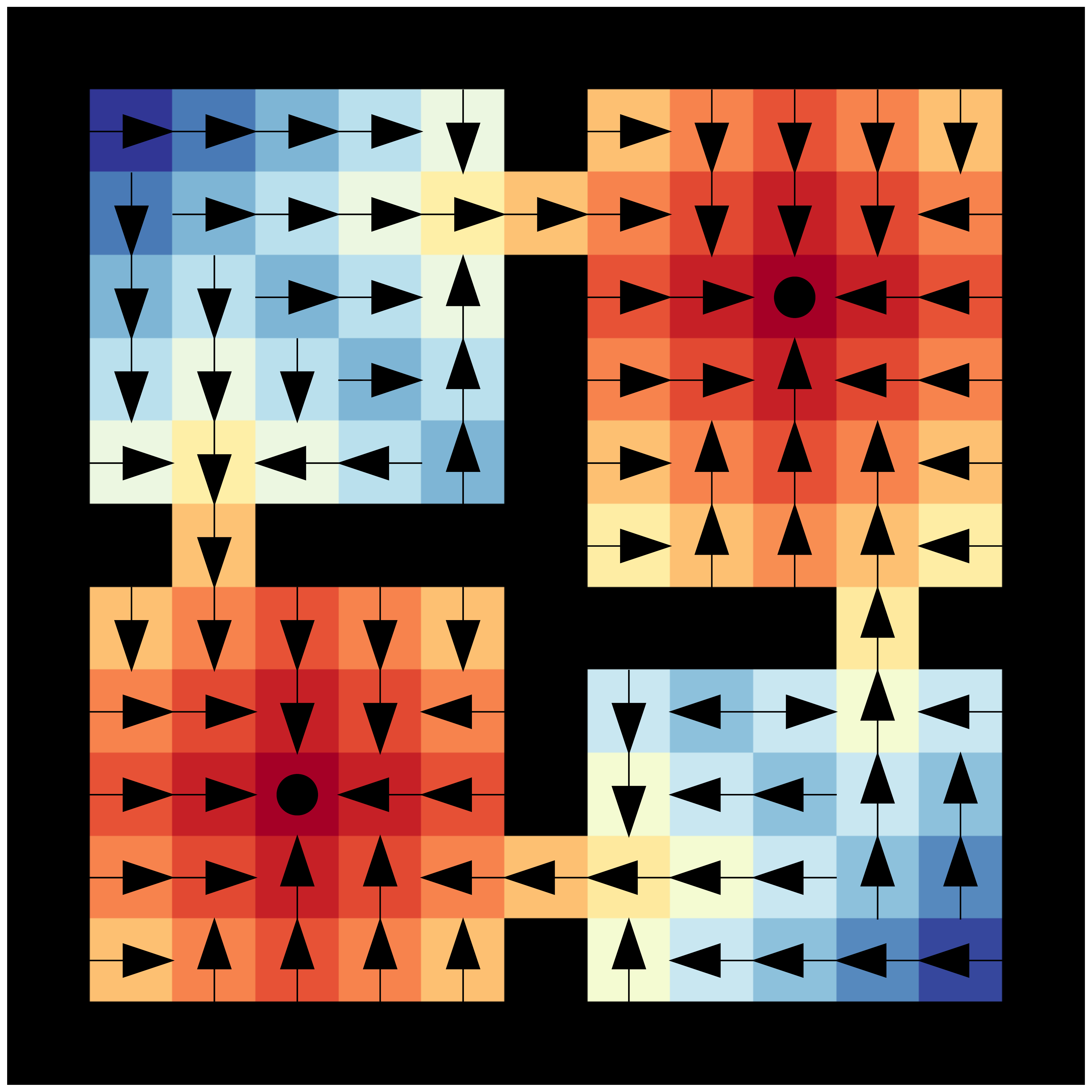}
        \caption{$M_{\text{L}} \veebar M_{\text{T}}$}
        \label{fig:e-dense}
    \end{subfigure}%
    ~ 
    \begin{subfigure}[t]{0.15\textwidth}
        \centering
        \includegraphics[height=0.85in]{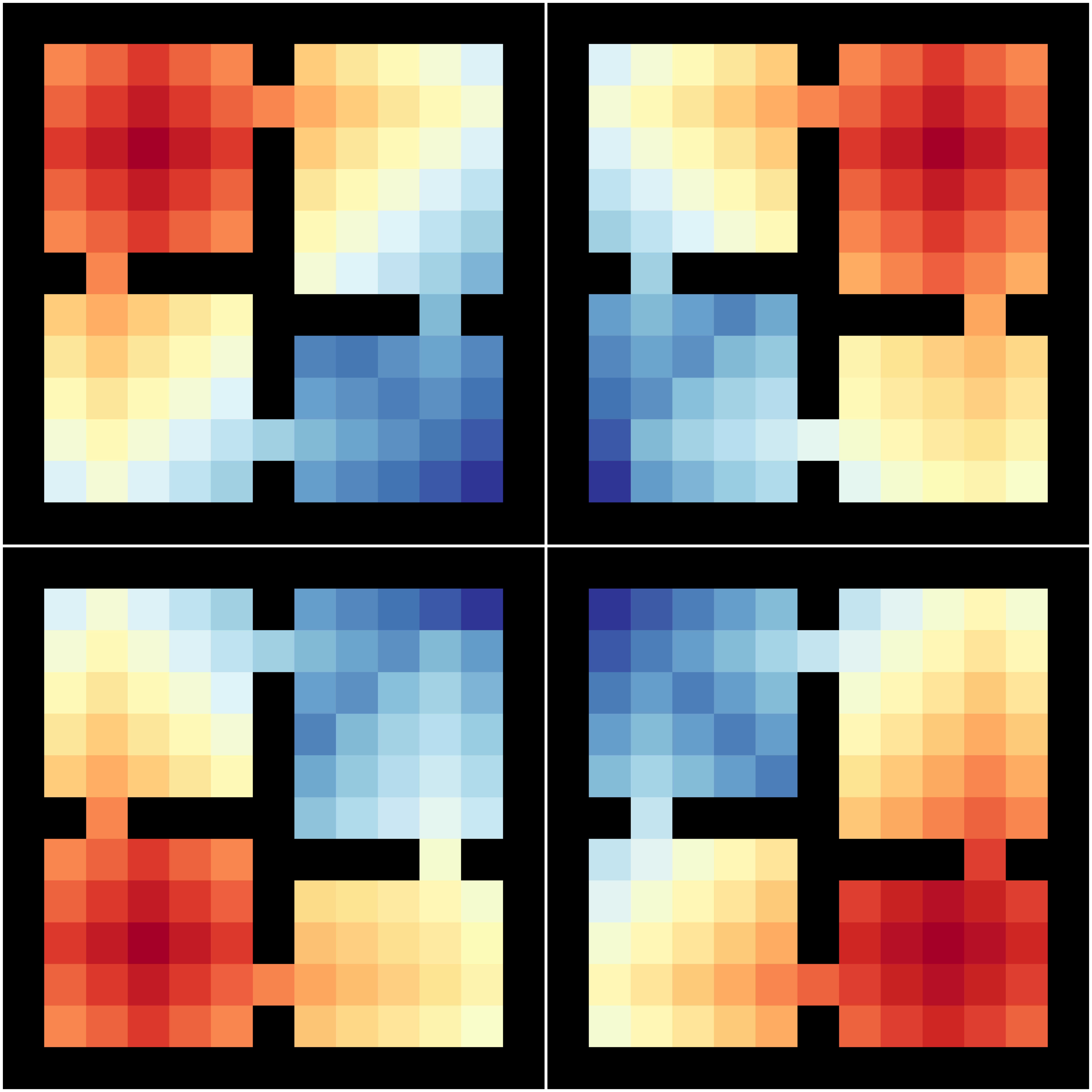}
        \includegraphics[height=0.85in]{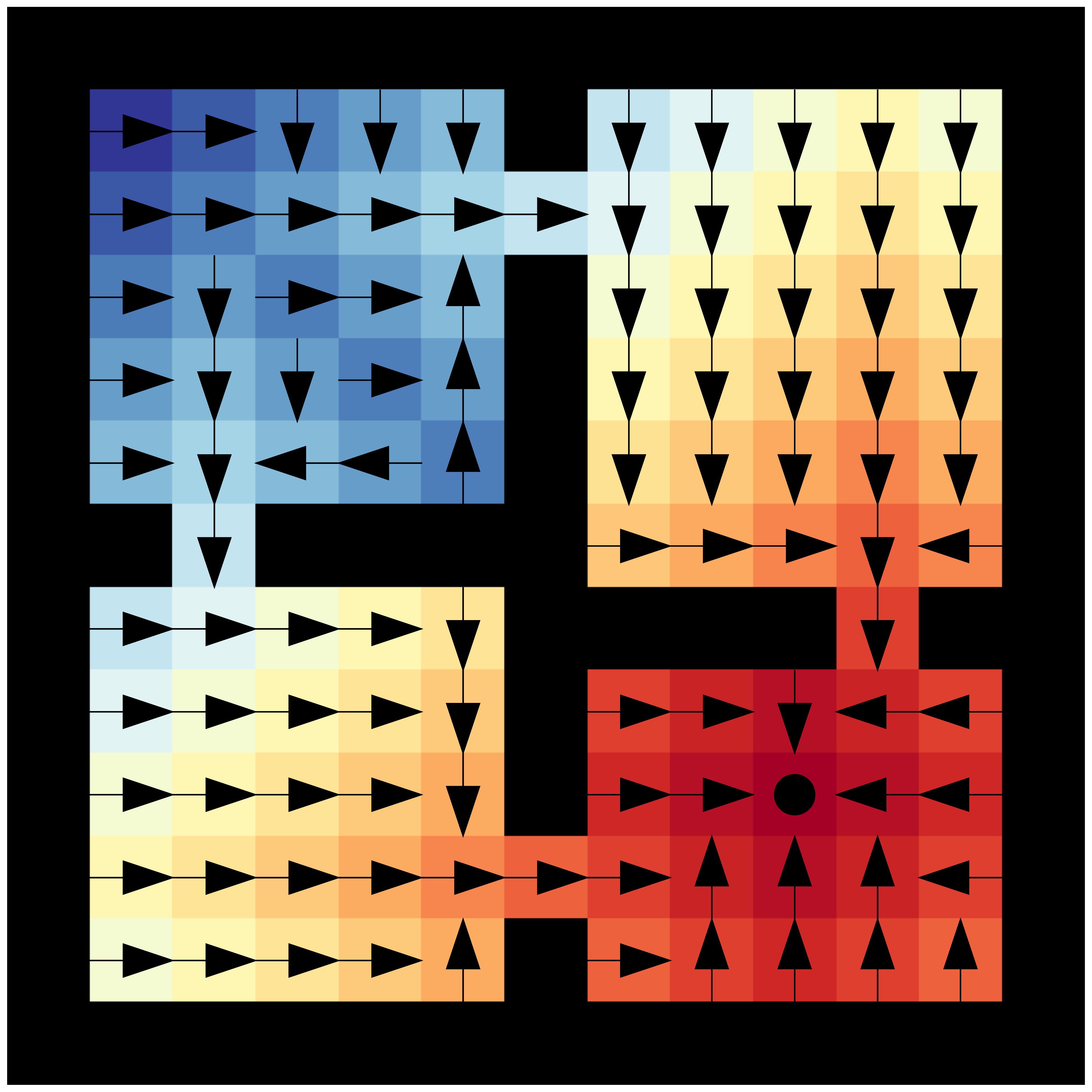}
        \caption{$M_{\text{L}} \barvee M_{\text{T}}$}
        \label{fig:f-dense}
    \end{subfigure}%
    \caption{An example of Boolean algebraic composition using the learned extended value functions with dense rewards. The top row shows the extended value functions while the bottom one shows the recovered regular value functions obtained my maximising over goals. Arrows represent the optimal action in a given state. (\subref{fig:a-dense}--\subref{fig:b-dense}) The learned optimal goal oriented value functions for the base tasks with dense rewards. (\subref{fig:c-dense}) Disjunctive composition.  (\subref{fig:d-dense}) Conjunctive composition.  (\subref{fig:e-dense}) Combining operators to model exclusive-or composition.  (\subref{fig:f-dense}) Composition that produces logical nor. We note that the resulting value functions are very similar to those produced in the sparse reward setting.}
    \label{fig:composition_dense}
\end{figure*}

We also modify the domain so that tasks need not share the same terminating states (that is, if the agent enters a terminating state for a \textit{different} task, the episode does not terminate and the agent can continue as if it were a normal state). 
This results in four versions of the experiment:

\begin{enumerate}[label=(\roman*)]
    \item \texttt{sparse reward, same absorbing set}
    \item \texttt{sparse reward, different absorbing set}
    \item \texttt{dense reward, same absorbing set}
    \item \texttt{dense reward, different absorbing set}
\end{enumerate}

We learn extended value functions for each of the above setups, and then compose them to solve each of the $2^4$ tasks representable in the Boolean algebra. 
We measure each composed value function by evaluating its policy in the sparse reward setting, averaging results over 100000 episodes.
The results are given by Figure~\ref{fig:dense_exp}.

\begin{figure}[h!]
    \centering
    \includegraphics[width=0.95\linewidth]{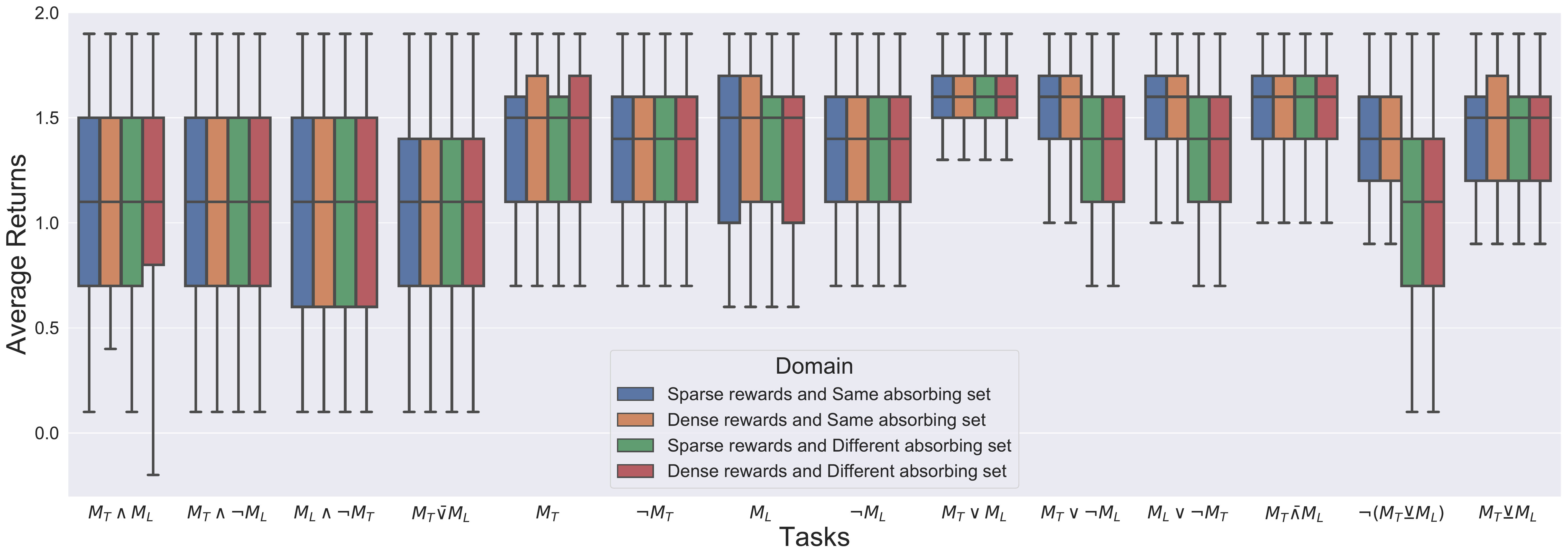}
    \caption{Box plots indicating returns for each of the 16 compositional tasks, and for each of the four variations of the domain. Results are collected over 100000 episodes with random start positions.}
    \label{fig:dense_exp}
\end{figure}

Our results indicate that extended value functions learned in the theoretically optimal manner (\texttt{sparse reward, same absorbing set}) are indeed optimal.
However, for the majority of the tasks, relaxing the restrictions on terminal states and reward functions results in policies that are either identical or very close to optimal. 

\begin{figure*}[h!]
    \centering
    \begin{subfigure}[t]{1\textwidth}
        \centering
        \includegraphics[width=0.95\linewidth]{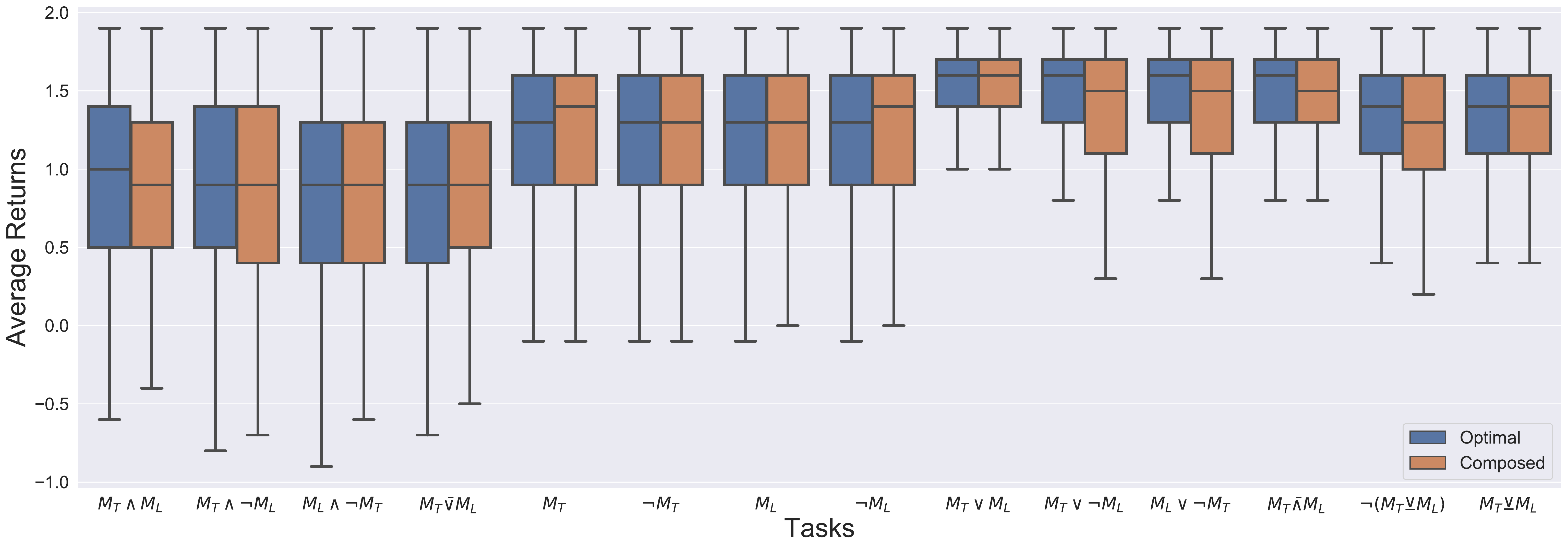}
        \caption{\textit{sp} = 0.1}
        \label{fig:w}
    \end{subfigure}%
    \\
        \centering
    \begin{subfigure}[t]{1\textwidth}
        \centering
        \includegraphics[width=0.95\linewidth]{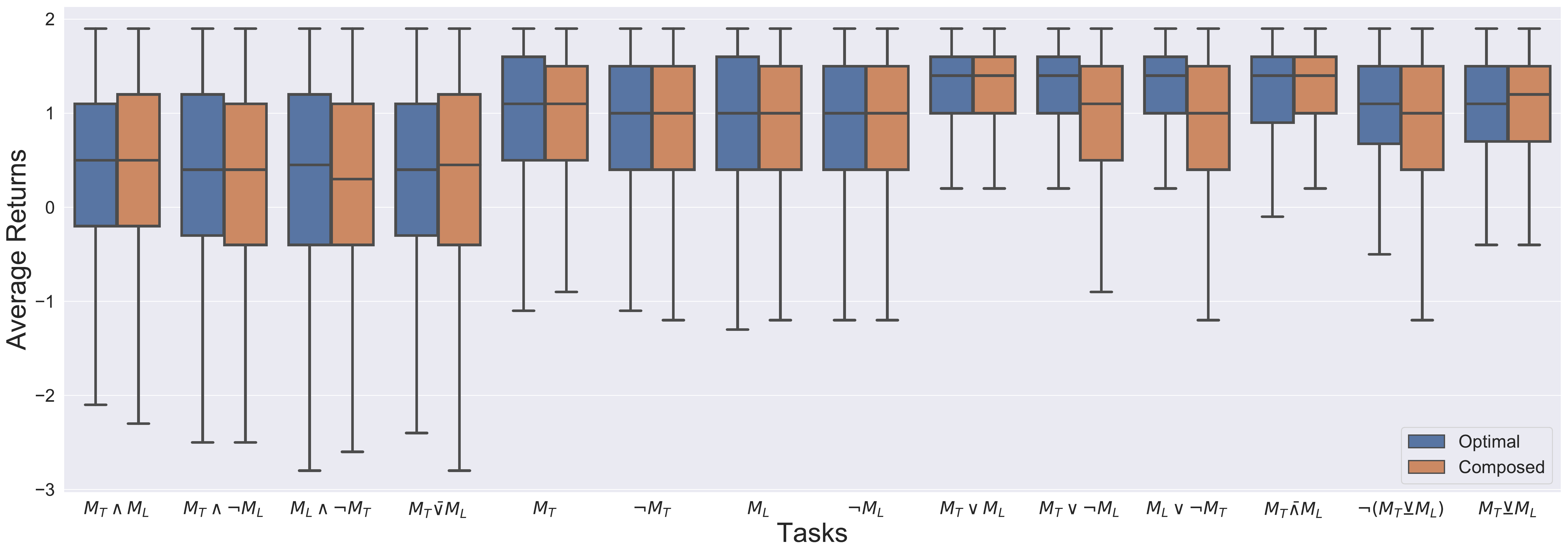}
        \caption{\textit{sp} = 0.3}
        \label{fig:u}
    \end{subfigure}%
    \caption{Box plots indicating returns for each of the 16 compositional tasks, and for each of the slip probabilities. Results are collected over 100000 episodes with random start positions.}
    \label{fig:dense_sp_exp}
\end{figure*}

Finally we investigate the effect of stochastic transition dynamics in addition to dense rewards and different absorbing sets. The domain is modified such that for all tasks there is a slip probability (\textit{sp}) when the agent takes actions in any of the cardinal directions. That is with probability \textit{1-sp} the agent goes in the direction it chooses and with probability \textit{sp} it goes in one of the other 3 chosen uniformly at random. The results are given in Figure~\ref{fig:dense_sp_exp}. Our results show that even when the transition dynamics are stochastic, the learned extended value functions can be composed to produce policies that are identical or very close to optimal.

In summary, we have shown that our compositional approach offers strong empirical performance, even when the theoretical assumptions are violated.

\subsection{Function Approximation Experiments}

In this section we investigate whether we can again loosen some of the restrictive assumptions when tackling high-dimensional environments.
In particular, we run the same experiments as those presented in Section 5, but modify the domain so that 
\begin{enumerate*}[label=(\roman*)]
  \item tasks need not share the same absorbing set,
  \item the \texttt{pickup-up} action is removed since the only terminal states are reaching the desired/goal objects (the agent immediately collects an object when reaching it), and
  \item the position of every object is randomised at the start of each episode.
\end{enumerate*}

We first learn to solve three base tasks: collecting purple objects (Figure~\ref{fig:purple_dts}) collecting blue objects (Figure~\ref{fig:blue_dts}) and collecting squares (Figure~\ref{fig:square_dts}). Notice that because the pickup action is removed, the environment terminates upon touching a desired object and the agent can no longer reach any other object. This results in the large dips in values we observe in the learned extended values. These extended values can now be composed to solve new tasks immediately. 

\begin{figure*}[h!]
        \centering
    \begin{minipage}[t]{0.3\textwidth}
        \centering
        \includegraphics[height=1.4in]{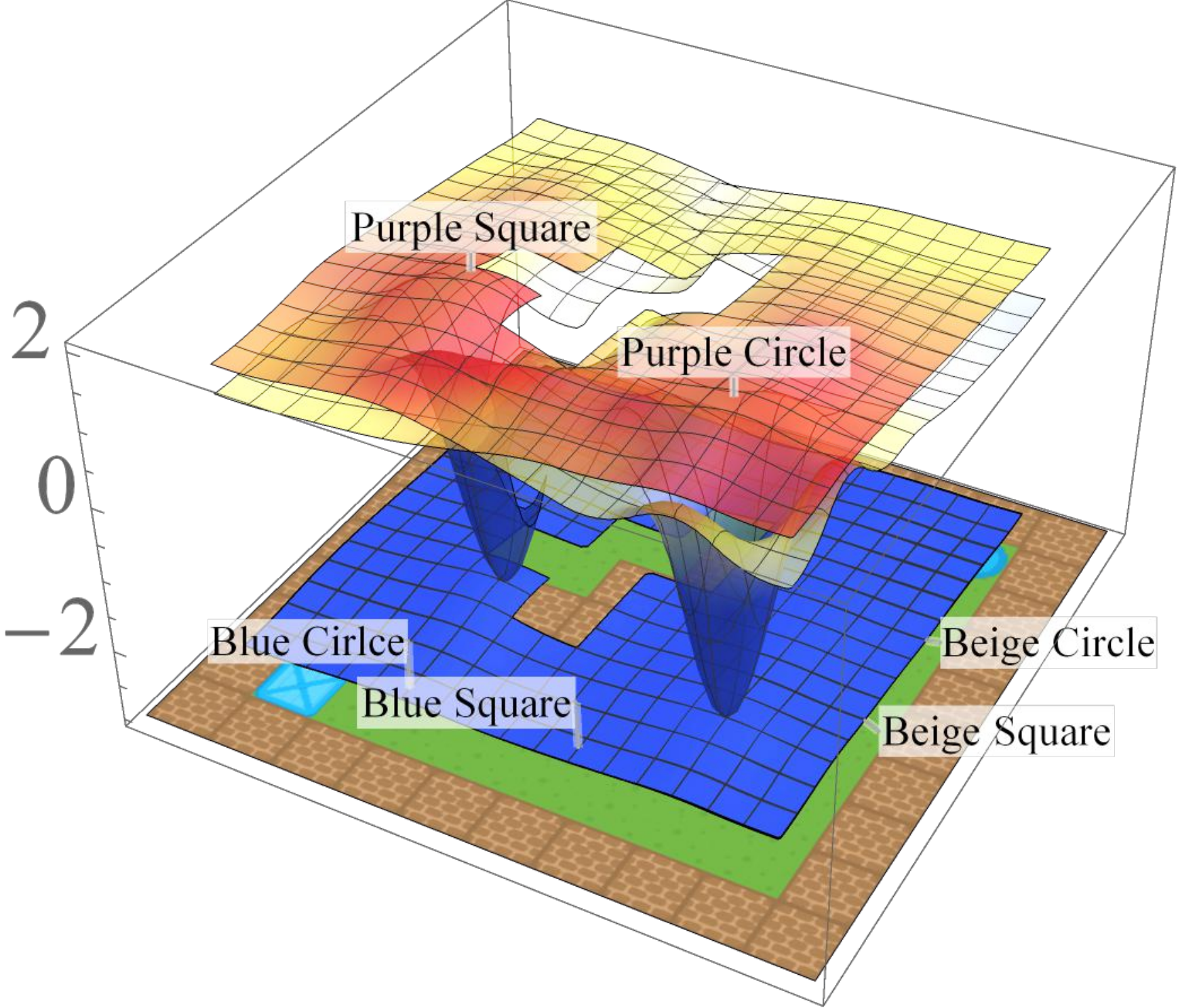}
        \caption{Extended value function for collecting purple objects.}
        \label{fig:purple_dts}
    \end{minipage}%
    \quad
    \begin{minipage}[t]{0.3\textwidth}
        \centering
        \includegraphics[height=1.4in]{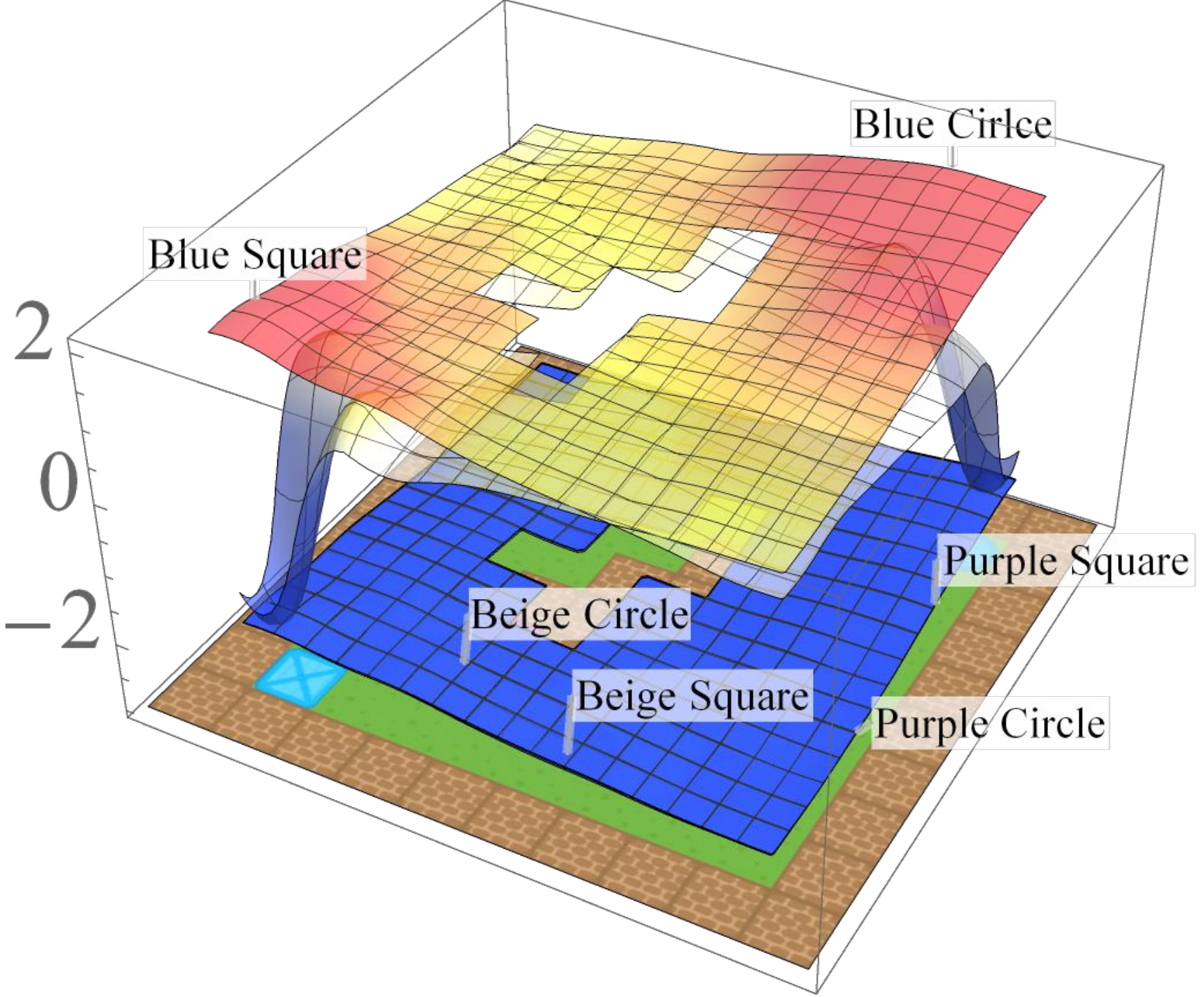}
        \caption{Extended value function for collecting blue objects.}
        \label{fig:blue_dts}
    \end{minipage}%
    \quad
    \begin{minipage}[t]{0.3\textwidth}
        \centering
        \includegraphics[height=1.4in]{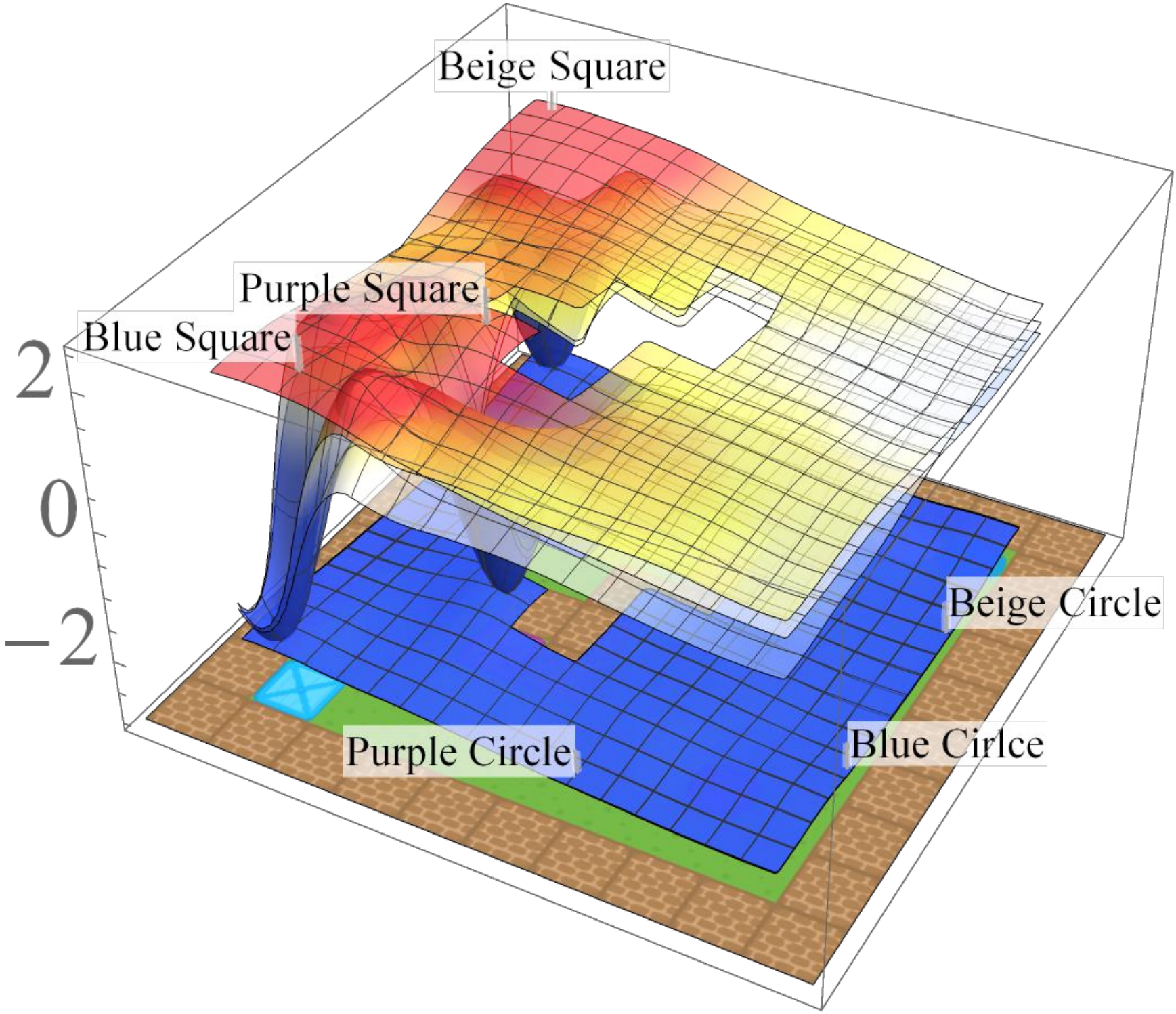}
        \caption{Extended value function for collecting squares.}
        \label{fig:square_dts}
    \end{minipage}%
\end{figure*}

Similarly to Section~5, we demonstrate composition characterised by disjunction, conjunction and exclusive-or, with the resulting value functions and trajectories illustrated by Figure~\ref{fig:repoman-new}. Since the extended value functions learn how to achieve all terminal states in a task and how desirable those terminal states are, we observe that it can still be leveraged for zero-shot composition even when the terminal states differ between tasks. Figure~\ref{fig:Boxman_returns_2} shows the average returns across random placements of the agent and objects.

\begin{figure}[h!]
    \centering
    \includegraphics[width=0.5\linewidth]{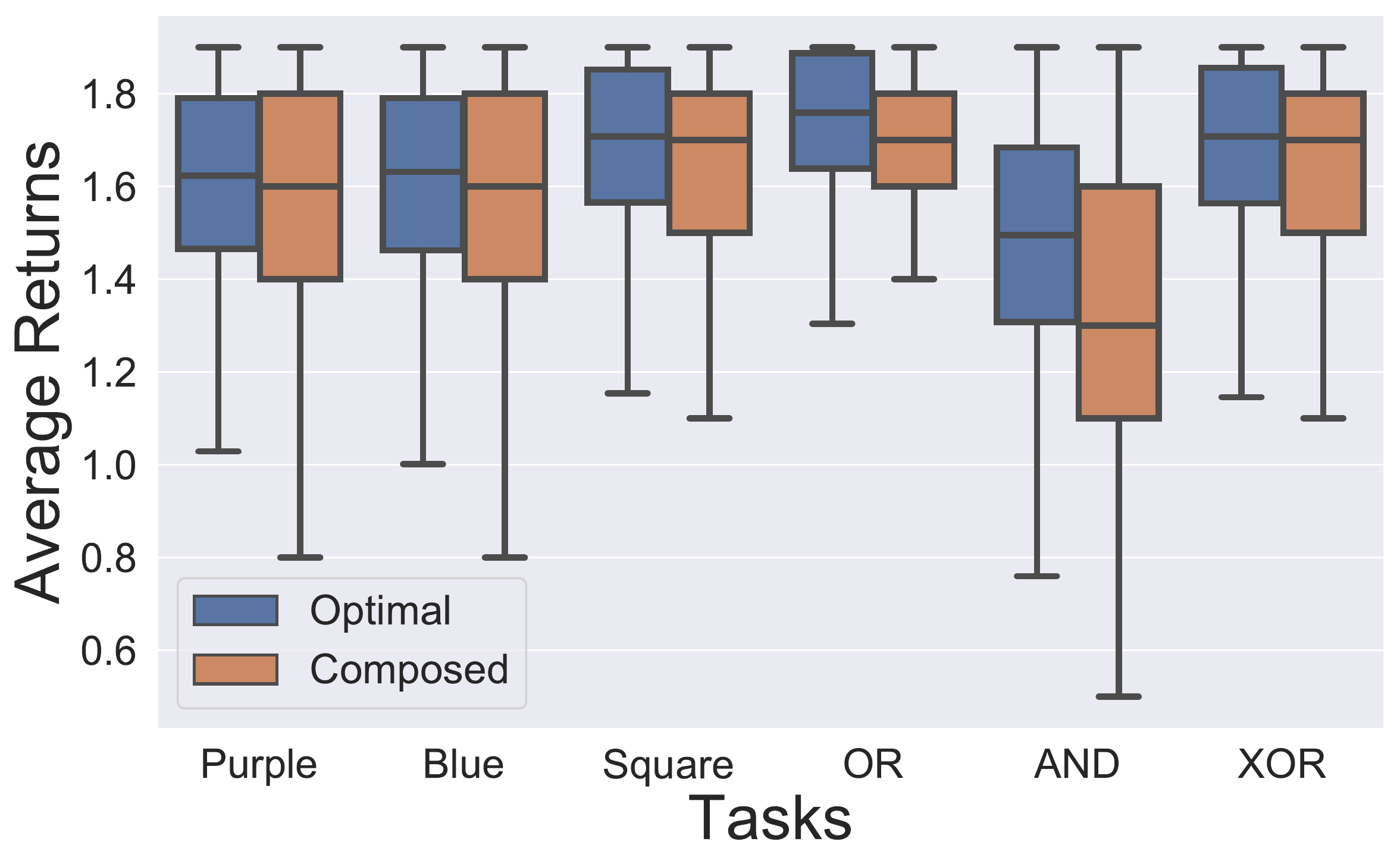}
    \caption{Average returns over 1000 episodes for the \textit{Purple}, \textit{Blue}, \textit{Square}, \textit{Purple OR Blue}, \textit{Blue AND Square} and \textit{Blue XOR Square} tasks.}
    \label{fig:Boxman_returns_2}
\end{figure}

\begin{figure*}[h!]
        \centering
    \begin{subfigure}[t]{0.3\textwidth}
        \centering
        \includegraphics[height=1.3in]{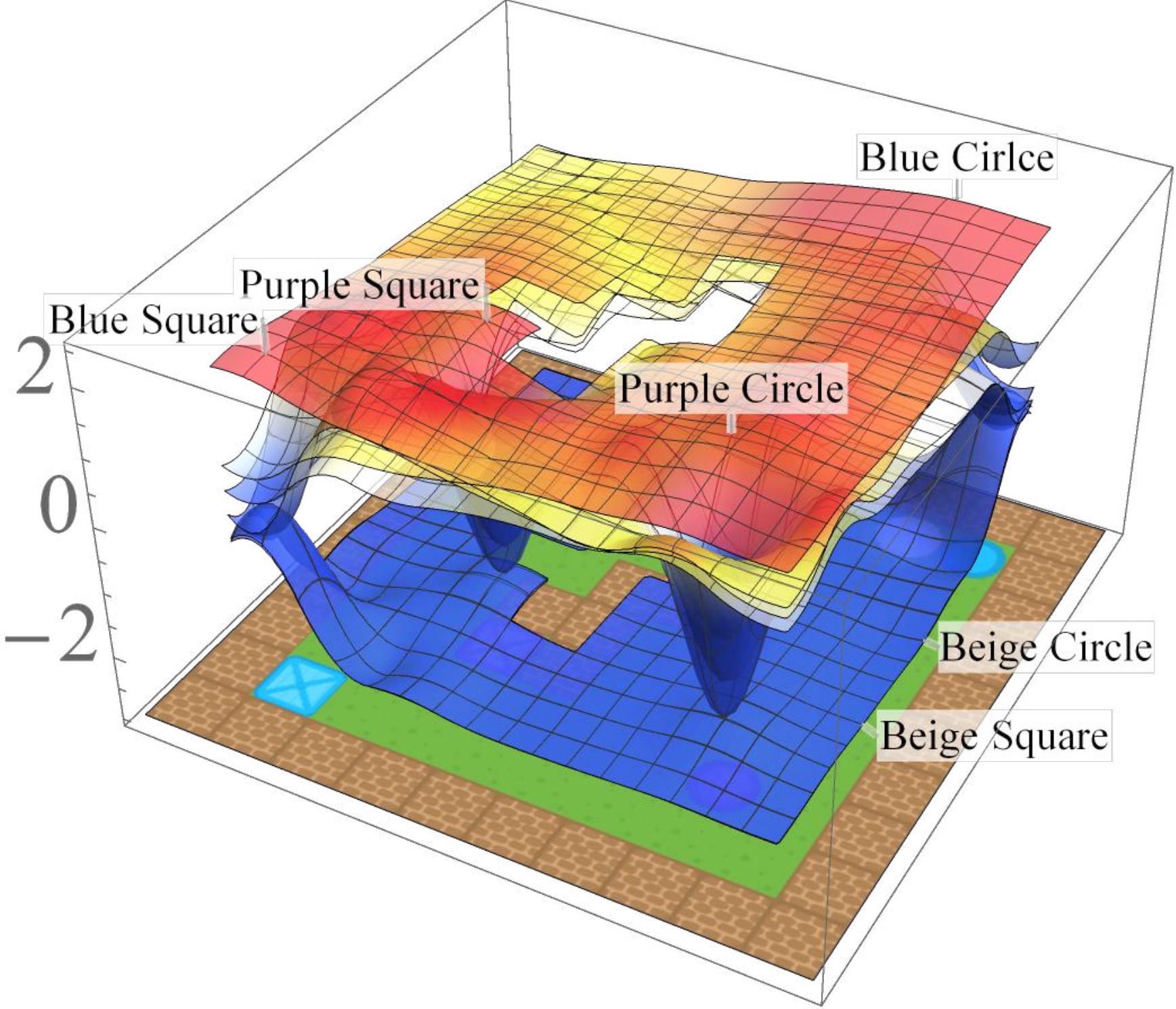}
        \caption{Extended value function for disjunctive composition.}
        \label{fig:u}
    \end{subfigure}%
    \quad
    \begin{subfigure}[t]{0.3\textwidth}
        \centering
        \includegraphics[height=1.3in]{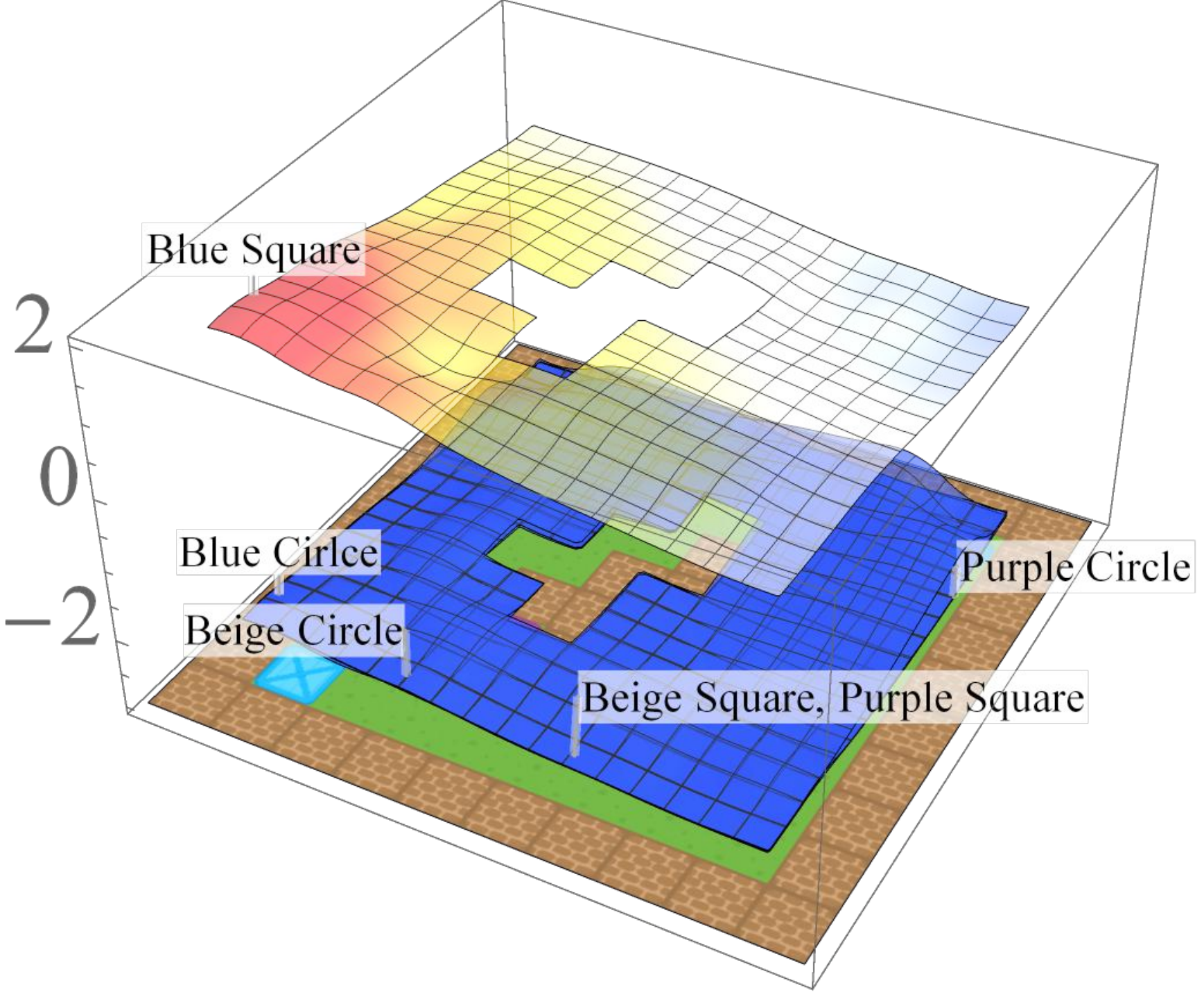}
        \caption{Extended value function for conjunctive composition.}
        \label{fig:v}
    \end{subfigure}%
    \quad
    \begin{subfigure}[t]{0.3\textwidth}
        \centering
        \includegraphics[height=1.3in]{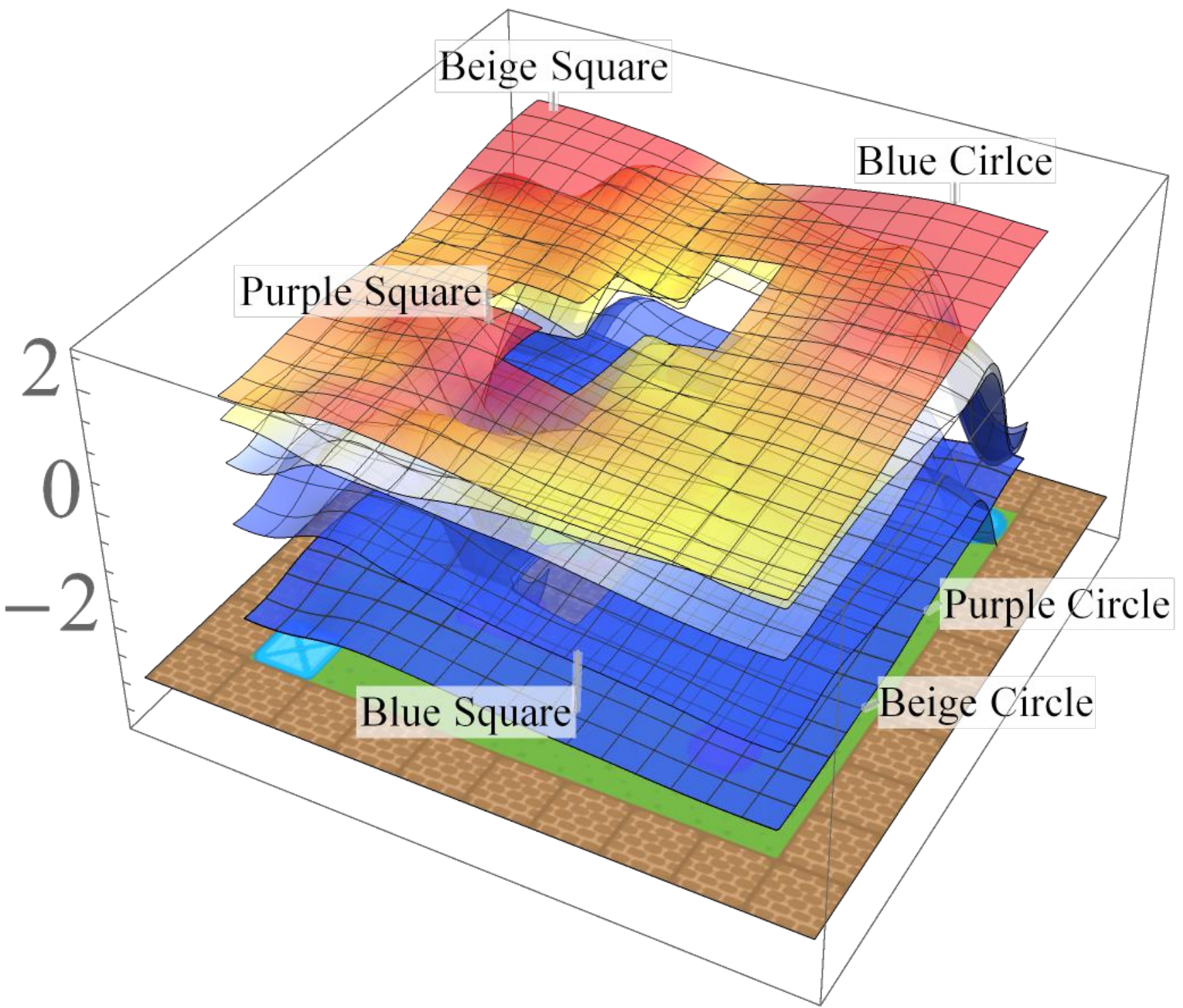}
        \caption{Extended value function for exclusive-or composition.}
        \label{fig:w}
    \end{subfigure}%
    \\
        \centering
    \begin{subfigure}[t]{0.3\textwidth}
        \centering
        \includegraphics[height=1.3in]{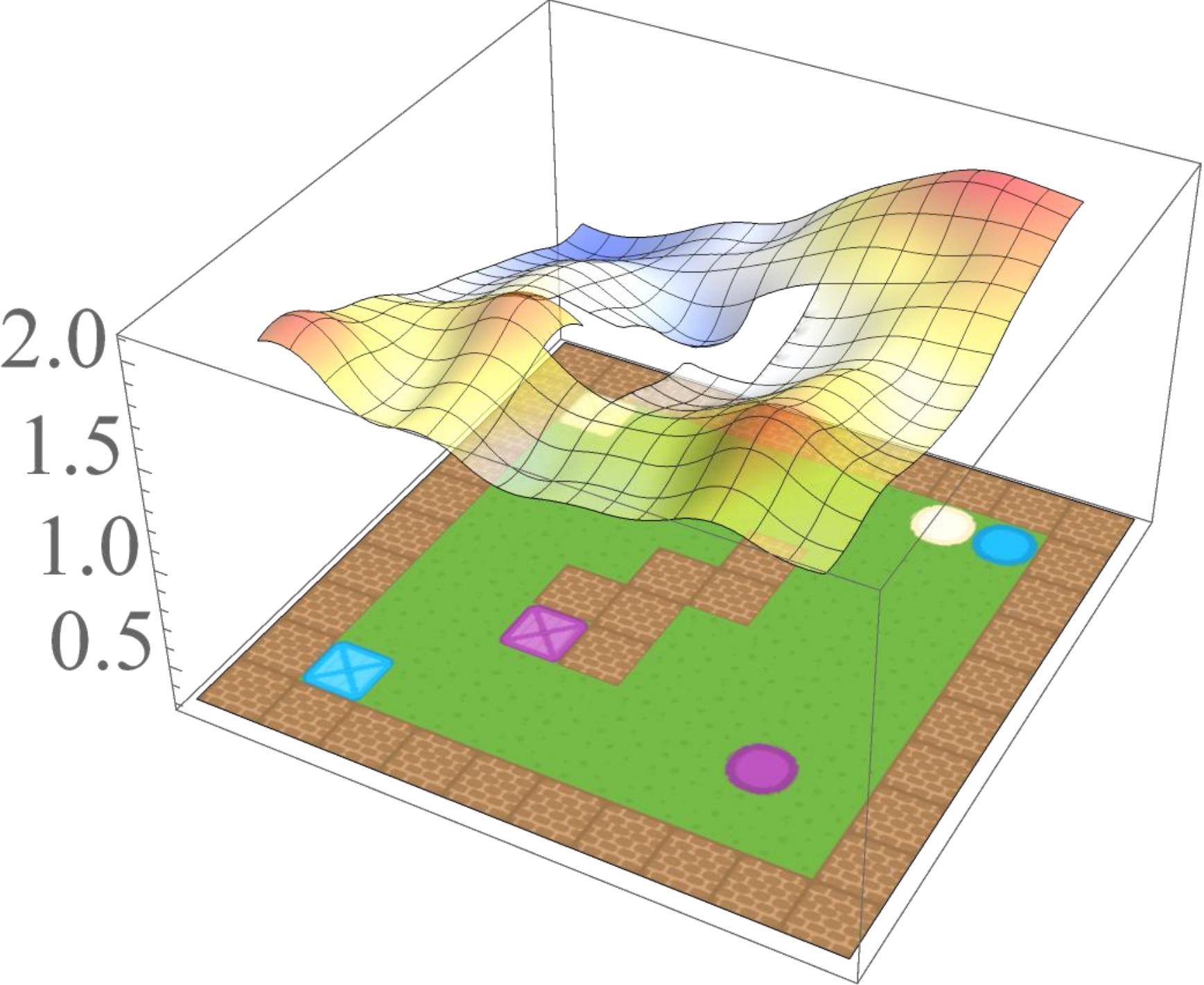}
        \caption{Value function for disjunctive composition.}
        \label{fig:u}
    \end{subfigure}%
    \quad
    \begin{subfigure}[t]{0.3\textwidth}
        \centering
        \includegraphics[height=1.3in]{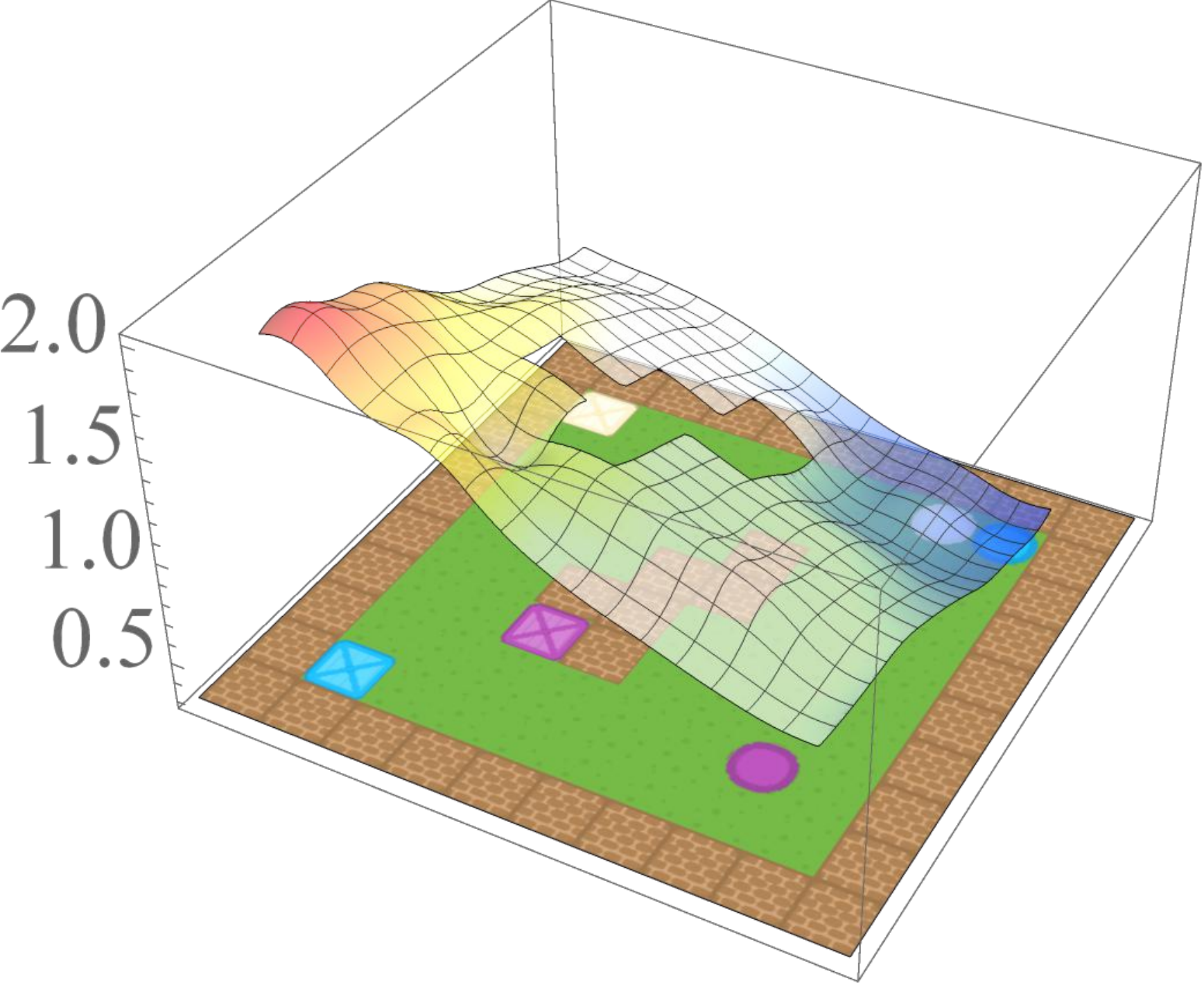}
        \caption{Value function for conjunctive composition.}
        \label{fig:v}
    \end{subfigure}%
    \quad
    \begin{subfigure}[t]{0.3\textwidth}
        \centering
        \includegraphics[height=1.3in]{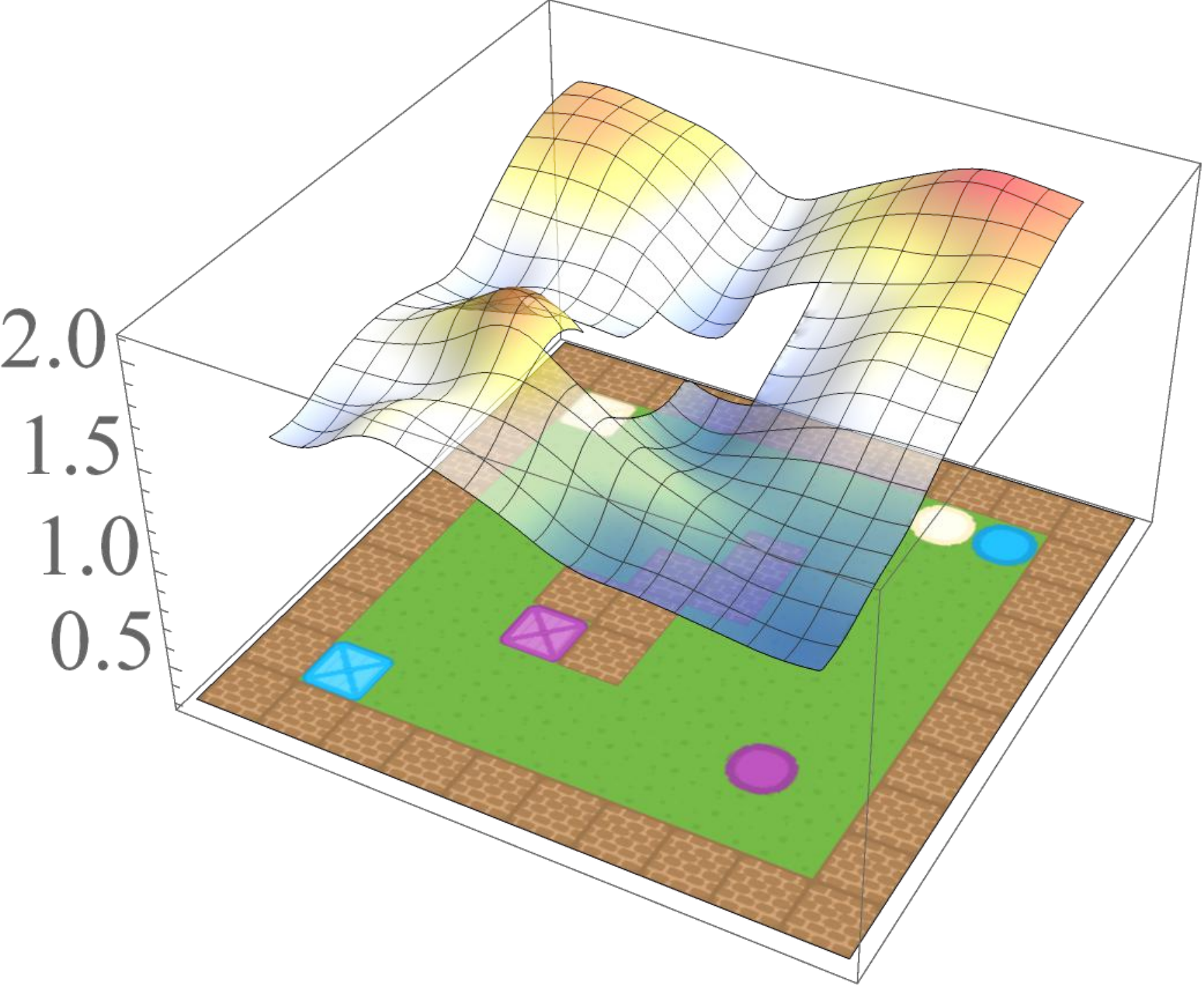}
        \caption{Value function for exclusive-or composition.}
        \label{fig:w}
    \end{subfigure}%
    \\
    \centering
    \begin{subfigure}[t]{0.3\textwidth}
        \centering
        \includegraphics[height=1.3in]{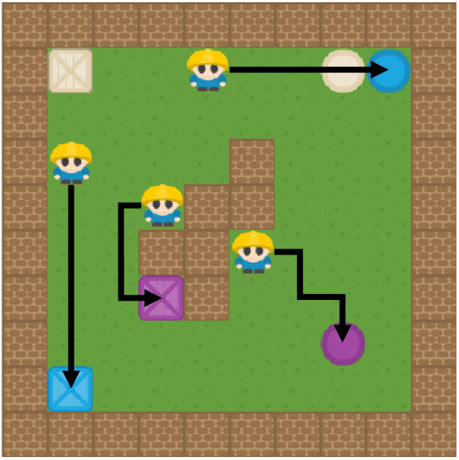}
        \caption{Trajectories for disjunctive composition (collect blue or purple objects).}
        \label{fig:u}
    \end{subfigure}%
    \quad
    \begin{subfigure}[t]{0.3\textwidth}
        \centering
        \includegraphics[height=1.3in]{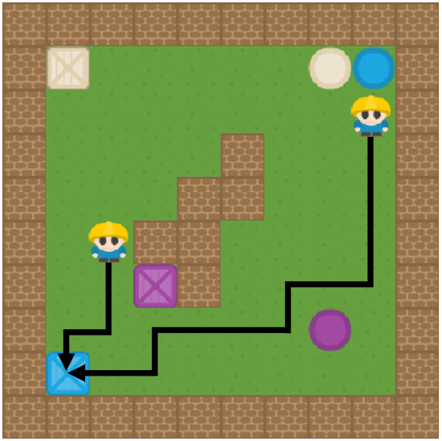}
        \caption{Trajectories for conjunctive composition (collect blue squares).}
        \label{fig:v}
    \end{subfigure}%
    \quad
    \begin{subfigure}[t]{0.3\textwidth}
        \centering
        \includegraphics[height=1.3in]{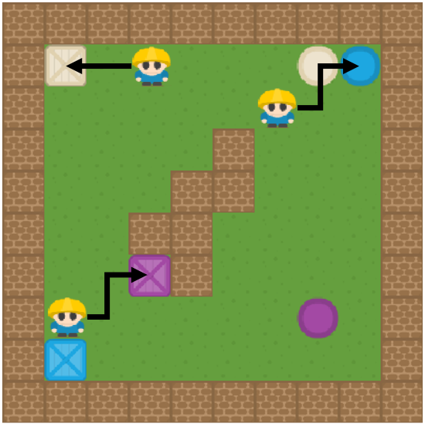}
        \caption{Trajectories for exclusive-or composition (collect blue or square objects, but not blue squares).}
        \label{fig:w}
    \end{subfigure}%
    \caption{Results for the video game environment with relaxed assumptions. We generate value functions to solve the disjunction of blue and purple tasks, and the conjunction and exclusive-or of blue and square tasks.}
    \label{fig:repoman-new}
\end{figure*}

In summary, we have shown that our compositional approach offers strong empirical performance, even when the theoretical assumptions are violated.
Finally, we expect that, in general, the errors due to these violations will be far outweighed by the errors due to non-linear function approximation.

\section{Selecting Base Tasks}

The Four Rooms domain requires the agent to navigate to one of the centres of the rooms in the environment. Figure~\ref{fig:rooms} illustrates the layout of the environment and the goals the agent must reach. 

\begin{figure}[h!]
    \centering
    \includegraphics[width=0.3\linewidth]{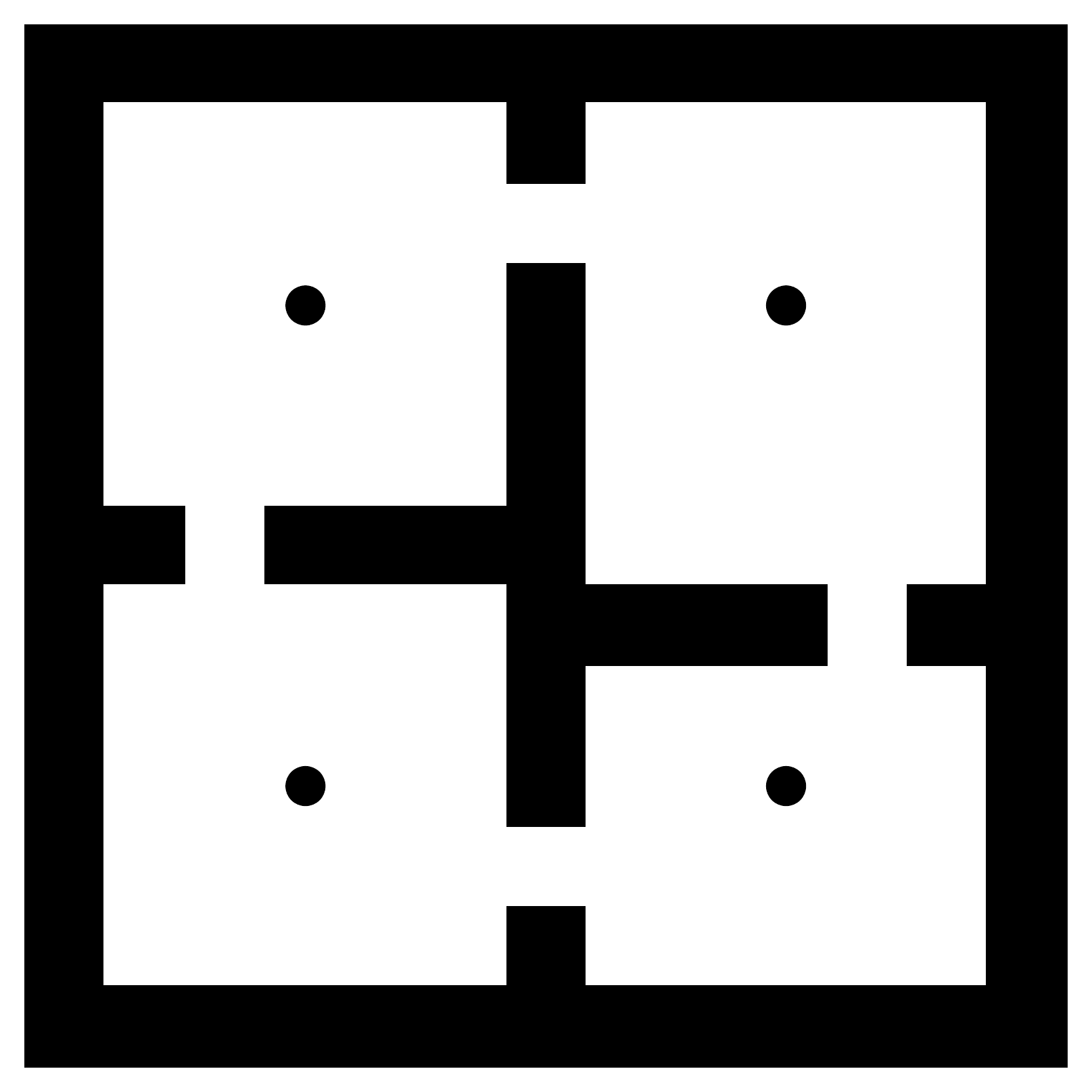}
    \caption{The layout of the Four Rooms domain. The circles indicate goals the agent must reach. We refer to the goals as \texttt{top-left}, \texttt{top-right}, \texttt{bottom-left}, and \texttt{bottom-right}.}
    \label{fig:rooms}
\end{figure}

Since we know the goals upfront, we can select a minimal set of base tasks by assigning each goal a   Boolean number, and then using the columns of the table to select the tasks. 
To illustrate, we assign Boolean numbers to the goals as follows:

\begin{table}[h!]
\centering
\begin{tabular}{lll}
$x_1$    & $x_2$ & Goals \\
\hline
$\rsmall$    & $\rsmall$    & \texttt{bottom-right}      \\
$\rsmall$    & $\rbig$    & \texttt{bottom-left}      \\
$\rbig$      & $\rsmall$    & \texttt{top-right}      \\
$\rbig$      & $\rbig$    & \texttt{top-left}      \\
\hline
\end{tabular}
\caption{Assigning labels to the individual goals. The two Boolean variables, $x_1$ and $x_2$, represent the goals for the base tasks the agent will train on.}
\label{tab:base}
\end{table}

As there are four goals, we can represent each uniquely with just two Boolean variables. 
Each column in Table~\ref{tab:base} represents a base task, where the set of goals for each task are those goals assigned a value $\rbig$. 
We thus have two base tasks corresponding to $x_1 = \{ \text{\texttt{top-right}}, \text{\texttt{top-left}} \}$ and $x_2 = \{ \text{\texttt{bottom-left}}, \text{\texttt{top-left}} \}$.

\section{UVFA Architecture and Hyperparameters}

In our experiments, we used a UVFA with the following architecture:

\begin{enumerate}
    \item Three convolutional layers:
    \begin{enumerate}
        \item Layer 1 has 6 input channels, 32 output channels, a kernel size of 8 and a stride of 4.
        \item Layer 2 has 32 input channels, 64 output channels, a kernel size of 4 and a stride of 2.
        \item Layer 3 has 64 input channels, 64 output channels, a kernel size of 3 and a stride of 1.
    \end{enumerate}
    \item Two fully-connected linear layers:
        \begin{enumerate}
        \item Layer 1 has input size 3136 and output size 512 and uses a ReLU activation function.
        \item Layer 2 has input size 512 and output size 4 with no activation function.

    \end{enumerate}
\end{enumerate}

We used the ADAM optimiser with batch size 32 and a learning rate of $10^{-4}$.
We trained every 4 timesteps and update the target Q-network every 1000 steps.
Finally, we used $\epsilon$-greedy exploration, annealing $\epsilon$ to $0.01$ over 100000 timesteps.



\bibliography{neurips2020}
\bibliographystyle{neurips2020}

%% file: math_commands.tex
\usepackage{amsmath,amsfonts,bm}

\def\eqref#1{equation~\ref{#1}}

\def\floor#1{\lfloor #1 \rfloor}
\def\1{\bm{1}}

\DeclareMathAlphabet{\mathsfit}{\encodingdefault}{\sfdefault}{m}{sl}
\SetMathAlphabet{\mathsfit}{bold}{\encodingdefault}{\sfdefault}{bx}{n}

\def\sR{{\mathbb{R}}}

\newcommand{\E}{\mathbb{E}}

\newcommand{\R}{\mathbb{R}}

\DeclareMathOperator*{\argmax}{arg\,max}

%% file: neurips2020.bbl
\begin{thebibliography}{24}
\providecommand{\natexlab}[1]{#1}
\providecommand{\url}[1]{\texttt{#1}}
\expandafter\ifx\csname urlstyle\endcsname\relax
  \providecommand{\doi}[1]{doi: #1}\else
  \providecommand{\doi}{doi: \begingroup \urlstyle{rm}\Url}\fi

\bibitem[Andrychowicz et~al.(2017)Andrychowicz, Wolski, Ray, Schneider, Fong,
  Welinder, McGrew, Tobin, Abbeel, and Zaremba]{andrychowicz2017hindsight}
Andrychowicz, M., Wolski, F., Ray, A., Schneider, J., Fong, R., Welinder, P.,
  McGrew, B., Tobin, J., Abbeel, P., and Zaremba, W.
\newblock Hindsight experience replay.
\newblock In \emph{Advances in Neural Information Processing Systems}, pp.\
  5048--5058, 2017.

\bibitem[Barreto et~al.(2017)Barreto, Dabney, Munos, Hunt, Schaul, van Hasselt,
  and Silver]{barreto17}
Barreto, A., Dabney, W., Munos, R., Hunt, J., Schaul, T., van Hasselt, H., and
  Silver, D.
\newblock Successor features for transfer in reinforcement learning.
\newblock In \emph{Advances in Neural Information Processing Systems}, pp.\
  4055--4065, 2017.

\bibitem[Barto \& Mahadevan(2003)Barto and Mahadevan]{barto03}
Barto, A. and Mahadevan, S.
\newblock Recent advances in hierarchical reinforcement learning.
\newblock \emph{Discrete Event Dynamic Systems}, 13\penalty0 (1-2):\penalty0
  41--77, 2003.

\bibitem[Bertsekas \& Tsitsiklis(1991)Bertsekas and Tsitsiklis]{bertsekas91}
Bertsekas, D. and Tsitsiklis, J.
\newblock An analysis of stochastic shortest path problems.
\newblock \emph{Mathematics of Operations Research}, 16\penalty0 (3):\penalty0
  580--595, 1991.

\bibitem[Fox et~al.(2016)Fox, Pakman, and Tishby]{fox15}
Fox, R., Pakman, A., and Tishby, N.
\newblock Taming the noise in reinforcement learning via soft updates.
\newblock In \emph{32nd Conference on Uncertainty in Artificial Intelligence},
  2016.

\bibitem[Haarnoja et~al.(2018)Haarnoja, Pong, Zhou, Dalal, Abbeel, and
  Levine]{haarnoja18}
Haarnoja, T., Pong, V., Zhou, A., Dalal, M., Abbeel, P., and Levine, S.
\newblock Composable deep reinforcement learning for robotic manipulation.
\newblock In \emph{2018 IEEE International Conference on Robotics and
  Automation}, pp.\  6244--6251. IEEE, 2018.

\bibitem[Hunt et~al.(2019)Hunt, Barreto, Lillicrap, and Heess]{hunt19}
Hunt, J., Barreto, A., Lillicrap, T., and Heess, N.
\newblock Composing entropic policies using divergence correction.
\newblock In \emph{Proceedings of the 36th International Conference on Machine
  Learning}, volume~97 of \emph{Proceedings of Machine Learning Research}, pp.\
   2911--2920. PMLR, 2019.

\bibitem[Jaksch et~al.(2010)Jaksch, Ortner, and Auer]{jaksch10}
Jaksch, T., Ortner, R., and Auer, P.
\newblock Near-optimal regret bounds for reinforcement learning.
\newblock \emph{Journal of Machine Learning Research}, 11\penalty0
  (Apr):\penalty0 1563--1600, 2010.

\bibitem[James \& Collins(2006)James and Collins]{james06}
James, H. and Collins, E.
\newblock An analysis of transient {M}arkov decision processes.
\newblock \emph{Journal of Applied Probability}, 43\penalty0 (3):\penalty0
  603--621, 2006.

\bibitem[Kaelbling(1993)]{kaelbling93}
Kaelbling, L.~P.
\newblock Learning to achieve goals.
\newblock In \emph{International Joint Conferences on Artificial Intelligence},
  pp.\  1094--1099, 1993.

\bibitem[Levine et~al.(2016)Levine, Finn, Darrell, and Abbeel]{levine16}
Levine, S., Finn, C., Darrell, T., and Abbeel, P.
\newblock End-to-end training of deep visuomotor policies.
\newblock \emph{The Journal of Machine Learning Research}, 17\penalty0
  (1):\penalty0 1334--1373, 2016.

\bibitem[Lillicrap et~al.(2016)Lillicrap, Hunt, Pritzel, Heess, Erez, Tassa,
  Silver, and Wierstra]{lillicrap16}
Lillicrap, T., Hunt, J., Pritzel, A., Heess, N., Erez, T., Tassa, Y., Silver,
  D., and Wierstra, D.
\newblock Continuous control with deep reinforcement learning.
\newblock In \emph{International Conference on Learning Representations}, 2016.

\bibitem[Mirowski et~al.(2017)Mirowski, Pascanu, Viola, Soyer, Ballard, Banino,
  Denil, Goroshin, Sifre, Kavukcuoglu, et~al.]{mirowski2016learning}
Mirowski, P., Pascanu, R., Viola, F., Soyer, H., Ballard, A., Banino, A.,
  Denil, M., Goroshin, R., Sifre, L., Kavukcuoglu, K., et~al.
\newblock Learning to navigate in complex environments.
\newblock In \emph{International Conference on Learning Representations}, 2017.

\bibitem[Mnih et~al.(2015)Mnih, Kavukcuoglu, Silver, Rusu, Veness, Bellemare,
  Graves, Riedmiller, Fidjeland, Ostrovski, et~al.]{mnih15}
Mnih, V., Kavukcuoglu, K., Silver, D., Rusu, A., Veness, J., Bellemare, M.,
  Graves, A., Riedmiller, M., Fidjeland, A., Ostrovski, G., et~al.
\newblock Human-level control through deep reinforcement learning.
\newblock \emph{Nature}, 518\penalty0 (7540):\penalty0 529, 2015.

\bibitem[Peng et~al.(2019)Peng, Chang, Zhang, Abbeel, and Levine]{peng19}
Peng, X., Chang, M., Zhang, G., Abbeel, P., and Levine, S.
\newblock {MCP}: Learning composable hierarchical control with multiplicative
  compositional policies.
\newblock \emph{arXiv preprint arXiv:1905.09808}, 2019.

\bibitem[Saxe et~al.(2017)Saxe, Earle, and Rosman]{saxe17}
Saxe, A., Earle, A., and Rosman, B.
\newblock Hierarchy through composition with multitask {LMDP}s.
\newblock \emph{Proceedings of the 34th International Conference on Machine
  Learning}, 70:\penalty0 3017--3026, 2017.

\bibitem[Schaul et~al.(2015)Schaul, Horgan, Gregor, and Silver]{schaul15}
Schaul, T., Horgan, D., Gregor, K., and Silver, D.
\newblock Universal value function approximators.
\newblock In \emph{Proceedings of the 32nd International Conference on Machine
  Learning}, volume~37 of \emph{Proceedings of Machine Learning Research}, pp.\
   1312--1320, Lille, France, 2015. PMLR.

\bibitem[Silver et~al.(2017)Silver, Schrittwieser, Simonyan, Antonoglou, Huang,
  Guez, Hubert, Baker, Lai, Bolton, et~al.]{silver17}
Silver, D., Schrittwieser, J., Simonyan, K., Antonoglou, I., Huang, A., Guez,
  A., Hubert, T., Baker, L., Lai, M., Bolton, A., et~al.
\newblock Mastering the game of go without human knowledge.
\newblock \emph{Nature}, 550\penalty0 (7676):\penalty0 354, 2017.

\bibitem[Sutton et~al.(1999)Sutton, Precup, and Singh]{sutton99}
Sutton, R., Precup, D., and Singh, S.
\newblock Between {MDP}s and semi-{MDP}s: A framework for temporal abstraction
  in reinforcement learning.
\newblock \emph{Artificial Intelligence}, 112\penalty0 (1-2):\penalty0
  181--211, 1999.

\bibitem[Todorov(2007)]{todorov07}
Todorov, E.
\newblock Linearly-solvable {Markov} decision problems.
\newblock In \emph{Advances in Neural Information Processing Systems}, pp.\
  1369--1376, 2007.

\bibitem[Todorov(2009)]{todorov09}
Todorov, E.
\newblock Compositionality of optimal control laws.
\newblock In \emph{Advances in Neural Information Processing Systems}, pp.\
  1856--1864, 2009.

\bibitem[Van~Niekerk et~al.(2019)Van~Niekerk, James, Earle, and
  Rosman]{vanniekerk19}
Van~Niekerk, B., James, S., Earle, A., and Rosman, B.
\newblock Composing value functions in reinforcement learning.
\newblock In \emph{Proceedings of the 36th International Conference on Machine
  Learning}, volume~97 of \emph{Proceedings of Machine Learning Research}, pp.\
   6401--6409. PMLR, 2019.

\bibitem[Veeriah et~al.(2018)Veeriah, Oh, and Singh]{veeriah2018many}
Veeriah, V., Oh, J., and Singh, S.
\newblock Many-goals reinforcement learning.
\newblock \emph{arXiv preprint arXiv:1806.09605}, 2018.

\bibitem[Watkins(1989)]{watkins89}
Watkins, C.
\newblock \emph{Learning from delayed rewards}.
\newblock PhD thesis, King's College, Cambridge, 1989.

\end{thebibliography}


\begin{thebibliography}{3}
\providecommand{\natexlab}[1]{#1}
\providecommand{\url}[1]{\texttt{#1}}
\expandafter\ifx\csname urlstyle\endcsname\relax
  \providecommand{\doi}[1]{doi: #1}\else
  \providecommand{\doi}{doi: \begingroup \urlstyle{rm}\Url}\fi

\bibitem[Barreto et~al.(2017)Barreto, Dabney, Munos, Hunt, Schaul, van Hasselt,
  and Silver]{barreto17}
Barreto, A., Dabney, W., Munos, R., Hunt, J., Schaul, T., van Hasselt, H., and
  Silver, D.
\newblock Successor features for transfer in reinforcement learning.
\newblock In \emph{Advances in Neural Information Processing Systems}, pp.\
  4055--4065, 2017.

\bibitem[Peng et~al.(2019)Peng, Chang, Zhang, Abbeel, and Levine]{peng19}
Peng, X., Chang, M., Zhang, G., Abbeel, P., and Levine, S.
\newblock {MCP}: Learning composable hierarchical control with multiplicative
  compositional policies.
\newblock \emph{arXiv preprint arXiv:1905.09808}, 2019.

\bibitem[Van~Niekerk et~al.(2019)Van~Niekerk, James, Earle, and
  Rosman]{vanniekerk19}
Van~Niekerk, B., James, S., Earle, A., and Rosman, B.
\newblock Composing value functions in reinforcement learning.
\newblock In \emph{Proceedings of the 36th International Conference on Machine
  Learning}, volume~97 of \emph{Proceedings of Machine Learning Research}, pp.\
   6401--6409. PMLR, 2019.

\end{thebibliography}
